\documentclass{article}

    \PassOptionsToPackage{numbers, compress, square}{natbib}
\usepackage[final]{neurips_2022}

\usepackage[utf8]{inputenc} %
\usepackage[T1]{fontenc}    %
\usepackage{hyperref}       %
\usepackage{url}            %
\usepackage{booktabs}       %
\usepackage{amsfonts}       %
\usepackage{nicefrac}       %
\usepackage{microtype}      %
\usepackage{xcolor}         %

\usepackage{natbib}

\usepackage{amsmath}

\usepackage{graphicx}
\usepackage{subcaption}
\usepackage{wrapfig}
\usepackage{sidecap}
\usepackage{relsize}

\usepackage{lscape}

\usepackage{algorithm}
\usepackage{algorithmic}
\usepackage{xcolor}
\usepackage{empheq}
\usepackage{amssymb}
\usepackage{mathtools}
\usepackage{amsthm}
\usepackage{colortbl}
\usepackage{thmtools}
\usepackage[capitalize,noabbrev]{cleveref}

\definecolor{linkcolor}{RGB}{83,83,182}
\hypersetup{
    colorlinks=true,
    citecolor=linkcolor,
    linkcolor=linkcolor
}

\theoremstyle{plain}
\newtheorem{theorem}{Theorem}[section]
\newtheorem{proposition}[theorem]{Proposition}
\newtheorem{lemma}[theorem]{Lemma}

\theoremstyle{definition}

\newtheorem{assumption}[theorem]{Assumption}
\theoremstyle{remark}

\usepackage{footmisc}

\newsavebox{\largestimage}

\usepackage{thmtools}
\usepackage{thm-restate}

\declaretheoremstyle[
    shaded={bgcolor=\color{HTML}{E9F6FF}}
]{shadedtheorem}

\definecolor{bggreen}{HTML}{EFFEF0}
\definecolor{bgblue}{HTML}{D5EAFA}
\renewcommand\fbox{\fcolorbox{black}{bggreen}}

\newcommand{\new}[1]{#1}

\newcommand*\widefbox[1]{\fbox{\hspace{1em}#1\hspace{1em}}}

\usepackage[textwidth=1.5cm, textsize=scriptsize]{todonotes}

\usepackage{titlesec}
\newcommand{\bbE}{\mathbb{E}}

\newcommand{\bbR}{\mathbb{R}}

\newcommand{\calL}{\mathcal{L}}

\newcommand{\calN}{\mathcal{N}}
\newcommand{\calO}{\mathcal{O}}

\newcommand{\calU}{\mathcal{U}}

\newcommand{\train}{{\mathrm{train}}}
\newcommand{\test}{{\mathrm{test}}}
\newcommand{\val}{{\mathrm{val}}}

\DeclareMathOperator{\diag}{\mathbf{diag}}

\newcommand{\diff}{\mathrm{d}}

\newcommand{\setcomb}[1]{[#1]}

\newcommand{\norme}[1]{\left\| #1 \right\|}

\title{A framework for bilevel optimization that enables  stochastic and global variance reduction algorithms}

\author{%
  Mathieu Dagréou \\
  Inria, CEA\\
  Université Paris-Saclay\\
  Palaiseau, France\\
  \texttt{mathieu.dagreou@inria.fr} \\
   \And
   Pierre Ablin \\
   CNRS\\
   Université Paris-Dauphine, PSL-University\\
   Paris, France\\
   \texttt{pierre.ablin@cnrs.fr} \\
   \AND
   Samuel Vaiter \\
   CNRS\\
   Université Côte d’Azur, LJAD\\
   Nice, France\\
   \texttt{samuel.vaiter@cnrs.fr} \\
   \And
   Thomas Moreau \\
    Inria, CEA\\
    Université Paris-Saclay\\
  Palaiseau, France\\
  \texttt{thomas.moreau@inria.fr} \\
}

\begin{document}

\maketitle

\begin{abstract}
 Bilevel optimization, the problem of minimizing a \emph{value function} which involves the arg-minimum of another function, appears in many areas of machine learning. In a large scale \new{empirical risk minimization} setting where the number of samples is huge, it is crucial to develop stochastic methods, which only use a few samples at a time to progress. However, computing the gradient of the value function involves solving a linear system, which makes it difficult to derive unbiased stochastic estimates.
To overcome this problem we introduce a novel framework, in which the solution of the inner problem, the solution of the linear system, and the main variable evolve at the same time. These directions are written as a sum, making it straightforward to derive unbiased estimates.
The simplicity of our approach allows us to develop global variance reduction algorithms, where the dynamic of all variables is subject to variance reduction.
We demonstrate that SABA, an adaptation of the celebrated SAGA algorithm in our framework, has $O(\frac1T)$ convergence rate, and that it achieves linear convergence under Polyak-\L ojasciewicz assumption.
This is the first stochastic algorithm for bilevel optimization that verifies either of these properties.
Numerical experiments validate the usefulness of our method.
\end{abstract}

\section{Introduction}
\label{sec:intro}
Bilevel optimization is attracting more and more attention in the machine learning community thanks to its wide range of applications. Typical examples are hyperparameters selection \cite{Bengio2000, Pedregosa2016, Franceschi2018,Bertrand2020}, data augmentation \cite{Cubuk2019, Rommel2022}, implicit deep learning \cite{Bai2019} or neural architecture search \cite{Liu2019}.
Bilevel optimization aims at minimizing a function whose value depends on the result of another optimization problem:
\begin{equation}\label{eq:pb}
  \min_{x\in\bbR^d} h(x) = F(z^*(x), x),\quad
  \text{such that } z^*(x) \in \arg\min_{z\in\bbR^p} G(z, x)\enspace,
\end{equation}
where $F$ and $G$ are two real valued functions defined on $\bbR^p\times\bbR^d$.
$G$ is called the \textit{inner function}, $F$ is the \textit{outer function} and $h$ is the \textit{value function}. Similarly, $z$ is the \textit{inner variable} and $x$ is the \textit{outer variable}. In most cases, the function $z^*$ can only be approximated by an optimization algorithm, which makes bilevel optimization problems challenging.
\begin{wrapfigure}{R}{.4\textwidth}
    \vspace{-1em}
    \centering
    \includegraphics[width=.4\columnwidth]{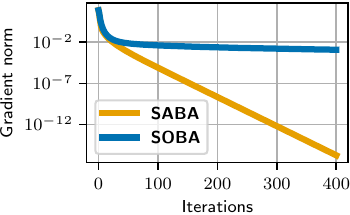}
    \vspace{-1.7em}
    \caption{Convergence curves of the two proposed methods on a toy problem. SABA is a stochastic method that achieves fast convergence on the value function.}
    \label{fig:intro_fig}
    \vspace{-1.2em}
\end{wrapfigure}
Under appropriate hypotheses, the function $h$ is differentiable, and the chain rule and implicit function theorem give for any $x\in\bbR^d$
\begin{equation}\label{eq:hgrad}
    \nabla h(x) = \nabla_2 F(z^*(x), x) + \nabla^2_{21}G(z^*(x), x) v^*(x) \enspace,
\end{equation}
where $v^*(x)\in \bbR^p$ is the solution of a linear system
\begin{equation}\label{eq:v_star_def}
    v^*(x) = - \left[\nabla_{11}^2G(z^*(x), x)\right]^{-1}\nabla_1 F(z^*(x), x) \enspace.
\end{equation}
In the light of \eqref{eq:hgrad} and \eqref{eq:v_star_def}, it turns out that the derivation of the gradient of $h$ at each iteration is cumbersome because it involves two subproblems: the resolution of the inner problem to find an approximation of $z^*(x)$ and the resolution of a linear system to find an approximation of $v^*(x)$. It makes the practical implementation of first-order methods like gradient descent for~\eqref{eq:pb} challenging.

As is the case in many machine learning problems, we suppose in this paper that $F$ and $G$ are empirical means:
$$F(z, x) = \frac{1}{m}\sum_{j=1}^m F_j(z,x),\quad G(z, x) = \frac{1}{n}\sum_{i=1}^n G_i(z,x)\enspace.$$
This structure suggests the use of stochastic methods to solve~\eqref{eq:pb}.
For single-level problems (that is, classical optimization problems where one function should be minimized), using Stochastic Gradient Descent (SGD;  \cite{Robbins1951,Bottou2010}) and variants is natural because individual gradients are straightforward unbiased estimators of the gradient. In the bilevel framework, we want to develop algorithms that make progress on problem~\eqref{eq:pb} by using only a few functions $F_j$ and $G_i$ at a time. However, since $\nabla h$ involves the inverse of the Hessian of $G$, building such stochastic algorithms is quite challenging, one of the difficulties being that there is no straightforward unbiased estimator of $\nabla h$.
Still, in settings where $m$ or $n$ are large, where computing even a single evaluation of $F$ or $G$ is extremely expensive, stochastic methods are the only scalable algorithms.

Variance reduction~\cite{Johnson2013, Defazio2014, schmidt2017minimizing, Fang2018, Cutkosky2019} is a popular technique to obtain fast stochastic algorithms. In a single-level setting, these methods build an approximation of the gradient of the objective function using only stochastic gradients. Contrary to SGD, the variance of the approximation goes to $0$ as the algorithm progresses, allowing for faster convergence. For instance, the SAGA method~\cite{Defazio2014} achieves linear convergence if the objective function satisfies a Polyak-\L ojasciewicz inequality, and $O(\frac1T)$ convergence rate on smooth non-convex functions~\cite{Reddi2016}.
The extension of these methods to bilevel optimization is a natural idea to develop faster algorithms. However, this idea is hard to implement because it is hard to derive unbiased estimators of $\nabla h$, let alone variance reduction ones.

\textbf{Contributions.}\quad We introduce a \textbf{novel framework for bilevel optimization} in \cref{sec:methods}, where the inner variable, the solution of the linear system~\eqref{eq:v_star_def} and the outer variable evolve jointly.
The evolution directions are
written as sums of derivatives of $F_j$ and $G_i$,
which allows us to derive simple unbiased stochastic estimators. In this framework, we propose SOBA, an extension of SGD (\cref{sec:soba}), and SABA (\cref{sec:saba}), an extension of the variance reduction algorithm SAGA~\cite{Defazio2014}.
In \cref{sec:theory} we analyse the convergence of our methods. SOBA is shown to achieve \new{$\inf_{t\leq T}\bbE[\|\nabla h(x^t)\|^2] = O(\log(T)T^{-\frac12})$} with decreasing step sizes. We prove that SABA with fixed step sizes achieves $\frac1T\sum_{t=1}^T\new{\bbE[}\|\nabla h(x^t)\|^2\new{]} = O(\frac 1T)$. SABA is therefore, to the best of our knowledge, the first \textbf{stochastic bilevel algorithm that matches the convergence rate of gradient descent} on $h$. We also prove that SABA achieves \textbf{linear convergence} under the assumption that $h$ satisfies a Polyak-\L ojasciewicz inequality. To the best of our knowledge, SABA is also the first stochastic bilevel algorithm to feature such a property. \new{Importantly, \textbf{these rates match the rates of the single level counterparts of each algorithm in non-convex setting} (SGD for SOBA and SAGA for SABA).}
Finally, in \cref{sec:experiments}, we provide an \textbf{extensive benchmark} of many stochastic bilevel methods on hyperparameters selection and data hyper-cleaning, and illustrate the usefulness of our approach.

\textbf{Related work.}\quad The bilevel optimization problem has a strong history in the optimization community, taking root in game theory~\cite{stackelberg1952theory}.
Gradient-based algorithms to solve~\eqref{eq:pb} can be mainly classified in two different categories depending on how $\nabla h$ is computed, by \emph{automatic} or \emph{implicit differentiation}.

Since the solution of the inner problem $z^*(x)$ is approximated by the output of an iterative algorithm, it is possible to use automatic differentiation~\cite{Wengert1964,linnainmaa1976taylor} to approximate $\nabla h(x)$. It consists in differentiating the different steps of the inner optimization algorithm -- see~\cite{Baydin2018a} for a review -- and has been applied successfully to several bilevel problems arising in machine learning~\cite{Domke2012, Franceschi2017}. One of the main drawbacks of this approach is that it requires to store in memory each iterate of the inner optimization algorithm, although this problem can sometimes be overcome using invertible optimization algorithms~\cite{Maclaurin2015} or truncated backpropagation \cite{Shaban2019}.

The use of the implicit function theorem to obtain~\eqref{eq:hgrad} and~\eqref{eq:v_star_def} is known as implicit differentiation~\cite{Bengio2000}.
While the cost of computing exactly~\eqref{eq:hgrad} can be prohibitive for large-scale problems,
\citet{Pedregosa2016} showed that we can still converge to a stationary point of the problem by using approximate solutions of the inner problem and linear system~\eqref{eq:v_star_def}, if the approximation error goes to $0$ sufficiently quickly.
The complexity of approximate implicit differentiation has been studied in \cite{Grazzi2020}. \citet{Ramzi2022} propose to reuse the computations done in the forward pass to approximate the solution of the linear system \eqref{eq:v_star_def} when the inner problem is solved thanks to a quasi-Newton method.

In the last few years, several works have proposed different strategies to solve~\eqref{eq:pb} in a stochastic fashion.
A first set of methods relies on \emph{two nested loops}: one inner loop to solve the inner problem with a stochastic method, and one outer loop to update the outer variable with an approximate gradient direction.
In \cite{Ghadimi2018, Ji2021a, Chen2021b} the authors use several SGD iterations for the inner problem and then use stochastic Neumann approximations to get an estimate solution of the linear system, which provides them with an approximation of $\nabla h$ used to update $x$. \new{The analysis of this kind of method was refined by \citet{Chen2021b}, allowing to achieve the same convergence rates as those of SGD.} The convergence of the hypergradient when using stochastic solvers for the inner problem and the linear system has been studied in \cite{Grazzi2021}. \citet{Arbel2022} replace the Neumann approximation by SGD steps to estimate~\eqref{eq:v_star_def}.
Other authors have proposed \emph{single loop} algorithms, alternating steps in the inner and the outer problem.
\citet{Hong2021} propose to perform Neumann approximations of the \new{inverse} Hessian and use a single SGD step for the inner problem.
It was refined in~\cite{Guo2021a} and~\cite{Yang2021} where the optimization procedure uses a momentum acceleration.
Other variations around this idea include~\cite{Huang2021, Khanduri2021, Chen2022, Guo2021, Li2022}.
We refer to \cref{tab:comparison} in appendix for a detailed comparison of these methods.

\textbf{Notation.}\quad
The set of integers between 1 and $n$ (included) is denoted $\setcomb{n}$.
For $f: \bbR^p\times\bbR^d\to\bbR$, we denote $\nabla_i f(z, x)$ its gradient w.r.t. the $i$\textsuperscript{th} variable. The Hessian of $f$ with respect to the first variable is denoted $\nabla^2_{11} f(z,x)\in\bbR^{p\times p}$, and the cross-derivatives matrix is $\nabla^2_{21} f(z,x)\in\bbR^{d\times p}$. If $v$ is a vector, $\|v\|$ is its Euclidean norm. If $M$ is a matrix, $\|M\|$ is its spectral norm. A function is said to be $L$-smooth, for $L > 0$, if it is differentiable, and its gradient is $L$-Lipschitz.

\section{Proposed framework}\label{sec:methods}

\begin{wrapfigure}{r}{.45\textwidth}
    \vspace{-2.3em}
    \begin{minipage}[t]{.45\textwidth}
        \begin{algorithm}[H]
        \begin{algorithmic}
           \STATE {\bfseries Input:} initializations $z_0\in\bbR^p$, $x_0\in\bbR^d$, $v_0\in\bbR^p$, number of iterations $T$, step size sequences $(\rho^t)_{t < T}$ and $(\gamma^t)_{t < T}$.
           \FOR{$t = 0,\dots, T-1$}
           \STATE
               Update $z$: $z^{t+1} = z^t - \rho^t D_z^t\enspace,$

               Update $v$: $v^{t+1} = v^t - \rho^t D_v^t\enspace,$

             Update $x$: $x^{t+1} = x^t - \gamma^t D_x^t\enspace,$

             where $D_z^t, D_v^t$ and $D_x^t$ are unbiased estimators of $D_z(z^t, v^t, x^t), D_v(z^t, v^t, v^t)$ and $D_x(z^t, v^t, x^t)$.
           \ENDFOR
        \end{algorithmic}
        \caption{General framework}
        \label{alg:sgd}
        \end{algorithm}
    \end{minipage}
    \vspace{-2.6em}
\end{wrapfigure}
In this section, we introduce our framework in which the solution of the inner problem, the solution of the linear system~\eqref{eq:v_star_def} and the outer variable all evolve at the same time, following directions that are written as a sum of derivatives of $F_j$ and $G_i$.
We define
\begin{align}
\label{eq:direction_z}
   D_z(z, v, x) &= \nabla_1 G(z, x) \enspace, \\
\label{eq:direction_v}
   D_v(z, v, x) &= \nabla_{11}^2 G(z, x)v +  \nabla_1 F(z, x) \enspace,\\
\label{eq:direction_x}
   D_x(z, v, x) &= \nabla_{21}^2 G(z, x)v +  \nabla_2 F(z, x) \enspace.
\end{align}
These directions are motivated by the fact that we have
$
    \label{eq:grad_value_with_v}
    \nabla h(x) = D_x(z^*(x), v^*(x), x)
$,
 with $z^*(x)$ the minimizer of $G(\cdot, x)$ and $v^*(x)$ the solution of $\nabla^2_{11}G(z^*(x), x)v = -\nabla_1F(z^*(x), x)$. When $x$ is fixed, we approximate $z^*$ by doing a gradient descent on $G$, following the direction $-D_z(z, v, x)$. Finally, when $z$ and $x$ are fixed, we find $v^*$ by following the direction $-D_v(z, v, x)$, which corresponds to a gradient descent on $v\mapsto \frac12\langle\nabla_{11}^2G(z, x)v, v\rangle + \langle \nabla_1F(z, x), v\rangle$.
The rest of the paper is devoted to the study of the global dynamics where the three variables $z, v$ and $x$ evolve at the same time, following stochastic approximations of $D_z, D_v$ and $D_x$.
The next proposition motivates the choice of these directions.

\begin{proposition}\label{prop:zeros_directions}
Assume that for all $x\in\bbR^d$, $G(\cdot,x)$ is strongly convex.
If $(z, v, x)$ is a zero of $(D_z, D_v, D_x)$, then $z = z^*(x)$, $v = v^*(x)$ and $\nabla h(x)=~0$.
\end{proposition}

We also note that the computation of these directions does \emph{not} require to compute the Hessian matrices $\nabla_{11}^2 G(z, x)$ and $\nabla_{21}^2 G(z, x)$: we only need to compute their product with a vector, which can be computed at a cost similar to that of computing a gradient.

The framework we propose is summarized in Algorithm~\ref{alg:sgd}. It consists in following a joint update rule in $(z, v, x)$ that follows directions $D_z^t, D_v^t$ and $D_x^t$ that are unbiased estimators of $D_z, D_v, D_x$.
The first and most important remark is that whereas $\nabla h$ cannot be written as a sum over samples, the directions $D_z, D_v$ and $D_x$ involve only simple sums, since their expressions are ``linear'' in $F$ and $G$:
\begin{align}
   \mathsmaller{D_z}(z, v, x) &\mathsmaller{=\frac1n\sum_{i=1}^n \nabla_1 G_i(z, x) \enspace,} \\
   \mathsmaller{D_v}(z, v, x) &\mathsmaller{=\frac1n \sum_{i=1}^n\nabla_{11}^2 G_i(z, x)v + \frac1m \sum_{j=1}^m \nabla_1 F_j(z, x) \enspace,} \\
   \mathsmaller{D_x}(z, v, x) &\mathsmaller{= \frac1n\sum_{i=1}^n\nabla_{21}^2 G_i(z, x)v + \frac1m \sum_{j=1}^m \nabla_2 F_j(z, x) \enspace.}
\end{align}
It is therefore straightforward to derive unbiased estimators of these directions.
In \cite{Li2022}, the authors considered one particular case of our framework, where each direction is estimated by using the STORM variance reduction technique (see \cite{Cutkosky2019}). Taking a step back by proposing the framework summarized in \cref{alg:sgd} opens the way to potential new algorithms that implement other techniques that exist in stochastic single-level optimization. In what follows, we study two of them.

\subsection{First example: the SOBA algorithm}
\label{sec:soba}
The simplest unbiased estimator is obtained by replacing each mean by one of its terms chosen uniformly at random, akin to what is done in classical single-level SGD. We call the resulting algorithm SOBA (StOchastic Bilevel Algorithm).
To do so, we choose two independent random indices $i\in\setcomb{n}$ and $j\in\setcomb{m}$ uniformly and estimate each term coming from $G$ using $G_i$ and each term coming from $F$ using $F_j$. This gives the unbiased \textbf{SOBA directions}
\begin{subequations}
\begin{empheq}[box=\widefbox]{align}
\label{eq:soba_z}
D_z^t &= \nabla_1G_i(z^t, x^t) \enspace,\\
\label{eq:soba_v}
 D_v^t &= \nabla_{11}^2G_i(z^t, x^t)v^t + \nabla_1F_j(z^t, x^t) \enspace,\\
 \label{eq:soba_x}
 D_x^t &=  \nabla_{21}^2G_i(z^t, x^t)v^t + \nabla_2 F_j(z^t, x^t) \enspace.
\end{empheq}
\end{subequations}

This provides us with a first algorithm, SOBA, where we plug \cref{eq:soba_z,eq:soba_v,eq:soba_x} in \cref{alg:sgd}. We defer its analysis to the next section.
Importantly, we use different step sizes for the update in $(z, v)$ and for the update in $x$.
We use the same step size in $z$ and in $v$ since
the inner problem and
the linear system have
similar conditioning, which is that of $\nabla^2_{11}G(z^t, x^t)$.
The need for a different step size for the outer and inner problems is clear: both problems can have different conditioning.

An important remark for SOBA is that all the stochastic directions used are computed at the same point $z^t, v^t$ and $x^t$ with the same indices $(i, j)$.
The update of $z$, $v$ and $x$ can thus be performed in parallel instead of sequentially, benefiting from hardware parallelism.
Moreover, this enables to share the computations between the different directions.
This is the case in hyperparameters selection where $G_i(z, x) = \ell_i(\langle z, d_i\rangle) + \frac x2 \|z\|^2$, with $d_i$ a training sample, and $\ell_i$ that measures how good is the prediction $\langle z, d_i\rangle$. In this setting, we have
$\nabla_1 G_i(z, x) = \ell_i'(\langle z, d_i\rangle)d_i + xz
$ and $\nabla^2_{11}G_i(z, x)v = \ell''_i(\langle z, d_i\rangle) \langle v, d_i\rangle d_i$.
The prediction $\langle z, d_i\rangle$ can thus be computed only once to obtain both quantities.
For more complicated models, where automatic differentiation is used to compute the different derivatives and Jacobian-vector products, we can store the computational graph only once to compute at the same time $\nabla_1 G_i(z, x), \nabla^2_{11}G_i(z, x)v$ and $\nabla^2_{21}G_i(z, x)v$, requiring only one backward pass, thanks to the $\mathcal{R}$ technique~\cite{pearlmutter1994fast}.

Finally, like all single loop bilevel algorithms, our method updates at the same time the inner and outer variable, avoiding unnecessary optimization of the inner problem when $x$ is far from the optimum.

\subsection{Global variance reduction with the SABA algorithm}
\label{sec:saba}
In classical optimization, SGD fails to reach optimal rates because of the variance of the gradient estimator. Variance reduction algorithms aim at reducing this variance, in order to follow directions that are closer to the true gradient and to achieve superior practical and theoretical convergence.

In our framework, since the directions $D_z, D_v$ and $D_x$ are all written as sums of derivatives of $F_j$ and $G_i$, it is easy to adapt most classical variance reduction algorithms.
We focus on the celebrated SAGA algorithm~\cite{Defazio2014}. The extension we propose is called SABA (Stochastic Average Bilevel Algorithm).
The general idea is to replace each sum in the directions $D$ by a sum over a memory, updating only one term at each iteration.
To help the exposition, we denote $y = (z, x, v)$ the vector of joint variables.
Since we have sums over $i$ and over $j$, we have two memories for each variable: $w^t_i$ for $i\in\setcomb{n}$ and $\tilde{w}^t_j$ for $j\in\setcomb{m}$, which keep track of the previous values of the variable $y$.

At each iteration $t$, we draw two random independent indices $i\in\setcomb{n}$ and $j\in\setcomb{m}$ uniformly and update the memories. To do so, we put $w^{t+1}_i = y^t$ and $w^{t+1}_{i'} = w^t_{i'}$ for $i'\neq i$, and $\tilde{w}^{t+1}_j = y^t$ and $\tilde{w}^{t+1}_{j'} = \tilde{w}^t_{j'}$ for $j'\neq j$. Each sum in the directions $D$ is then approximated using SAGA-like rules:  given $n$ functions $\phi_{i'}$ for $i'\in \setcomb{n}$, we define
$
    S[\phi, w]^t_i = \phi_{i}(w^{t+1}_i) - \phi_{i}(w^t_i) +\frac1n\sum_{i'=1}^n\phi_{i'}(w^t_{i'})
$.
This is an unbiased estimator of the average of the $\phi$'s since
$
\bbE_i\Big[S[\phi, w]^t_i\Big] = \frac1n\sum_{i=1}^n \phi_i(y^t)
$.

With a slight abuse of notation, we call $\nabla_{11}^2G v$ the sequence of functions $(y\mapsto \nabla_{11}^2G_i(z, x)v)_{i\in\setcomb{n}}$ and $\nabla_{21}^2Gv$ the sequence of functions $(y\mapsto\nabla_{21}^2G_i(z, x)v)_{i\in\setcomb{n}}$. We define the \textbf{SABA directions} as
\begin{subequations}
\begin{empheq}[box=\widefbox]{align}
    \label{eq:saba_z}
    D_z^t &= S[\nabla_1G, w]^t_i \enspace,\\
    \label{eq:saba_v}
    D_v^t &= S[\nabla_{11}^2Gv, w]^t_i + S[\nabla_1 F, \tilde{w}]^t_j \enspace,\\
    \label{eq:saba_x}
    D_x^t&= S[\nabla_{21}^2Gv, w]^t_i + S[\nabla_2 F, \tilde{w}]^t_j \enspace.
\end{empheq}
\end{subequations}
These estimators are unbiased estimators of the directions $D_z, D_v$ and $D_x$. The SABA algorithm corresponds to \cref{alg:sgd} where we use \cref{eq:saba_z,eq:saba_v,eq:saba_x} as update directions.
When taking a step size $\gamma^t= 0$ in the outer problem, hereby stopping progress in $x$, we recover the iterations of the SAGA algorithm on the inner problem.
In practice, the sum in $S$ is computed by doing a rolling average \new{(see \cref{app:exp} for precision)}, and the quantities $\phi_i(w_{i}^t)$ are stored rather than recomputed: the cost of computing the SABA directions is the same as that of SGD. It requires an additional memory for the five quantities, of total size $n\times p + (n + m)\times (p + d) $ floats that can be reduced by using larger batch sizes. Indeed, if $b_\mathrm{in}$ and $b_\mathrm{out}$ are respectively the inner and the outer batch sizes, the memory load is reduced to $n_b \times p + (n_ b + m_b)\times (p\times d)$ with $n_b = \lceil\frac{n}{b_\mathrm{inn}}\rceil$ and $m_b  = \lceil\frac{m}{b_\mathrm{out}}\rceil$ which are smaller than the number of samples. This memory load can also be reduced in specific cases, for instance when $G$ and $F$ correspond to linear models, where the individual gradients and Hessian-vector products are proportional to the samples. In this case, we only store the proportionality ratio, reducing the memory load to $3n + 2m$ floats.
Like for SOBA, the computations of the new quantities $\phi_i(w_i^{t+1})$ are done in parallel, thus benefiting from hardware acceleration and shared computations.
Despite this memory load, using SAGA-like variance reduction instead of STORM as done in \cite{Li2022, Yang2021, Khanduri2021} has the advantage to bring the variance of the estimate directions to zero, enabling faster $O(\frac1T)$ convergence.

In the next section, we show that SABA is fast. It essentially has the same properties as SAGA: despite being stochastic, it converges with fixed step sizes, and reaches the same rate of convergence as gradient descent on $h$.

\section{Theoretical analysis}
\label{sec:theory}
In this section, we provide convergence rates of SOBA and SABA under some classical assumptions. Note that, unlike most of the stochastic bilevel optimization papers, we work in finite sample setting rather than the more general expectation setting. Actually, SABA does not make any sense for functions that don't have a finite-sum structure. However, we stress that SOBA could be studied in a more general setting to obtain the same bounds as here. Also, the finite sum setting is still interesting since doing empirical risk minimization is very common in practice in machine learning.
The proofs and the constants in big-$O$ are deferred in \cref{app:proofs}.

\subsection{Background and assumptions}\label{subsec:assumptions}
We start by stating some regularity assumptions on the functions $F$ and $G$.
\begin{assumption}\label{ass:1}
The function $F$ is \new{twice differentiable. The derivatives $\nabla F$ and $\nabla^2 F$ are Lipschitz continuous in $(z,x)$ with respective Lipschitz constants $L^F_1$ and $L^F_2$.}
\end{assumption}
Note that the above assumption is typically verified
in the machine learning context, \emph{e.g.,} when $F$ is the ordinary least squares (OLS) loss or the logistic loss.

\begin{assumption} \label{ass:2}
The function $G$ is \new{three times} continuously differentiable on $\bbR^p\times\bbR^d$. For any $x\in\bbR^d$, $G(\,\cdot\,,x)$ is $\mu_G$-strongly convex. \new{The derivatives $\nabla G$, $\nabla^2 G$ and $\nabla^3 G$ are Lipschitz continuous in $(z,x)$ with respective Lipschitz constants $L^G_1$, $L^G_2$ and $L^G_3$.}
\end{assumption}

Strong convexity and smoothness with respect to $z$ of $G$ are verified when $G$ is a regularized least-squares/logistic regression with a full rank design matrix, when the data is not separable for the logistic regression. Moreover, the strong convexity ensures the existence and uniqueness of the inner optimization problem for any $x\in\bbR^d$.

\begin{assumption}\label{ass:3}
There exists $C_F~>~0$ such that for any $x$ we have $\|\nabla_1 F(z^*(x), x)\|\leq~C_F$.
\end{assumption}
This assumption, combined with the strong convexity of $G(\,\cdot\,,x)$, shows boundedness of $v^*$.
This assumption holds, for instance, in the case of hyperparameters selection for a Ridge regression problem. \new{Note that in Assumptions~\ref{ass:1} and~\ref{ass:2}, we assume more regularity of $F$ and $G$ than in stochastic bilevel optimization literature (see for instance \cite{Ghadimi2018, Hong2021, Ji2021a, Arbel2022}). It is necessary to get the smoothness of $v^*$ which will allow to adapt the proof of \citet{Chen2021b} and get tight convergence rates.}
The following lemma gives us some smoothness properties of the considered directions that will be useful to derive convergence rates of our methods.
\begin{lemma}\label{lemma:smoothness}
Under the Assumptions \ref{ass:1} to \ref{ass:3}, there exist constants $L_z$, $L_v$ and $L_x$ such that $\|D_z(z, v, x)\|^2 \leq L_z^2 \|z-z^*(x)\|^2$, $\|D_v(z, v, x)\|^2 \leq L_v^2(\|z-z^*(x)\|^2 + \|v-v^*(x)\|^2)$ and $\|D_x(z, v, x)-\nabla h(x)\|^2\leq L_x^2(\|z-z^*(x)\|^2 + \|v-v^*(x)\|^2)$.
\end{lemma}

In first order optimization, a fundamental assumption on the objective function is the smoothness assumption. In the case of vanilla gradient descent applied to a function $f$, it allows to get a convergence rate of $\|\nabla f(x^t)\|^2$ in $O(1/T)$, i.e. convergence to a stationary point~\cite{Nesterov2004}. The following lemma proved by \citet[Lemma 2.2]{Ghadimi2018} ensures the smoothness of $h$.
\begin{lemma}
\label{lemma:h_smoothness}
Under the Assumptions \ref{ass:1} to \ref{ass:3}, the function $h$ is $L^h$-smooth for some $L^h>0$.
\end{lemma}

The constant $L^h$ is specified in \cref{app:cst_h}.
As usual with the analysis of stochastic methods, we define the expected norms of the directions
$V_z^t = \bbE[\|D_z^t\|^2]$, $V_v^t = \bbE[\|D_v^t\|^2]$ and $V_x^t = \bbE[\|D_x^t\|^2]$,
where the expectation is taken over the past. Thanks to variance-bias decomposition, they are the sum of the variance of the stochastic direction and the squared-norm of the unbiased direction.
\new{For SOBA, we use classical bounds on variances like} those found for instance in~\cite{Hong2021}:
\begin{assumption}
\label{assumption:bounded_gradients}
There exist $B_z$ and $B_v$ such that for all $t$, $\new{\bbE_t[\|D^t_z\|^2]}~\leq~B_z^2(1 +~\|D_z(z^t, v^t, x^t)\|^2)$ and $\new{\bbE_t[\|D^t_v\|^2]}\leq B_v^2(1 +\|D_v(z^t, v^t, x^t)\|^2)$ \new{where $\bbE_t$ denotes the expectation conditionally to $(z^t, v^t, x^t)$}.
\end{assumption}
\new{For SOBA and SABA, we need to bound the expected norm of $D^t_x$. For SABA, this assumption allows to get the same sample complexity as SAGA for single-level problems.
\begin{assumption}
\label{assumption:bounded_gradients_x}
There exists $B_x$ such that for all $t$, $\bbE_t[\|D^t_x\|^2]\leq B_x^2$.
\end{assumption}}
\new{Assumptions \ref{assumption:bounded_gradients} and \ref{assumption:bounded_gradients_x} are} verified for instance, if all the $G_i$ and $\nabla_1G_i$ have at most quadratic growth, and if  $F$ has bounded gradients. They are also verified if the iterates remain in a compact set.
Note that we do not assume that $G$ has bounded gradients, as this would contradict its strong-convexity.
Finally, for the analysis of SABA, we need regularity on each $G_i$ and $F_j$:

\begin{assumption}
\label{ass:indiv_lipschitz}
For all $i\in\setcomb{n}$ and $j\in\setcomb{m}$, the functions $\nabla G_i$, $\nabla F_j$, $ \nabla_{11}^2G_i$ and $\nabla_{21}^2G_i$ are Lipschitz continuous in $(z, x)$.
\end{assumption}
\subsection{Fundamental descent lemmas}
Our analysis for SOBA and SABA is based on the control of both
$
    \delta_z^t = \bbE[\|z^t -z^*(x^t)\|^2]
$
and
$
    \delta_v^t = \bbE[\|v^t - v^*(x^t)\|^2]
$,
Strong convexity of $G$ \new{and smoothness of $z^*(x)$ and $v^*(x)$} allow to obtain the following lemma \new{by adapting the proof of \citet{Chen2021b}}. In what follows, we drop the dependency of the step sizes $\rho$ and $\gamma$ in $t$ for clarity.
\begin{lemma}\label{lemma:coupled_inequalities}
Assume that
\new{$\gamma^2\leq\min\left(\frac{\mu_GL_*^2}{4B_x^2L_{zx}^2}, \frac{\mu_GL_*^2}{8B_x^2L_{vx}^2}\right)\rho$}. We have:
\begin{align*}
    \delta^{t+1}_z &\leq \left(1-\frac{\rho\mu_G}4\right)\delta^t_z + 2\rho^2 V_z^t + \beta_{zx}\new{\gamma^2}V_x^t \new{+ \overline{\beta}_{zx}\frac{\gamma^2}{\rho} \bbE[\|D_x(z^t,v^t,x^t)\|^2]} \\
    \delta^{t+1}_v &\leq \left(1-\frac{\rho\mu_G}8\right)\delta_v^t +   \beta_{vz}\rho\delta_z^t +2\rho^2V_v^t +  \beta_{vx}\new{\gamma^2}V_x^t \new{+\overline{\beta}_{zx}\frac{\gamma^2}{\rho} \bbE[\|D_x(z^t,v^t,x^t)\|^2]}
\end{align*}
where \new{$\beta_{zx} = \beta_{vx} = 3L_*^2$, $\overline{\beta}_{zx} = \frac{8L_*^2}{\mu_G}$, $\overline{\beta}_{vx} = \frac{16L_*^2}{\mu_G}$}, $L_*$ is the maximum between the Lipschitz constants of $z^*$ and $v^*$ (see \cref{lemma:smoothness_star}), $\beta_{vz} = \frac1{\mu_G^3}(L^F_1\mu_G+L^G_2)^2$, \new{$L_{zx}$ and $L_{vx}$ are respectively the smoothness constants of $z^*$ and $v^*$}.
\end{lemma}
We insist that this result is obtained in general for \cref{alg:sgd} with arbitrary unbiased directions.
We can therefore invoke this lemma for the analysis of both SOBA and SABA.
We use the smoothness of $h$ to get the following lemma, \new{which is similar to \citep[Lemma 1]{Chen2021b}}.
\begin{lemma}\label{lemma:descent_lemma}
Let $h^t = \bbE[h(x^t)]$ and $g^t =\bbE[\|\nabla h(x^t)\|^2]$. We have
$$
h^{t+1}\leq h^t -\frac\gamma 2g^t \new{-\frac\gamma2 \bbE[\|D_x(z^t, v^t, x^t)\|^2]}+ \frac{\gamma}2L_x^2(\delta_z^t + \delta_v^t) + \frac{L^h}2 \gamma^2V_x^t \enspace.
$$
\end{lemma}
If \new{$z^t = z^*(x^t)$, $v^t=v^*(x^t)$, that is} $\delta_z$, $\delta_v$ both cancel and \new{$D_x(z^t, v^t, x^t) = \nabla h(x^t)$}, we get an inequality reminiscent of the smoothness inequality for SGD on $h$.

\subsection{Analysis of SOBA}

The analysis of SOBA is based on Lemmas~\ref{lemma:h_smoothness} and~\ref{lemma:coupled_inequalities}.
We have the following theorem, with fixed step sizes depending on the number of iterations:

\begin{restatable}[Convergence of SOBA, fixed step size]{shadedtheorem}{sobafixed}
\label{thm:soba_fixed}
Fix an iteration $T > 1$ and assume that Assumptions \ref{ass:1} to \ref{assumption:bounded_gradients_x} hold.
We consider fixed steps \new{$\rho^t = \frac{\overline{\rho}}{\sqrt{T}}$ and $\gamma^t=\xi\rho^t$ with $\overline{\rho}$ and $\xi$ precised in the appendix}. Let $(x^t)_{t\geq 1}$ the sequence of outer iterates for SOBA. Then, \new{$$\frac1T\sum_{t=1}^T\bbE[\|\nabla h(x^t)\|^2] = O(T^{-\frac12})\enspace.$$}
\end{restatable}
As opposed to \cite{Hong2021}, we do not need that the ratio $\frac\gamma\rho$ goes to 0, which allows to get a complexity (that is, the number of call to oracles to have an $\epsilon$-stationary solution) in $O(\epsilon^{-2})$ better than the $\tilde{O}(\epsilon^{-\frac52})$ they have. Also, note that this rate is the same as the one of SGD for non-convex and smooth objective \cite{Ghadimi2013,Bottou2018}.
We obtain a similar rate using decreasing step sizes:

\begin{restatable}[Convergence of SOBA, decreasing step size]{shadedtheorem}{sobadecreasing}
\label{thm:decreasing_step_size}
Assume that Assumptions \ref{ass:1} to \ref{assumption:bounded_gradients_x} hold. We consider steps \new{$\rho^t = \overline{\rho}t^{-\frac12}$} and \new{$\gamma^t=\xi\rho$}. Let $x^t$ the sequence of outer iterates for SOBA. Then,\new{~$$\inf_{t\leq T}\bbE[\|\nabla h(x^t)\|^2] = O(\log(T)T^{-\frac12})\enspace.$$}
\end{restatable}

\new{
As for SGD, SOBA suffers from the need of decreasing step sizes to get actual convergence because of the variance of the estimation on each directions.}
On the other hand, the analysis of SABA leverages the dynamic of all three variables, resulting in fast convergence with fixed step sizes.

\subsection{SABA: a stochastic method with optimal rates}
\label{sec:saga}

\new{In what follows, we denote $N = n + m$ the total number of samples.} The following theorem shows $O(\new{N^{\frac23}}T^{-1})$ convergence for the SABA algorithm in the general case where we only assume smoothness of $h$.
Our analysis of SABA is inspired by the analysis of single-level SAGA by~\citet{Reddi2016}.

\begin{restatable}[Convergence of SABA, smooth case]{shadedtheorem}{sabasmooth}
\label{th:cvg_saba_smooth}
Assume that Assumptions \ref{ass:1} to \ref{ass:3} and \ref{assumption:bounded_gradients_x} to \ref{ass:indiv_lipschitz} hold. We suppose \new{$\rho = \rho'N^{-\frac23}$ and $\gamma = \xi \rho$}, where $\rho'$ and $\xi$ depend only on $F$ and $G$ and are specified in appendix. Let $x^t$ the iterates of SABA. Then, $$\frac1T\sum_{t=1}^T\bbE[\|\nabla h(x^t)\|^2] = O\left(\new{N^{\frac23}}T^{-1}\right)\enspace.$$
\end{restatable}

To prove the theorem, the idea is to control the distance from the memory to the current variables.
We define
$
S^t=\frac1n\sum_{i=1}^n \|y^t - w^t_i\|^2 + \frac1m\sum_{j=1}^m \|y^t - \tilde{w}^t_j\|^2 \enspace.
$
In appendix, we show that we can find scalars $\phi_s, \phi_z, \phi_v>0$ such that the quantity $\mathcal{L}^t = h^t + \phi_s S^t + \phi_z \delta_z^t + \phi_v\delta_v^t$ satisfies $\mathcal{L}^{t+1} \leq \mathcal{L}^t - \frac\gamma2 g^t$. Summing these inequalities for $t=1\dots T$ and using the fact that $\mathcal{L}^t$ is lower bounded demonstrates the theorem.

\new{
Note that the step sizes are constant with respect to the time, but they scale with $N^{-\frac23}$. As a consequence, the sample complexity is $O(N^{\frac23}\epsilon^{-1})$ which is analogous of the one of SAGA for non-convex single level problems \cite{Reddi2016}. This is better than the sample complexity of \cref{alg:sgd} with full batch directions, which is $O(N\epsilon^{-1})$. Hence, with SABA, we get the best of both worlds: the stochasticity makes the scaling in $N$ of the sample complexity goes from $N$ in full batch mode to $N^{\frac23}$ for SABA, and the variance reduction makes the scaling in $\epsilon$ goes from $\epsilon^{-2}$ for SOBA to $\epsilon^{-1}$ for SABA. Our experiments in \cref{sec:experiments} confirm this gain.
}

Furthermore, if we assume that $h$ satisfies a Polyak-\L ojasiewicz (PL) inequality, we recover linear convergence. Recall that $h$ has the PL property if there exists $\mu_h > 0$ such that for all $x \in \bbR^d$, $\frac12\|\nabla h(x)\|^2\geq \mu_h(h(x) - h^*)$ with $h^*$ the minimum of~$h$.
\begin{restatable}[Convergence of SABA, PL case]{shadedtheorem}{sabapl}
\label{th:cvg_saba_pl}
Assume that $h$ satisfies the PL inequality and that Assumptions \ref{ass:1} to \ref{ass:3} and \ref{assumption:bounded_gradients_x} to \ref{ass:indiv_lipschitz} hold. We suppose \new{$\rho = \rho'N^{-\frac23}$ and $\gamma = \xi \rho'N^{-1}$}, where $\rho'$ and $\xi$ depend only on $F$ and $G$ and are specified in appendix. Let $x^t$ the iterates of SABA and \new{$c' \triangleq \min\left(\mu_h, \frac1{16P'}\right)$ with $P'$ specified in the appendix}. Then,
$$\bbE[h^T] - h^* = (1-\new{c'}\gamma)^T(h^0-h^* + C^0)$$
where $C^0$ is a constant specified in appendix that depends on the initialization of $z, v, x$ and memory.
\end{restatable}

The proof is similar to that of the previous theorem: we find coefficients $\phi_s, \phi_z, \phi_v$ such that $\mathcal{L}^t= h^t + \phi_s S^t + \phi_z \delta_z^t + \phi_v\delta_v^t$ satisfies the inequality $\mathcal{L}^{t+1} \leq (1 - \new{c'}\gamma)\mathcal{L}^t$, which is then unrolled. Note that in the case where we initialize $z$ and $v$ with $z^0 = z^*(x^0)$, $v^0 =v^*(x^0)$, and the memories $w_i^0 = w^0$, $\tilde{w}_j^0=w^0$ for all $i, j$, the constant $C^0$ cancels and the bound simplifies to $\bbE[h(x^T)] -h^* \leq (1 - c'\gamma)^T(h(x^0) - h^*)$.

Just like classical variance reduction methods in single-level optimization, this theorem shows that our method achieves linear convergence under PL assumption on the value function. To the best of our knowledge, our method is the first stochastic bilevel optimization method that enjoys such property.
We note that the PL hypothesis is more general than $\mu_h$-strong convexity of $h$ -- it is a necessary condition for strong convexity.

We see here the importance of \emph{global} variance reduction. Indeed, using  variance reduction only on $z$ and SGD on $x$ would lead to sub-linear convergence in $x$. This would be the case even with a perfect estimation of $z^*(x)$. Similarly, using variance reduction only on $x$ and SGD on $z$ would lead to sub-linear convergence in $z$, and hence in $x$. Using global variance reduction with respect to each variable as we propose here is the only way to achieve linear convergence.
We now turn to experiments, where we find that our method is also promising from a practical point of view.
\section{Experiments}
\label{sec:experiments}

Here we compare the performances of SOBA and SABA with competitor methods on different tasks.

The different methods being compared are stocBiO \cite{Ji2021a}, AmiGO \cite{Arbel2022},
FSLA \cite{Li2022},
MRBO \cite{Yang2021}, TTSA \cite{Hong2021}, BSA \cite{Ghadimi2018} and SUSTAIN \cite{Khanduri2021}.
A detailed account of the experiments is provided in \cref{app:exp}.~\footnote{The code of the benchmark is available at \url{https://github.com/benchopt/benchmark_bilevel} and the results are displayed in \url{https://benchopt.github.io/results/benchmark_bilevel.html}.}

\begin{figure}[h]
        \centering
        \savebox{\largestimage}{\includegraphics[width=.345\linewidth]{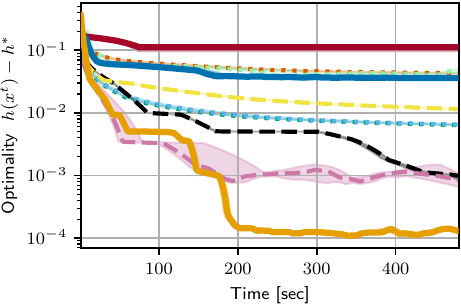}}
        \hfill
        \begin{subfigure}[b]{.345\linewidth}
            \centering
             \usebox{\largestimage}
            \caption{Logistic regression}
            \label{exp:ijcnn1}
        \end{subfigure}
        \hfill
        \begin{subfigure}[b]{.345\linewidth}
            \raisebox{\dimexpr.5\ht\largestimage-.5\height}{%
            \includegraphics[width=\linewidth]{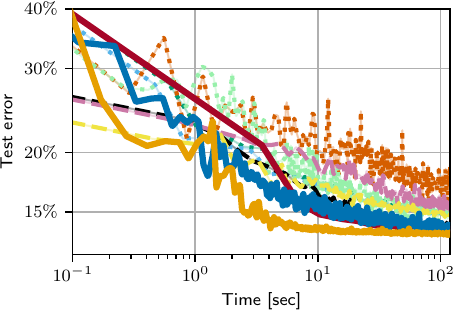}}
            \caption{Datacleaning}
            \label{exp:datacleaning_0-5}
        \end{subfigure}
        \hfill
        \begin{subfigure}[t]{.29\linewidth}
            \centering
            \raisebox{\dimexpr .8\ht\largestimage-.5\height}{%
            \includegraphics[width=\linewidth]{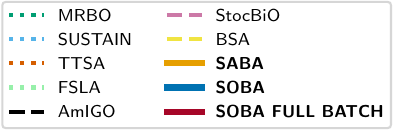}}
        \end{subfigure}
        \hfill

        \caption{Comparison of SOBA and SABA with other stochastic bilevel optimization methods. For each algorithm, we plot the median performance over 10 runs. In both experiments, SABA achieves the best performance. The dashed lines are for one loop competitor methods, the dotted lines are for two loops methods and the solid lines are  the proposed methods.
        \textbf{Left}: hyperparameter selection for $\ell^2$ penalized logistic regression on IJCNN1 dataset , \textbf{Right}: data hyper-cleaning on MNIST with $p=0.5$ corruption rate.
        }
        \label{fig:expes}
        \vspace{-1.5em}
\end{figure}

\subsection{Hyperparameters selection}

The first task we perform is hyperparameters selection to choose regularization parameters on $\ell^2$ logistic regression.
Let us denote $((d_i^\mathrm{train}, y_i^\mathrm{train}))_{1\leq i\leq n}$ and $((d_i^\mathrm{val}, y_i^\mathrm{val}))_{1\leq i\leq m}$ the training and the validation sets. In this case, the inner variable $\theta$ corresponds to the parameters of the model, and the outer variable $\lambda$ to the regularization. The functions $F$ and $G$ of the problem \eqref{eq:pb} are the logistic loss, with $\ell^2$ penalty for $G$, that is to say $F(\theta, \lambda) = \frac1m\sum_{i=1}^m \varphi( y^\mathrm{val}_i\langle d^\mathrm{val}_i, \theta\rangle)$ and $
G(\theta, \lambda) = \frac1n\sum_{i=1}^n \varphi(y^\mathrm{train}_i\langle  d^\mathrm{train}_i, \theta\rangle) + \frac12\sum_{k=1}^pe^{\lambda_k}\theta_k^2$
where $\varphi(u) = \log(1+e^{-u})$.
We fit a binary classification model on the  \texttt{IJCNN1}\footnote{\url{https://www.csie.ntu.edu.tw/~cjlin/libsvmtools/datasets/binary.html}} dataset. Here, $n= 49~990$, $m=91~701$ and $p=22$.

The suboptimality gap is plotted in \cref{exp:ijcnn1} for each method. The lowest values are reached by SABA. Moreover, SABA is the only single-loop method that reaches a suboptimality below $10^{-3}$. SOBA reaches a quite high final value but slightly better than TTSA and FSLA. The gap between SOBA and SABA highlights the benefits of variance reduction: it gives us a lower plateau and the fixed step sizes enable faster convergence.
\subsection{Data hyper-cleaning}

The second task we perform is data hyper-cleaning introduced in \cite{Franceschi2017} on the \texttt{MNIST}\footnote{\url{http://yann.lecun.com/exdb/mnist/}} dataset.
The data is patitioned into a training set $(d^\mathrm{train}_i, y^\mathrm{train}_i)$, a validation set $(d^\mathrm{val}_i, y^\mathrm{val}_i)$, and a test set. The training set contains 20000 samples, the validation set 5000 samples and the test set 10000 samples. The targets $y$ take values in $\{0,\dots,9\}$ and the samples $x$ are in dimension $784$.
Each sample in the training set is \emph{corrupted} with probability $p$: a sample is corrupted when we replace its label $y_i$ by a random label in $\{0,\dots, 9\}$. Samples in the validation and test sets are not corrupted.
The goal of datacleaning is to train a multinomial logistic regression on the train set and learn a weight per training sample, that should go to $0$ for corrupted samples.
This is formalized by the bilevel optimization problem \eqref{eq:pb} with $F(\theta, \lambda) = \frac1m\sum_{i=1}^m \ell(\theta d^\mathrm{val}_i, y^\mathrm{val}_i)$ and $G(\theta, \lambda) = \frac1n\sum_{i=1}^n \sigma(\lambda_i)\ell(\theta d^\mathrm{train}_i, y^\mathrm{train}_i) + C_r\|\theta\|^2$
where $\ell$ is the cross entropy loss and $\sigma$ is the sigmoid function.
The inner variable $\theta$ is a matrix of size $10\times 784$, and the outer variable $\lambda$ is a vector in dimension $n_\mathrm{train} = 20000$.

For the estimated parameters $\theta$ during optimization, we report in \cref{exp:datacleaning_0-5} the test error, \emph{i.e.}, the percent of wrong predictions on the testing data. We use for this experiment a corruption probability $p=0.5$. In general, the error decreases quickly until it reaches a final value. We observe that our method SABA outperforms all the other methods by reaching faster its smallest error, which is smaller than the ones of the other methods. For SOBA, it reaches a lower final error than stocBiO and BSA. In appendix, we provide other convergence curves, and find that for higher values of $p$, SABA is still the fastest algorithm to reach its final accuracy. Overall, we find that among all methods, even those that implement variance reduction (that is FSLA, MRBO, SUSTAIN, SABA), SABA is the one that demonstrates the best empirical performance.

\section{Conclusion}
In this paper, we have presented a framework for bilevel optimization that enables the straightforward development of stochastic algorithms.
The gist of our framework is that the directions in \cref{eq:direction_z,eq:direction_v,eq:direction_x} are all written as simple sums  of samples derivatives. We leveraged this fact to propose SOBA, an extension of SGD to our framework, and SABA, an extension of SAGA to our framework, which both achieve similar convergence rates as their single level counterparts. Finally, we think that our framework opens a large panel of potential methods for stochastic bilevel optimization involving techniques of extrapolation, variance reduction, momentum and so on.

\begin{ack}
We thank Othmane Sebbouh, Zaccharie Ramzi and Benoît Malézieux for their precious comments.
The authors acknowledge the support of the ANER RAGA BFC.
SV acknowledges the support of the ANR GraVa ANR-18-CE40-0005.
This work is supported by a public grant overseen by the French National Research Agency (ANR) through the program UDOPIA, project funded by the ANR-20-THIA-0013-01 and DATAIA convergence institute (ANR-17-CONV-0003). \end{ack}

\bibliography{biblio}
\bibliographystyle{plainnat}

\clearpage
\section*{Checklist}

\begin{enumerate}

\item For all authors...
\begin{enumerate}
  \item Do the main claims made in the abstract and introduction accurately reflect the paper's contributions and scope?
    \answerYes{See \cref{sec:theory} and \cref{sec:experiments}}
  \item Did you describe the limitations of your work?
    \answerYes{}
  \item Did you discuss any potential negative societal impacts of your work?
    \answerNA{}
  \item Have you read the ethics review guidelines and ensured that your paper conforms to them?
    \answerYes{}
\end{enumerate}

\item If you are including theoretical results...
\begin{enumerate}
  \item Did you state the full set of assumptions of all theoretical results?
    \answerYes{See \cref{subsec:assumptions}}
        \item Did you include complete proofs of all theoretical results?
    \answerYes{See \cref{app:proofs}}
\end{enumerate}

\item If you ran experiments...
\begin{enumerate}
  \item Did you include the code, data, and instructions needed to reproduce the main experimental results (either in the supplemental material or as a URL)?
    \answerYes{}
  \item Did you specify all the training details (e.g., data splits, hyperparameters, how they were chosen)?
    \answerYes{See \cref{app:exp}}
        \item Did you report error bars (e.g., with respect to the random seed after running experiments multiple times)?
    \answerYes{}
        \item Did you include the total amount of compute and the type of resources used (e.g., type of GPUs, internal cluster, or cloud provider)?
    \answerYes{See \cref{app:exp}}
\end{enumerate}

\item If you are using existing assets (e.g., code, data, models) or curating/releasing new assets...
\begin{enumerate}
  \item If your work uses existing assets, did you cite the creators?
    \answerYes{}
  \item Did you mention the license of the assets?
    \answerYes{}
  \item Did you include any new assets either in the supplemental material or as a URL?
    \answerYes{}
  \item Did you discuss whether and how consent was obtained from people whose data you're using/curating?
    \answerNA{}
  \item Did you discuss whether the data you are using/curating contains personally identifiable information or offensive content?
    \answerNA{}
\end{enumerate}

\item If you used crowdsourcing or conducted research with human subjects...
\begin{enumerate}
  \item Did you include the full text of instructions given to participants and screenshots, if applicable?
    \answerNA{}
  \item Did you describe any potential participant risks, with links to Institutional Review Board (IRB) approvals, if applicable?
    \answerNA{}
  \item Did you include the estimated hourly wage paid to participants and the total amount spent on participant compensation?
    \answerNA{}
\end{enumerate}

\end{enumerate}

\newpage

\appendix
\counterwithin{figure}{section}

\section{Extensive comparison between stochastic methods for bilevel optimization}

We provide here tables summarizing other methods in stochastic bilevel optimization. They are grouped between methods that are based on two nested loops and methods that use only one loop.

In the following tables, the inner iterations are referred with the variable $k$ %
 and the outer iterations are referred with the variable $t$ (or $T$ for the total number of iterations).

 In the literature, there are three main ways to perform Hessian inversion. The HIA, first proposed in \cite{Ghadimi2018}, and SHIA, proposed in \cite{Ji2021a}, procedures used for Hessian inversion are precised in \cref{alg:hia} and~\ref{alg:shia}. These methods are based on Neumann approximation of the inverse of a matrix. SGD for Hessian inversion refers to Stochastic Gradient Descent on $v\mapsto \frac12\langle \nabla^2_{11} G(z,x)v,v\rangle - \langle \nabla_1 F(z,x), v\rangle$. The complexity refers to the number of call to the oracles to get an $\epsilon$-stationary solution. In these complexities, the notation $\tilde{O}$ hide polynomial factors in $\log\epsilon^{-1}$.

\begin{algorithm}
    \begin{algorithmic}
       \STATE {\bfseries Input:} variables $z\in\bbR^p$, $x\in\bbR^d$, gradient $\nabla_1 F(z,x)\in\bbR^p$, maximum number of iterations $b$, a parameter $\eta$.
       \STATE Set $v^0 = \nabla_1 F(z,x)$
       \STATE Choose $p\in\{0,\dots,b-1\}$ randomly.
       \FOR{$k = 1,\dots, p$}
       \STATE
          Sample $i\in\setcomb{n}$

          Update $v$ : $v^{k+1} = (I - \eta\nabla_{11}^2 G(z, x))v^k$
       \ENDFOR
       \STATE {\bfseries Return:} $b\eta v_{p+1}$
    \end{algorithmic}
    \caption{Hessian Inverse Approximation (HIA)}
    \label{alg:hia}
\end{algorithm}

\begin{algorithm}
    \begin{algorithmic}
    \STATE {\bfseries Input:} variables $z\in\bbR^p$, $x\in\bbR^d$, gradient $\nabla_1 F(z,x)\in\bbR^p$, maximum number of iterations $b$, a parameter $\eta$.
    \STATE Set $v^0 = \nabla_1 F(z,x)$
    \STATE Set $s^0 = v^0$
    \FOR{$k=0\dots,b-1$}
    \STATE
        Sample $i\in\setcomb{n}$

        Update $v$: $v^{k+1} = (I - \eta\nabla_{11}^2 G(z, x))v^k$

        Update $s$: $s^{k+1} = s^k + v^{k+1}$
    \ENDFOR

    \STATE {\bfseries Return:} $\eta s^b$
    \end{algorithmic}
    \caption{Summed Hessian Inverse Approximation (SHIA)}
    \label{alg:shia}
\end{algorithm}

The momentum column refers to the use of STORM \cite{Cutkosky2019} momentum in the inner loop or the outer loop. This momentum can be applied to either the inner or the implicit gradient estimate. If we consider the current estimate $y^t = (z^t, v^t, x^t)$ and the previous estimate $y^{t-1} = (z^{t-1}, v^{t-1}, x^{t-1})$, and we apply STORM to the quantity $\phi(y^t)$ with the memory $\hat \phi^t$, the momentum update rule reads
\[
    \hat\phi^{(t+1)} = \eta \phi(y^t) + (1 - \eta) (\hat \phi^t + \phi(y^t) - \phi(y^{t-1}))\enspace .
\]
 Note that this update requires to evaluate the quantity $\phi$ twice per iteration, once in $y^t$ and once in $y^{t-1}$. The memory is need to store the previous estimates $y^{t-1}$ as well as the running estimate of the gradient $\hat \phi$.

\newcolumntype{F}[1]{>{\centering}m{#1}}
\begin{landscape}
\begin{table}
    \centering
        \begin{tabular}{|F{18ex}|F{10ex}|F{14ex}|F{9ex}|F{8ex}|F{8ex}|F{19ex}c|}
        \hline
            \rowcolor{gray!20}
            Method (Two-loops) &  Hessian inversion & Inner loop & Momentum & LR in & LR out & Complexity& \\
            \hline
            BSA \\\cite{Ghadimi2018}& HIA & SGD on inner & No &$O(k^{-1})$&$O(T^{-1/2})$&$O(\epsilon^{-3})$& \\
            \hline
            stocBiO \\ \cite{Ji2021a}& SHIA & SGD on inner & No & Constant &  Constant &$\tilde{O}(\epsilon^{-2})$& \\
            \hline
            VRBO \\ \cite{Yang2021}& SHIA & SPIDER on inner & Yes \\(SPIDER) & Constant &  Constant &$\tilde{O}(\epsilon^{-3/2})$& \\
            \hline
            AmIGO \\ \cite{Arbel2022}& SGD & SGD on inner & No & Constant &  Constant & $O(\epsilon^{-2})$& \\
            \hline
            \rowcolor{gray!20}
            Method (One-loop)&  Hessian inversion & Inner step & Momentum & LR in & LR out & Complexity& \\
            \hline
            TTSA\\ \cite{Hong2021} & HIA & SGD & No  &$O(T^{-2/5})$&$O(T^{-3/5})$&$\tilde{O}(\epsilon^{-5/2})$& \\
            \hline
            SMB\\ \cite{Guo2021a} & HIA & SGD with momentum & Yes &Constant&Constant&$\tilde{O}(\epsilon^{-4})$&  \\
            \hline
            MRBO \\ \cite{Yang2021} & SHIA & SGD with STORM & Yes \\ (STORM) &$O(t^{-1/3})$&$O(t^{-1/3})$&$\tilde{O}(\epsilon^{-3/2})$& \\
            \hline
            STABLE \\ \cite{Chen2022}& Direct & SGD & No &$O(T^{-1/2})$&$O(T^{-1/2})$&$O(\epsilon^{-2})$&  \\
            \hline
            SUSTAIN \\\cite{Khanduri2021}& HIA & SGD with STORM & Yes \\ (STORM) &$O(t^{-1/3})$&$O(t^{-1/3})$&$O(\epsilon^{-3/2})$& \\
            \hline
            SVRB \\ \cite{Guo2021} & Direct + momentum & SGD with momentum & Yes &$O(t^{-1/3})$&$O(t^{-1/3})$&$\tilde{O}(\epsilon^{-3})$&  \\
            \hline
            SBFW\\ \cite{Akhtar2021}& HIA & SGD & No &$O(t^{-1/2})$&$O(T^{-3/4})$&$\tilde{O}(\epsilon^{-4})$& \\
            \hline
            FSLA\\ \cite{Li2022}& SGD with STORM  & SGD with STORM & Yes (STORM) &$O(t^{-1/2})$&$O(T^{-1/2})$&$O(\epsilon^{-2})$& \\
            \hline
            \rowcolor{blue!10}
            \textbf{SOBA} & SGD step & SGD & No &$O(t^{-\new{1/2}})$&$O(t^{-\new{1/2}})$&$O(\epsilon^{-2})$&\\
            \hline
            \rowcolor{blue!10}
            \textbf{SABA} & SAGA step & SAGA & No &Constant&Constant&$\boldsymbol{O((n+m)^{\new{2/3}}\epsilon^{-1})}$&\\
            \hline
        \end{tabular}
        \vspace{1em}
    \caption{Comparison of the stochastic bilevel optimization solvers in the literature.\\The complexity represents the number of oracle calls necessary to attain an $\epsilon$ accurate stationary point.}
    \label{tab:comparison}
\end{table}
\end{landscape}

\section{Details on experiments}\label{app:exp}
We provide here additional informations on the experiments.

\subsection{Generalities}
All the experiments are performed with \texttt{Python}, using the package \texttt{Benchopt}~\cite{Moreau2022} and Numba~\cite{lam2015numba} for fast implementation of stochastic methods. For each problem, we use oracles for a function given function $f$ that $(f(z,x), \nabla_1 f(z, x), \nabla^2_{11} f(z,x)v, \nabla^2_{21} f(z,x)v)$ avoiding duplicate computation of intermediate results for these quantities.

We find that using mini-batches instead of individual samples to compute the stochastic estimates allowed for much faster computations, thanks to hardware acceleration and vectorization of the computations. We use continuous batches to avoid random memory access that slow down the computations. Concretely, if $i_b$ is the index of the current batch and $B$ is the batch-size, the indices of the corresponding samples are those in the set $\{i_b \times B,\dots, (i_b+1)\times B - 1\}$. By doing so, the samples in a same batch are contiguous in memory, which facilitates the access. We use a batch-size of 64 in all experiments.

For the methods involving an inner loop (stocBiO, BSA, AmIGO), we perform 10 inner steps at each outer iteration as proposed in the papers which introduced these methods. For the approximate Hessian vector product, we perform 10 steps per outer iteration for each methods using HIA (BSA, TTSA, SUSTAIN), SHIA (MRBO, stocBiO) or SGD (AmIGO) for the inversion of the linear system.

For the step sizes, they all have the form $\rho^t = \alpha/t^a$ and $\gamma^t = \beta/ t^b$. For the pair of exponents $(a,b)$, we choose the theoretical one from the original papers, that is $(1/2,1/2)$ for BSA and FSLA, $(1/3,1/3)$ for MRBO and SUSTAIN, $(0,0)$ for SABA, AmIGO and stocBiO, $(2/5, 3/5)$ for TTSA and SOBA. For $(\alpha, \beta)$, we perform a grid search (the grid is precised in the subsection dedicated to each experiment) and we keep for each method, the pair $(\alpha, \beta)$ that gives the lowest value of $h$ (for the hyperparameters) or the lowest test accuracy (for the data cleaning task) in median over 10 runs for each possible pair. When we use HIA or SHIA for the Hessian inversion, we set $\eta = \alpha$ since the Hessian inversion problem has the same conditioning as the inner optimization problem.

For the STORM's momentum parameter in MRBO and SUSTAIN, we take $0.5/t^{2/3}$.

\new{For SABA, we have to maintain the estimate $S[\phi,w]^i_t = \phi_i(w^{t+1}_i) - \phi_i(w^t_i) + \frac1n\sum_{i'=1}^n \phi_{i'}(w^t_{i'})$ of $\frac1n\sum_{i=1}^n \phi_i(y^t)$ (see \cref{sec:saba} for the notations). The sum inside $S$ is maintain by performing a rolling mean on the past gradients computed. More precisely, $A_t = \frac1n \sum_{i’=1}^n \phi_{i’}(w_{i’}^t)$. To get $A_{t+1}$, instead of computing the summing all the gradients stored, which has $O(n)$ computational complexity, we do $A_{t+1} = A_t + \frac1n (\phi_i(w^{t+1}_i) - \phi_i(w^t_i))$, which is equivalent mathematically but has $O(1)$ computational complexity.}

\new{
\subsection{Hyperparameter selection on a toy problem}
The \cref{fig:intro_fig} corresponds to the methods SABA et SOBA applied to an hyperparameter selection problem for a Ridge regression. We generate $1000$ samples $x_1,\dots,x_{1000}\in\bbR^10$ for $\calN(0, I_10)$. We generate a parameter $\beta\sim\calN(0, I_{10})$ and do $y = (X\odot W)\beta + \epsilon$ where $\epsilon \sim\calN(0, 0.01I_{10})$ and the entries of $W$ have the form $W_{i,j} = 1 + u_jv_{i,j}$ with $v_{i,j}\sim\calU([0, 1])$ and $u_j\sim\calU([0, 1])$ if $1\leq j \leq 5$ or $u_j\sim\calU([0, 10])$ if $6\leq j \leq 10$. Then we use 750 pairs $(x_i^\mathrm{train}, y_i^\mathrm{train})_{1\leq i\leq 750}$ as training samples and the remaining pairs $(x_i^\mathrm{val}, y_i^\mathrm{val})_{1\leq i\leq 250}$ as validation samples. Finally, we solve \eqref{eq:pb} with
$$F(\theta, \lambda) = \frac1{2n_\mathrm{val}}\sum_{i=1}^{n_\mathrm{val}}((x_i^\val)^\top\theta - y_i^\val)^2$$ and $$G(\theta, \lambda) = \frac1{2n_\mathrm{train}}\sum_{i=1}^{n_\mathrm{train}}((x_i^\train)^\top\theta - y_i^\train)^2+\frac\lambda2\|\theta\|^2$$
with $n_\mathrm{train} = 750$ and $n_\mathrm{val} = 250$.
}
\subsection{Hyperparameters selection on IJCNN1}\label{app:ijcnn}
In this experiment, we select the parameters regularization for a multiregularized logistic regression model precised in Equations~\eqref{eq:test_loss} and~\eqref{eq:train_loss} where we have one hyperparameter per feature
\begin{align}
    \label{eq:test_loss}
    F(\theta, \lambda) &= \frac1m\sum_{i=1}^m \varphi( y^\mathrm{val}_i\langle d^\mathrm{val}_i, \theta\rangle)\enspace\text{and} \\
\label{eq:train_loss}
G(\theta, \lambda) &= \frac1n\sum_{i=1}^n \varphi(y^\mathrm{train}_i\langle  d^\mathrm{train}_i, \theta\rangle) + \frac12\theta^\top\diag(e^{\lambda_1},\dots,e^{\lambda_p})\theta\enspace.
\end{align}
Note that the parametrization in $e^\lambda$  of the penalty instead of $\lambda$ can be surprising at first glance, but it is classical in the bilevel optimization literature \cite{Pedregosa2016, Ji2021a, Grazzi2021} because it avoids positivity constraints on $\lambda$.
In order to choose the select proper parameters $(\alpha,\beta)$ for each algorithm, we perform a grid search. We search $\alpha$ in a set of 9 values between $2^{-5}$ and $2^3$ spaced on a log scale. For $\beta$, we choose $r$ in a set of 7 values between $10^{-2}$ and $10$ spaced on a logarithmic scale and we set $\beta = \frac\alpha{r}$.

For this experiments, we use Just-In-Time (JIT) compilation thanks to the package Numba \cite{lam2015numba}, to decrease the python overhead in the iteration loop.

To evaluate the value function $h$, we use L-BFGS \cite{Liu1989} to solve compute $z^*(x^t)$ and then evaluate the function $h(x^t) = F(z^*(x^t), x^t)$.

\begin{figure}[h]
        \centering
        \savebox{\largestimage}{\includegraphics[width=.345\linewidth]{figures/ijcnn1_cr.pdf}}
        \hfill
        \begin{subfigure}[b]{.345\linewidth}
            \centering
             \usebox{\largestimage}
        \end{subfigure}
        \hfill
        \begin{subfigure}[b]{.345\linewidth}
            \raisebox{\dimexpr.5\ht\largestimage-.5\height}{%
            \includegraphics[width=\linewidth]{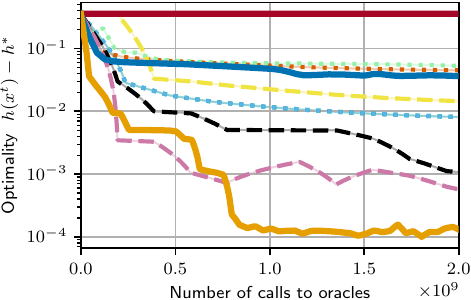}}
        \end{subfigure}
        \hfill
        \begin{subfigure}[t]{.29\linewidth}
            \centering
            \raisebox{\dimexpr .8\ht\largestimage-.9\height}{%
            \includegraphics[width=\linewidth]{figures/legend_neurips.pdf}}
        \end{subfigure}
        \hfill
        \caption{\new{Comparison of SOBA and SABA with other stochastic bilevel optimization methods in a problem of hyperparameter selection for $\ell^2$ penalized logistic regression on IJCNN1 dataset. For each algorithm, we plot the median performance over 10 runs. In both plots, SABA achieves the best performance. The dashed lines are for one loop competitor methods, the dotted lines are for two loops methods and the solid lines are  the proposed methods.
        \textbf{Left}: performance in running time, \textbf{Right}: performance in number of gradient/Hessian-vector products sampled.}
        }
        \label{app:fig:ijcnn}
\end{figure}

\subsection{Data hyper-cleaning}\label{app:cleaning}

For the regularization parameter $C_r$, we choose $C_r = 0.2$ after a manual search in order to get the best final test accuracy.

In this experiment, the selection of the good pair $(\alpha, \beta)$ is also performed by grid search. The parameter $\alpha$ is picked in a set of 11 numbers between $10^{-3}$ and $100$ spaced on a logarithmic scale. For $\beta$, we choose $r$ in a set of 11 values between $10^{-5}$ and $1$ spaced on a logarithmic scale and we set $\beta = \frac\alpha{r}$.

Note that in this case, we could not use JIT from Numba since at the moment of the experiment, the softmax function coming from Scipy was not compatible with Numba.

We report in \cref{fig:app:expes_app} some additional convergence curves with different corruption probabilities $p\in\{0.5, 0.7, 0.9\}$ (the figure in the main text corresponds to $p=0.5$). SABA is always the fastest algorithm to reach its final accuracy.

\begin{figure}[h]
        \centering
        \hfill
        \begin{subfigure}[b]{.3\linewidth}
            \centering
            \includegraphics[width=\linewidth]{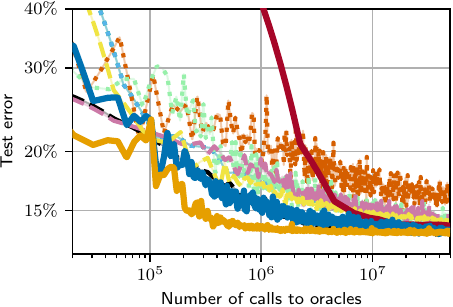}
        \end{subfigure}
        \hfill
        \begin{subfigure}[b]{.3\linewidth}
            \includegraphics[width=\linewidth]{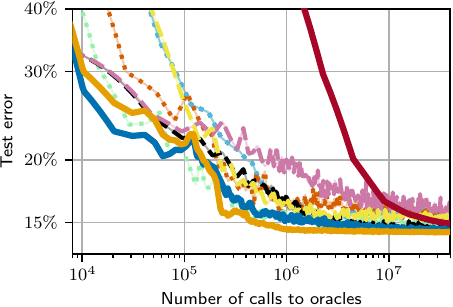}
        \end{subfigure}
        \hfill
        \begin{subfigure}[b]{.3\linewidth}
            \centering
            \includegraphics[width=\linewidth]{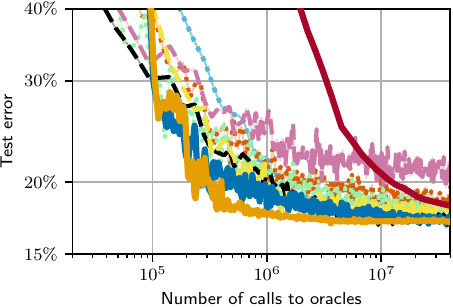}
        \end{subfigure}
        \hfill

        \hfill
        \begin{subfigure}[b]{.3\linewidth}
            \centering
            \includegraphics[width=\linewidth]{figures/datacleaning0_5_cr.pdf}
            \caption{$p=0.5$}
        \end{subfigure}
        \hfill
        \begin{subfigure}[b]{.3\linewidth}
            \includegraphics[width=\linewidth]{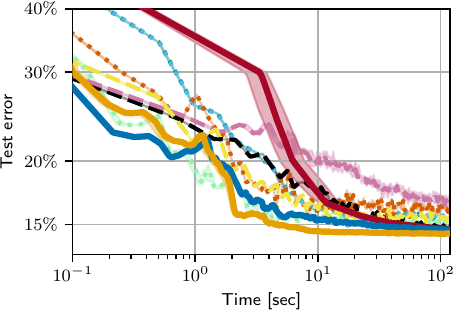}
            \caption{$p=0.7$}
        \end{subfigure}
        \hfill
        \begin{subfigure}[b]{.3\linewidth}
            \centering
            \includegraphics[width=\linewidth]{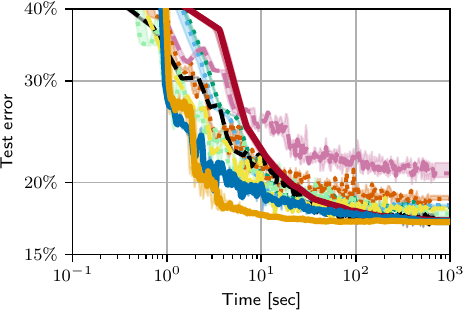}
            \caption{$p=0.9$}
        \end{subfigure}
        \hfill

        \vspace{1em}
        \begin{subfigure}[b]{.8\linewidth}
            \centering
            \includegraphics[width=\linewidth]{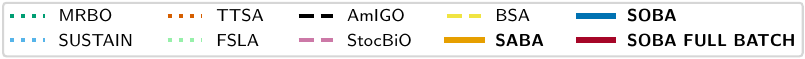}
        \end{subfigure}

        \caption{Datacleaning experiment, with different corruption probability (higher means that more data are contamined). \new{\textbf{Top: } Performance with respect to the number of gradient/Hessian-vector product sampled, \textbf{Bottom: } Performance with respect to running time}}
        \label{fig:app:expes_app}
\end{figure}

\new{
\subsection{Additional experiment: Hyperparameter selection on the covtype dataset}
We also perform an additional experiment which consists in selecting the best regularization parameter for a $\ell^2$-regularized multinomial logistic regression problem on the covtype dataset\footnote{\url{https://scikit-learn.org/stable/modules/generated/sklearn.datasets.fetch_covtype.html}}. This dataset contains $581,012$ samples with $p = 54$ features and there are $C = 7$ classes. We used $n = 371,847$ train samples, $m = 92,962$ samples and $n_\test = 116,203$ test samples. We fit a multiclass logistic regression on this dataset, with one hyperparameter per class. This means that, if $(d^\train_i, y^\train_i)_{i\in\setcomb{n}}$ and $(d^\val, y^\val)_{i\in\setcomb{m}}$ are respectively the training samples and the validation samples, we solve the Problem \eqref{eq:pb} with
\begin{align*}
F(\theta, \lambda) &= \frac1m\sum_{i=1}^m \ell(\theta d^\mathrm{val}_i, y^\mathrm{val}_i)\enspace\text{and}\\
G(\theta, \lambda) &= \frac1n\sum_{i=1}^n \ell(\theta d^\mathrm{train}_i, y^\mathrm{train}_i) + \sum_{c=1}^Ce^{\lambda_c}\sum_{i=1}^p\theta_{i,c}^2
\end{align*}
where $\theta\in\bbR^{p\times C}$ and $\lambda \in\bbR^C$.
}

\new{
As for the other experiments, we performed and grid search over 63 pairs $(\alpha, \beta)$ to set the step sizes. The parameter $\alpha$ is chosen among values between $2^{-5}$ and $2^3$ spaced in log scale. For $\beta$, we choose it in a set of values between $10^{-2}$ and $10$ spaced in log scale. We used a batch size of 64. The experiment took 525 CPU hours.
}

\new{
We show in \cref{app:fig:covtype} the error on the test samples with respect to the running time and the number of gradients/Hessian-vector products sampled. We observe that SABA and SOBA achieve the best performances. The initial gap between the first and the second plot for SABA is due to the overhead of the initialization of the memory. This gap can be reduced by increasing the batch size.
\begin{figure}[h]
        \centering
        \savebox{\largestimage}{\includegraphics[width=.345\linewidth]{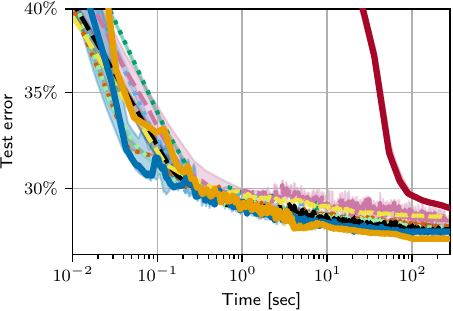}}
        \hfill
        \begin{subfigure}[b]{.345\linewidth}
            \centering
             \usebox{\largestimage}
        \end{subfigure}
        \hfill
        \begin{subfigure}[b]{.345\linewidth}
            \raisebox{\dimexpr.5\ht\largestimage-.5\height}{%
            \includegraphics[width=\linewidth]{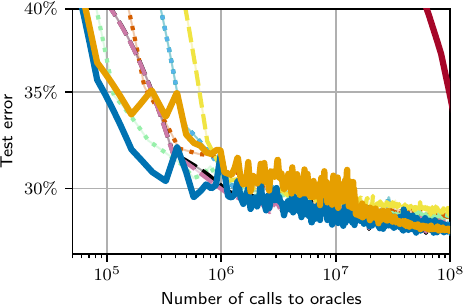}}
        \end{subfigure}
        \hfill
        \begin{subfigure}[t]{.29\linewidth}
            \centering
            \raisebox{\dimexpr .8\ht\largestimage-.9\height}{%
            \includegraphics[width=\linewidth]{figures/legend_neurips.pdf}}
        \end{subfigure}
        \hfill
        \caption{\new{Comparison of SOBA and SABA with other stochastic bilevel optimization methods in a problem of hyperparameter selection for $\ell^2$ penalized multical logistic regression on covtype dataset. For each algorithm, we plot the median performance over 10 runs. The dashed lines are for one loop competitor methods, the dotted lines are for two loops methods and the solid lines are  the proposed methods.
        \textbf{Left}: performance in running time, \textbf{Right}: performance in number of gradient/Hessian-vector products sampled.}
        }
        \label{app:fig:covtype}
\end{figure}
}
\newpage

\section{Proofs}\label{app:proofs}

\subsection{Proof of \cref{prop:zeros_directions}}
\begin{proof}
Let $(z,v,x)$ a zero of $(D_z, D_v, D_x)$.
For $D_z$, this means that $\nabla_1 G(z, x) = 0$. Since $G(\,\cdot\,, x)$ is strongly convex, $z$ is the minimizer of $G(\,\cdot\,, x)$, \textit{i.e.} $z=z^*(x)$.
The fact that $(z,v,x)$ is a zero of $D_v$ implies that $\nabla_{11}^2 G(z,x)v = - \nabla_1 F(z, x)$. Replacing $z$ by its value, we get $v = - \left[\nabla_{11}^2 G(z^*(x),x)\right]^{-1}\nabla_1 F(z^*(x), x)$ which is $v^*(x)$ by definition.
Putting all together and using the expression of $\nabla h(x)$ given by \eqref{eq:hgrad}, we get
$$D_x(z,v,x) = \nabla_2 F(z^*(x),x) + \nabla_{21} G(z^*(x),x)v^*(x) = \nabla h(x)\enspace .$$
On the other hand, $D_x(z,v,x) = 0$ so $\nabla h(x) = 0$.
\end{proof}

\subsection{Proof of \cref{lemma:smoothness}}

\begin{proof}
Let $(z, v, x)\in\bbR^p\times\bbR^p\times\bbR^d$. Using the fact that $\nabla_1 G(z^*(x), x) = 0$ and the $L^G_1$-smoothness of $G(\,\cdot\,, x)$, we have
\begin{align*}
\boxed{
    \| D_z(z,v,x)\|^2 = \| \nabla_1 G(z, x) - \nabla_1 G(z^*(x), x)\|^2 \leq L^2_G\| z - z^*(x)\|^2\enspace .
    }
\end{align*}

For $D_v$, since $\nabla^2_{11} G(z^*(x), x)v^*(x) = -\nabla_1 F(z^*(x), x)$, we write
\begin{align}
    \|D_v\| &= \| (\nabla_{11}^2 G(z,x) v + \nabla_1 F(z, x)) - (\nabla_{11}^2 G(z^*(x),x) v^*(x) + \nabla_1 F(z^*(x), x))\|\\
    &\leq \|[\nabla^2_{11} G(z,x) - \nabla^2_{11} G(z^*(x),x)] v^*(x)\|  +\| \nabla_{11}^2 G(z,x)[v-v^*(x)]\|\\\nonumber
    &\quad + \|\nabla_1 F(z,x) - \nabla_1 F(z^*(x), x)\|\enspace .
\end{align}

For the first term, we use the Lipschitz continuity of $\nabla^2_{11} G$:
$$
 \|[\nabla^2_{11} G(z,x) - \nabla^2_{11} G(z^*(x),x)] v^*(x)\| \leq L^G_2 \|z - z^*(x)\|\|v^*(x)\|\enspace .
$$
Then, since $G$ in $\mu_G$-strongly convex w.r.t. $z$, $\nabla_1 F(z^*(\cdot),\,\cdot)$ is bounded and $v^*(x) = -[\nabla^2_{11}G(z^*(x),x)]^{-1}\nabla_1 F(z^*(x), x)$, we have
\begin{equation}\label{app:eq:bound_hess_g_v_star}
\|[\nabla^2_{11} G(z,x) - \nabla^2_{11} G(z^*(x),x)] v^*(x)\|\leq \frac{L^G_2 C_F}{\mu_G} \|z - z^*(x)\|\enspace .
\end{equation}

For the second term, we use the $L^G_1$-smoothness of $G(\,\cdot\,, x)$ and for the third term, we use the $L^F_1$-smoothness of $F$ and we finally get

\begin{align}
    \| D_v\| \leq \left(\frac{L^G_2 C_F}{\mu_G}+L^F_1\right) \|z - z^*(x)\| + L^G_1\|v - v^*(x)\| \enspace .
\end{align}

Then, taking $L_v = \sqrt{2}\max\left(\frac{L^G_2 C_F}{\mu_G}+L^F, L^G_1\right)$, we get
\begin{align}
\boxed{
    \| D_v(z, v, x)\|^2 \leq L_v^2(\|z-z^*(x)\|^2 + \|v-v^*(x)\|^2)\enspace .}
\end{align}

For $D_x(z,v,x) - \nabla h(x)$ we start by writing

\begin{align}
    \|D_x(z,v,x) - \nabla h(x)\| &\leq \| \nabla_2 F(z, x) - \nabla_2 F(z^*(x), x)\| + \| \nabla_{21}^2 G(z,x)v - \nabla_{21}^2 G(z^*(x), x)v^*(x)\|\\
    &\leq \| \nabla_2 F(z, x) - \nabla_2 F(z^*(x), x)\| + \|\nabla_{21}^2 G(z,x)\| \|v-v^*(x)\|\\\nonumber
    &\qquad + \|v^*(x)\|\|\nabla_{21}^2 G(z,x) - \nabla_{21}^2 G(z^*(x), x)\|\enspace .
\end{align}

We bound the first term using the fact that $\nabla_2 F$ is $L^F_1$-Lipschitz continuous. For the second term, the fact that $\nabla_{21}^2 G$ is bounded thanks to the Lipschitz continuity of $\nabla_1 G(z, \,\cdot\,)$. For the third term, we use that $\nabla_{21}^2 G(\,\cdot\,,x)$ is $L^G_2$-Lipschitz continuous and the same derivation as \cref{app:eq:bound_hess_g_v_star}. We finally get

\begin{align}
    \| D_x - \nabla h(x)\| \leq \left(L^F_1+\frac{C_FL^G_2}{\mu_G}\right)\|z-z^*(x)\| + L^G_1\|v-v^*(x)\|\enspace .
\end{align}

Taking $L_x = \sqrt{2}\max\left(L^F_1+\frac{C_FL^G_2}{\mu_G}, L^G_1\right)$ yields
\begin{equation}
\boxed{
     \| D_x(z, v, x) - \nabla h(x)\|^2 \leq L_x^2 (\|z-z^*(x)\|^2 + \|v-v^*(x)\|^2)\enspace .}
\end{equation}
\end{proof}

\subsection{Smoothness constant of $h$}\label{app:cst_h}
From \citet[Lemma 2.2]{Ghadimi2018}, we get the \cref{lemma:descent_lemma} which states the $L^h$-smoothness of $h$ with $$
L^h = L^F_1+ \frac{2L^F_1L^G_2+C_F^2 L^G_2}{\mu_G}
+\frac{L_{11}^GL^G_1C_F+L^G_1L^G_2C_F + (L^G_1)^2L^F_1}{\mu_G^2}
+\frac{(L^G_1)^2L^G_2C_F}{\mu_G^3}\enspace.
$$

\subsection{Lemmas on the regularity of $z^*$ and $v^*$}

We start by showing the Lipschitz continuity of $z^*$ and $v^*$.

\begin{lemma}\label{lemma:smoothness_star}
There exists a constant $L_*>0$ such that for any $x_1, x_2\in\bbR^d$ we have
$$
\|z^*(x_1)-z^*(x_2)\|\leq L_*\|x_1-x_2\|,\quad \|v^*(x_1)-v^*(x_2)\|\leq L_*\|x_1-x_2\|\enspace .
$$
\end{lemma}
\begin{proof}
 Let $x\in\bbR^d$. The Jacobian of $z^*$ is given by $\diff z^*(x) = -[\nabla_{11}^2 G(z^*(x), x)]^{-1}\nabla^2_{1,2} G(z^*(x), x)$. Thanks to the $\mu_G$-strong convexity of $G$ and the fact that $\nabla_{21}^2 G$ is bounded, we have
 $
 \|\diff z^*(x)\|\leq \frac{L^G_1}{\mu_G}\enspace.
 $
 Thus, $z^*$ is Lipschitz continuous.

 For $\norme{v^*(x_1)-v^*(x_2)}$, we start from the definition $v^*$:
    \begin{align}
        \|v^*(x_1)-v^*(x_2)\| &= \|[\nabla^2_{11} G(z^*(x_1), x_1)]^{-1}\nabla_1 F(z^*(x_1), x_1) - [\nabla^2_{11} G(z^*(x_2), x_2)]^{-1}\nabla_1 F(z^*(x_2), x_2)\|\\
        &\leq  \|([\nabla^2_{11} G(z^*(x_1), x_1)]^{-1}-[\nabla^2_{11} G(z^*(x_2), x_2)]^{-1}\nabla_1 F(z^*(x_1), x_1)\|\\\nonumber
        &\qquad + \|[\nabla^2_{11} G(z^*(x_2), x_2)]^{-1}(\nabla_1 F(z^*(x_2), x_2)-\nabla_1 F(z^*(x_1), x_1))\|\enspace .
    \end{align}

For the first term, we use that for any invertible matrix $A$ and $B$ we have $A^{-1}-b^{-1} = A^{-1}(B-A)B^{-1}$ to get

\begin{align*}
\|[\nabla^2_{11} G(z^*(x_1), x_1)]^{-1}-\nabla^2_{11} G(z^*(x_2), x_2)]^{-1}\| &= \|[\nabla^2_{11} G(z^*(x_1), x_1)]^{-1}(\nabla^2_{11} G(z^*(x_2), x_2)]-\\\nonumber
&\qquad\nabla^2_{11} G(z^*(x_1), x_1)])[\nabla^2_{11} G(z^*(x_2), x_2)]^{-1}\|\\
&\leq \frac{1}{\mu_G^2} \|\nabla^2_{11} G(z^*(x_1), x_1)-\nabla^2_{11} G(z^*(x_2), x_2)\|\\
&\leq \frac{L^G_2}{\mu_G^2}\|(z^*(x_1),x_1) - (z^*(x_2),x_2)\|\\
&\leq \frac{L^G_2}{\mu_G^2}[\|z^*(x_1) - z^*(x_2)\| + \| x_1 -  x_2\|]\\
&\leq \frac{L^G_2}{\mu_G^2}\left[1 +\frac{L^G_1}{\mu_G}\right]\| x_1 -  x_2\|\enspace .
\end{align*}

And then, since $\nabla_1 F(z^*(\,\cdot\,), \,\cdot\,)$ is bounded:
$$
\norme{([\nabla^2_{11} G(z^*(x_1), x_1)]^{-1}-[\nabla^2_{11} G(z^*(x_2), x_2)]^{-1}\nabla_1 F(z^*(x_1), x_1)}\leq \frac{C_FL^G_2}{\mu_G^2}\left[1 +\frac{L^G_1}{\mu_G}\right]\| x_1 -  x_2\|\enspace .
$$

For the second term, the strong convexity of $G(\,\cdot\,, x)$ and the fact that $\nabla_1 F$ is Lipschitz continuous lead to
\begin{align}
    \|[\nabla^2_{11} G(z^*(x_2), x_2)]^{-1}(\nabla_1 F(z^*(x_2), x_2)-\nabla_1 F(z^*(x_1), x_1))\|&\leq \frac{1}{\mu_G}\|\nabla_1 F(z^*(x_2), x_2)-\nabla_1 F(z^*(x_1), x_1)\|\\
    &\leq \frac{L^F_1}{\mu_F}\|(z^*(x_1),x_1) - (z^*(x_2),x_2)\|\\
    &\leq \frac{L^F_1}{\mu_G}[\|z^*(x_1)-z^*(x_2)\| + \|x_1-x_2\|]\\
    &\leq \frac{L^F_1}{\mu_G}\left[1 +\frac{L^G_1}{\mu_G}\right]\| x_1 -  x_2\|\enspace .
\end{align}

Then we get
\begin{align}
        \|v^*(x_1)-v^*(x_2)\|\leq \left[ \frac{C_FL^G_2}{\mu_G^2}\left[1 +\frac{L^G_1}{\mu_G}\right]+\frac{L^F_1}{\mu_G}\left[1 +\frac{L^G_1}{\mu_G}\right]\right]\|x_1-x_2\|\enspace.
\end{align}

We conclude by setting
$$
L_* = \max\left(\frac{L^G_1}{\mu_G}, \frac{C_FL^G_2}{\mu_G^2}\left[1 +\frac{L^G_1}{\mu_G}\right]+\frac{L^F_1}{\mu_G}\left[1 +\frac{L^G_1}{\mu_G}\right]\right)\enspace .
$$
\end{proof}

In what follows, we denote by $\bbE_t[\,\cdot\,]$ the expectation conditionally on $z^t$, $v^t$ and $x^t$.

\new{
We have the smoothness property of $z^*$ provided in \citep[Lemma 2]{Chen2021b}.
\begin{lemma}\label{app:lemma:smoothness_z_star}
    Under the Assumptions \ref{ass:1}, \ref{ass:2} and \ref{ass:3}, the function $z^*:\bbR^d\to\bbR^p$ is $L_{zx}$-smooth with
    \begin{equation}\label{eq:smoothness_z_star}
        L_{zx} = \frac{L^G_2(1+L_*)}{\mu_G}+\frac{L_1^GL_{11}^G(1+L_*)}{\mu_G^2}\enspace.
    \end{equation}
\end{lemma}
}

\new{
We establish the same result for $v^*$. To this, we need more regularity on $G$ and $F$.
\begin{lemma}\label{app:lemma:smoothness_v_star}
    The function $v^*:\bbR^d\to\bbR^p$ is differentiable and its differential is defined for any $x, \epsilon \in\bbR^d$ by:
    \begin{align}\label{eq:diff_v_star}
        \diff v^*(x).\epsilon &= [\nabla^2_1G(z^*(x),x)]^{-1}[\nabla^2_{11}F(z^*(x), x)\diff z^*(x).\epsilon + \nabla^2_{12}F(z^*(x), x).\epsilon]\\\nonumber
        &\qquad - [\nabla^2_1G(z^*(x),x)]^{-1}[(\nabla^3_{111}G(z^*(x),x)|\diff z^*(x).\epsilon) + (\nabla^3_{112}G(z^*(x),x)|\epsilon)]\\\nonumber
        &\qquad \times[\nabla^2_1G(z^*(x),x)]^{-1}\nabla_1F(z^*(x), x)
    \end{align}
    where for any $z, \alpha\in\bbR^p$ and $x\in\bbR^d$, $(\nabla^3_{111}G(z,x)|\alpha)\in\bbR^{p\times p}$ is defined by
    $$
    (\nabla^3_{111}G(z,x)|\alpha) = \left[\sum_{k=1}^p\frac{\partial^3 G}{\partial z_i\partial z_j\partial z_k}(z, x)\alpha_k\right]_{1\leq i,j\leq p}
    $$
    and for any $\beta\in\bbR^d$, $(\nabla^3_{112}G(z,x)|\beta)\in\bbR^{p\times p}$ is defined by
    $$
    (\nabla^3_{112}G(z,x)|\beta) = \left[\sum_{k=1}^p\frac{\partial^3 G}{\partial z_i\partial z_j\partial x_k}(z, x)\beta_k\right]_{1\leq i,j\leq p}\enspace .
    $$
    Moreover, $\diff v^*$ is $L_{vx}$-Lipschitz continuous.
\end{lemma}
\begin{proof}
    Let $x, \epsilon\in\bbR^d$. Using the differentiability of $\nabla_{11}^2 G$, $\nabla_1 F$ and of the matrix inversion, we have
    \begin{align*}
        v^*(x+\epsilon) &= [\nabla^2_{11}G(z^*(x+\epsilon), x+\epsilon)]^{-1}\nabla_1F(z^*(x+\epsilon), \epsilon)\\
        &= [\nabla^2_{11}G(z^*(x),x) + (\nabla_{111}^3 G(z^*(x), x)|\diff z^*(x).\epsilon) + (\nabla^3_{112}G(z^*(x), x)|\epsilon) + o(\|\epsilon\|)]^{-1}\\\nonumber
        &\qquad \times (\nabla_1 F(z^*(x), x) + \nabla^2_{11}F(z^*(x),x)\diff z^*(x).\epsilon + \nabla^2_{12}F(z^*(x),x)\epsilon + o(\|\epsilon\|))\\
        &= \left\{[\nabla^2_{11}G(z^*(x),x)]^{-1} \right.\\\nonumber
        &\qquad\left.- [\nabla^2_{11}G(z^*(x),x)]^{-1}[(\nabla_{111}^3 G(z^*(x), x)|\diff z^*(x).\epsilon) + (\nabla^3_{112}G(z^*(x), x)|\epsilon)]\right.\\\nonumber
        &\qquad\left.\times[\nabla^2_{11}G(z^*(x),x)]^{-1} + o(\|\epsilon\|)\right\}
        \\\nonumber
        &\qquad \times (\nabla_1 F(z^*(x), x) + \nabla^2_{11}F(z^*(x),x)\diff z^*(x).\epsilon + \nabla^2_{12}F(z^*(x),x)\epsilon + o(\|\epsilon\|))\\
        &= v^*(x) + [\nabla^2_1G(z^*(x),x)]^{-1}[\nabla^2_{11}F(z^*(x), x)\diff z^*(x).\epsilon + \nabla^2_{12}F(z^*(x), x).\epsilon]\\\nonumber
        &\qquad - [\nabla^2_1G(z^*(x),x)]^{-1}[(\nabla^3_{111}G(z^*(x),x)|\diff z^*(x).\epsilon) + (\nabla^3_{112}G(z^*(x),x)|\epsilon)][\nabla^2_1G(z^*(x),x)]^{-1}\\\nonumber
        &\qquad \times \nabla_1F(z^*(x), x) + o(\|\epsilon\|)
    \end{align*}
    that proves \eqref{eq:diff_v_star}.
    Now, let $x, y, \epsilon\in\bbR^d$ with $\|\epsilon\| = 1$. Let us denote
    $$A(x, \epsilon) = - [\nabla^2_1G(z^*(x),x)]^{-1}[(\nabla^3_{111}G(z^*(x),x)|\diff z^*(x).\epsilon) + (\nabla^3_{112}G(z^*(x),x)|\epsilon)][\nabla^2_1G(z^*(x),x)]^{-1}
    $$ and
    $$
    B(x,\epsilon) =  \nabla^2_{11}F(z^*(x),x)\diff z^*(x).\epsilon + \nabla^2_{12}F(z^*(x),x)
    $$
    so that $\diff v^*(x).\epsilon = [\nabla^2_{11}G(z^*(x),x)]^{-1}B(x,\epsilon) + A(x, \epsilon)\nabla_1 F(z^*(x), x)$.
    We have
    \begin{align}
        (\diff v^*(x) - \diff v^*(y)).\epsilon &=  [\nabla^2_{11}G(z^*(x),x)]^{-1}B(x,\epsilon) + A(x, \epsilon)\nabla_1 F(z^*(x), x) \\\nonumber
        &\qquad - [\nabla^2_{11}G(z^*(y),y)]^{-1}B(y,\epsilon) - A(y, \epsilon)\nabla_1 F(z^*(y), y)\\
        &= [\nabla^2_{11}G(z^*(x),x)]^{-1}(B(x,\epsilon) - B(y,\epsilon))\\\nonumber
        &\qquad+ ([\nabla^2_{11}G(z^*(x),x)]^{-1} - [\nabla^2_{11}G(z^*(y),y)]^{-1})B(y,\epsilon)\\\nonumber
        &\qquad+ A(x,\epsilon)(\nabla_1 F(z^*(x), x) - \nabla_1 F(z^*(y),y)) \\\nonumber
        &\qquad+ (A(x, \epsilon) - A(y,\epsilon))\nabla_1 F(z^*(y), y)\enspace .
    \end{align}
    We can now bound each term using the regularity assumptions on $G$ and $F$:
    \begin{align}
        \|[\nabla^2_{11}G(z^*(x),x)]^{-1}(B(x,\epsilon) - B(y,\epsilon))\|&\leq \frac{1}\mu_G (\|\nabla^2_{11}F(z^*(x),x)\diff z^*(x)-\nabla^2_{11}F(z^*(y),y)\diff z^*(y)\|\\\nonumber
        & +\|\nabla^2_{12}F(z^*(x),x) -\nabla^2_{12}F(z^*(y),y)\|)\\
        &\leq \frac{1}\mu_G (\|\nabla^2_{11}F(z^*(x),x)-\nabla^2_{11}F(z^*(y),y)\|\|\diff z^*(x)\|\\\nonumber
        &\qquad+ \|\diff z^*(x) - \diff z^*(y)\|\|\nabla^2_{11}F(z^*(y), y)\|\\\nonumber
        &\qquad+L^F_2(\|z^*(x)-z^*(y)\|+\|x-y\|)\\
        &\leq \frac{1}\mu_G (L^F_2L_*(1+L_*) + L_{zx}L^F_1+L^F_2(1 + L_*))\|x-y\|\\
    \end{align}
    For the second term:
    \begin{align}
        \|([\nabla^2_{11}G(z^*(x),x)]^{-1} - [\nabla^2_{11}G(z^*(y),y)]^{-1})B(y,\epsilon)\| &\leq \frac{1}{\mu_G^2} \|\nabla^2_{11}G(z^*(x),x)-\nabla^2_{11}G(z^*(y),y)\|\|B(y,\epsilon)\|\\
        &\leq \frac{1}{\mu_G^2} \|\nabla^2_{11}G(z^*(x),x)-\nabla^2_{11}G(z^*(y),y)\| \\\nonumber
        &\qquad \times(\|\nabla^2_{11} F(z^*(x), x)\|\|\diff z^*(x)\| + \|\nabla^2_{12} F(z^*(x),x)\|)\\
        &\leq \frac{(L^G_2+L^F_1)(L_*+1)}{\mu_G^2}\|x-y\|
    \end{align}
    For the third term, we have:
    \begin{align}
        \|A(x,\epsilon)(\nabla_1 F(z^*(x), x) - \nabla_1 F(z^*(y),y))\| &\leq \frac{L^F_1(1+L^*)}{\mu_G^2}\|(\nabla^3_{111}G(z^*(x),x)|\diff z^*(x).\epsilon) \\\nonumber+ &\qquad(\nabla^3_{112}G(z^*(x),x)|\epsilon)\|\|x-y\|\\
        &\leq \frac{(L^F_1+L^G_2)(1+L^*)}{\mu_G^2}\|x-y\|
    \end{align}
    And finally, for the forth term:
    \begin{align}
        \|(A(x, \epsilon) - A(y,\epsilon))\nabla_1 F(z^*(y), y)\| &\leq C_F\{\|[\nabla^2_{11}G(z^*(x),x)]^{-1}\|\\\nonumber
        &\qquad\times \|(\nabla^3_{111}G(z^*(x),x)|\diff z^*(x).\epsilon) + (\nabla^3_{112}G(z^*(x),x)|\epsilon)\|\\\nonumber
        &\qquad\times \|[\nabla^2_{11}G(z^*(x),x)]^{-1}-[\nabla^2_{11}G(z^*(y),y)]^{-1}\| \\\nonumber
        &\qquad + \|[\nabla^2_{11}G(z^*(x),x)]^{-1}-[\nabla^2_{11}G(z^*(y),y)]^{-1}\|\\\nonumber
        &\qquad \times \|(\nabla^3_{111}G(z^*(x),x)|\diff z^*(x).\epsilon) + (\nabla^3_{112}G(z^*(x),x)|\epsilon)\|\\\nonumber
        &\qquad \times\|[\nabla^2_{11}G(z^*(y),y)]^{-1}\|\\\nonumber
        &\qquad + \|[\nabla^2_{11}G(z^*(y),y)]^{-1}\|^2\\\nonumber
        &\qquad \times (\|(\nabla^3_{111}G(z^*(x),x)|\diff z^*(x).\epsilon) - (\nabla^3_{111}G(z^*(y),y)|\diff z^*(y).\epsilon)\|\\
        &\qquad\qquad \|(\nabla^3_{112}G(z^*(x),x)|\epsilon)-(\nabla^3_{112}G(z^*(y),y)|\epsilon)\| )\}\\\nonumber
        &\leq C_F\left\{2\frac{2L^G_2(1+L^*)}{\mu_G^3} + \frac{L^G_3(1+L^*)}{\mu_G^2}\right\}\|x-y\|
    \end{align}
    Thus $v^*$ is $L_{vx}$-smooth with
    $$
    L_{vx} = \frac{L^F_2L_*(1+L_*) + L_{zx}L^F_1+L^F_2(1 + L_*)}{\mu_G} + 2\frac{(L^G_2+L^F_1)(L_*+1)}{\mu_G^2} + \frac{C_FL^G_3(1+L^*)}{\mu_G^2} + 4\frac{C_FL^G_2(1+L^*)}{\mu_G^3}\enspace .
    $$
\end{proof}
}

\subsection{Proof of \cref{lemma:coupled_inequalities}}

We now provide the proof of \cref{lemma:coupled_inequalities}.
\begin{proof}
\textbf{Inequality for $\delta_z$.\quad}

\new{We start by expanding the square:
\begin{align}\label{app:eq:expansion_delta_z}
    \|z^{t+1} - z^*(x^{t+1})\|^2 &=  \|z^{t+1} - z^*(x^t)\|^2 + \|z^*(x^{t+1}) - z^*(x^t)\|^2 \\\nonumber
    &\qquad- 2 \langle z^{t+1} - z^*(x^t), z^*(x^{t+1}) - z^*(x^t)\rangle
\end{align}
}

We study each member, using the unbiasedness of $D^t_z$ and the $\mu_G-$strong convexity of $G(\,\cdot\,, x^t)$:
\begin{align}
\bbE_t[\|z^{t+1} - z^*(x^t)\|^2] &= \bbE_t[\|z^t - z^*(x^t)\|^2]-2\rho\bbE_t[\langle D^t_z, z^t - z^*(x^t)\rangle] + \rho^2\bbE_t[\|D_z^t\|^2]\\
&= \|z^t - z^*(x^t)\|^2 - 2\rho \langle \nabla_1 G(z^t, x^t), z^t - z^*(x^t)\rangle+ \rho^2\bbE_t[\|D_z^t\|^2]\\
&\leq (1 - \rho\mu_G )\|z^t - z^*(x^t)\|^2 + \rho^2\bbE_t[\|D_z^t\|^2]\enspace.
\end{align}
Taking the total expectation yields
\begin{equation}
    \bbE[\|z^{t+1} - z^*(x^t)\|^2]\leq (1 - \rho\mu_G)\delta_z^t + \rho^2V_z^t\enspace.
\end{equation}
The second member is bounded using Lipschitz continuity of $z^*$:
$$
\bbE[\|z^*(x^{t+1}) - z^*(x^t)\|^2]\leq L^2_* \bbE[\|x^{t+1}-x^t\|^2] = L^2_*\gamma^2V_x^t\enspace .
$$

\new{For the remaining scalar product, we have
\begin{align}
- 2 \langle z^{t+1} - z^*(x^t), z^*(x^{t+1}) - z^*(x^t) \rangle &= - 2 [\langle z^{t} - z^*(x^t), z^*(x^{t+1}) - z^*(x^t) \rangle - \rho \langle D^t_z, z^*(x^{t+1}) - z^*(x^t)\rangle]\enspace .
\end{align}
}

\new{
The second term can be bounded using Cauchy-Schwarz inequality, the Lipschitz-continuity of $z^*$ and Young inequality:
\begin{align}
     \bbE[\rho \langle D^t_z, z^*(x^{t+1}) - z^*(x^t)\rangle] &\leq \bbE[\rho\|D^t_z\|\|z^*(x^{t+1}) - z^*(x^t)\|]\\
     &\leq \rho L_*\bbE[\|D^t_z\|\|x^{t+1} - x^t\|]\\
     &\leq \frac{\rho^2}2 V_z^t + \frac{L_*^2}{2}\|x^{t+1} - x^t\|^2\\
     &\leq \frac{\rho^2}2 V_z^t + L_*^2\frac{\gamma^2}{2}V_x^t\enspace.
\end{align}
}

\new{For $-2\langle z^t-z^*(x^t), z^*(x^{t+1})-z^*(x^t)\rangle$, we follow the proof of \cite{Chen2021b} which consists in making appear the "unbiased part of $z^*(x^{t+1}-z^*(x^t)$ by a linear approximation.
More precisely, we have
\begin{align}
    \langle z^t-z^*(x^t), z^*(x^{t+1})-z^*(x^t)\rangle &= \underbrace{\langle z^t-z^*(x^t), \diff z^*(x^t)(x^{t+1}-x^t)\rangle}_A\\\nonumber
    &\qquad \underbrace{\langle z^t-z^*(x^t), z^*(x^{t+1})-z^*(x^t)-\diff z^*(x^t)(x^{t+1}-x^t)\rangle}_B\enspace .
\end{align}
For $A$, we use the unbiasedness of $D^t_x$, Cauchy-Schwarz inequality, the Lipschitz continuity of $z^*$ (\cref{lemma:smoothness_star}) and the identity $ab\leq \eta a^2 + \frac{b^2}\eta$ for any $\eta>0$:
\begin{align}
    -2\bbE[A] &= -2\gamma\bbE[\langle z^t-z^*(x^t), \diff z^*(x^t)D^t_x\rangle]\\
    &=-2\gamma\bbE[\langle z^t-z^*(x^t), \diff z^*(x^t)\bbE_t[D^t_x]\rangle]\\
    &=-2\gamma\bbE[\langle z^t-z^*(x^t), \diff z^*(x^t)D_x(z^t, v^t, x^t)\rangle]\\
    &\leq 2\gamma\bbE[\|z^t-z^*(x^t)\|\|\diff z^*(x^t)D_x(z^t, v^t, x^t)\|]\\
    &\leq 2L_*\gamma\bbE[\|z^t-z^*(x^t)\|\|D_x(z^t, v^t, x^t)\|]\\
    &\leq 2\eta\delta^t_z + \frac{2L_*^2}\eta\gamma^2\bbE[\|D_x(z^t,v^t,x^t)\|^2]\enspace.
\end{align}
We take $\eta = \frac{\rho\mu_G}{4}$ and we get
\begin{align}
    -2\bbE[A] &\leq \frac{\rho\mu_G}2\delta^t_z + \frac{8L_*^2}{\mu_G}\frac{\gamma^2}{\rho}\bbE[\|D_x(z^t,v^t,x^t)\|^2]\enspace .
\end{align}
For $B$, we use Cauchy-Schwarz inequality, the smoothness of $z^*$ (\cref{app:lemma:smoothness_z_star}), Young inequality and the boundedness of $\bbE_t[\|D_x^t\|^2]$ to get
\begin{align}
    -2\bbE[B] &\leq 2\bbE[\|z^t-z^*(x^t)\|\| z^*(x^{t+1})-z^*(x^t)-\diff z^*(x^t)(x^{t+1}-x^t)\|]\\
    &\leq L_{zx}\bbE[\|z^t-z^*(x^t)\|\|x^{t+1}-x^t\|^2]\\
    &\leq L_{zx}\nu\bbE[\|z^t-z^*(x^t)\|^2\|x^{t+1}-x^t\|^2]+\frac{L_{zx}}\nu \bbE[\|x^{t+1}-x^t\|^2]\\
    &\leq L_{zx}\nu\gamma^2\bbE[\|z^t-z^*(x^t)\|^2\bbE_t[\|D^t_x\|^2]]+\frac{L_{zx}\gamma^2}\nu V^t_x\\
    &\leq L_{zx}B_x^2\nu\gamma^2\delta^t_z + \frac{L_{zx}\gamma^2}{\nu}V_x^t\enspace.
\end{align}
We take $\nu = \frac{L_{zx}}{L_*^2}$ and we get
\begin{align}
    -2\bbE[B]&\leq \frac{L_{zx}^2B_x^2\gamma^2}{L_*^2}\delta^t_z + L_*^2\gamma^2 V_x^t
\end{align}
}

\new{
Now, using $\gamma^2\leq \frac{\rho\mu_G L_*^2}{B_x^2L_{zx}^2}$, we end up with
\begin{equation}\label{app:descent_z}
    \boxed{
    \delta^{t+1}_z \leq (1-\frac{\rho\mu_G}4)\delta^t_z + 2\rho^2 V_z^t + \beta_{zx}\gamma^2V_x^t + \overline{\beta}_{zx}\frac{\gamma^2}{\rho} \bbE[\|D_x(z^t,v^t,x^t)\|^2]\enspace,
    }
\end{equation}
with $\beta_{zx} = 3L_*^2$ and $\overline{\beta}_{zx} = \frac{8L_*^2}{\mu_G}$.
}

\textbf{Inequality for $\delta_v$.\quad} \new{We proceed in a similar way for $v$:
\begin{equation}\label{app:eq:young_delta_v}
\delta^{t+1}_v \leq \bbE[\|v^{t+1} - v^*(x^t)\|^2] + \bbE[\|v^*(x^{t+1}) - v^*(x^t)\|^2] - 2\bbE[\langle v^{t+1}-v^*(x^{t}), v^*(x^{t+1}) - v^*(x^t)\rangle]\enspace .
\end{equation}}

For the first term, we have
\begin{align}
   \bbE_t[\|v^{t+1} - v^*(x^t)\|^2] &= \|v^{t} - v^*(x^t)\|^2 - 2\rho \langle  D_v(z^t, v^t, x^t), v^t - v^*(x^t)\rangle + \rho^2\bbE_t[\|D_v^t\|^2]
\end{align}
Now, using that $D_v(z^*(x^t), v^*(x^t), x^t)=0$:
\begin{align}
    \langle D_v(z^t, v^t, x^t), v^t-v^*(x^t)\rangle &= \langle D_v(z^t, v^t, x^t) - D_v(z^*(x^t), v^*(x^t), x^t), v^t-v^*(x^t)\rangle \\
    &= \langle \nabla^2_{11}G(z^t, x^t)(v^t-v^*(x^t)), v^t-v^*(x^t)\rangle \\\nonumber
    &\qquad+ \langle (\nabla^2_{11}G(z^t, x^t)-\nabla^2_{11}G(z^*(x^t), x^t)) v^*(x^t), v^t-v^*(x^t)\rangle\\\nonumber
    &\qquad+\langle (\nabla_1 F(z^t, x^t)-\nabla_{1}F(z^*(x^t), x^t)), v^t-v^*(x^t)\rangle\\
    &\geq \mu_G \|v^t-v^*(x^t)\|^2  - \frac{L^G_2 C_F}{\mu_G}\|z^t-z^*(x^t)\| \|v^t-v^*(x^t)\| \\\nonumber
    &\qquad- L^F_1\|z^t-z^*(x^t)\| \|v^t-v^*(x^t)\|\\
    &\geq \mu_G \|v^t-v^*(x^t)\|^2 - \omega \|z^t-z^*(x^t)\| \|v^t-v^*(x^t)\|
\end{align}
where $\omega = L^F_1+\frac{L^G_2 C_F}{\mu_G}$.
We then use $\omega\|z^t-z^*(x^t)\|\|v^t-v^*(x^t)\| \leq \frac12c\|v^t-v^*(x^t)\|^2 + \frac{\omega^2}{2c}\|z^t-z^*(x^t)\| ^2$ with $c = \mu_G$ to get
$$
-\langle  D_v(z^t, v^t, x^t), v^t - v^*(x^t)\rangle \leq - \frac12\mu_G\delta_v^t +\frac{\omega^2}{2\mu_G}\delta_z^t\enspace.
$$
We get the overall inequality by taking the total expectation
$$
\bbE[\|v^{t+1} - v^*(x^t)\|^2] \leq \left(1 - \frac{\rho\mu_G}{2}\right)\delta_v^t + \rho\frac{\omega^2}{2\mu_G}\delta_z^t +\rho^2V_v^t\enspace .
$$
We also use Lipschitz on $v^*$ to bound the other term
$$
\bbE[\|v^*(x^{t+1}) - v^*(x^t)\|^2] \leq L_*^2\gamma^2V_x^t\enspace.
$$

\new{
As previously, the scalar product is bounded by:
\begin{align}
    -\bbE[\langle v^{t+1}-v^*(x^{t}), v^*(x^{t+1}) - v^*(x^t)\rangle] &= - \bbE[\langle v^{t}-v^*(x^{t}), v^*(x^{t+1}) - v^*(x^t)\rangle] - \rho\bbE[\langle D^t_v, v^*(x^{t+1}) - v^*(x^t)\rangle]\\
    &\leq \bbE[\langle z^t-z^*(x^t),v^*(x^{t+1})-v^*(x^t)\rangle] + \frac{\rho^2}2V_v^t + L_*^2\frac{\gamma^2}2V_x^t
\end{align}}

\new{
We do similar manipulations pour $v^*$, thanks to \cref{app:lemma:smoothness_v_star}.
}
\new{
We have as for $z$ from \cref{lemma:smoothness_star} for any $\eta>0$:
\begin{align}
    -\bbE[\langle v^t - v^*(x^t), \diff v^*(x^t)(x^{t+1}-x^t)\rangle] &\leq  \eta\delta^t_v + \frac{L_*^2\gamma^2}{\eta}\bbE[\|D_x(z^t, v^t, x^t)\|^2]\enspace .
\end{align}
We take $\eta = \frac{\rho\mu_G}{8}$ and we get
\begin{align}
    -\bbE[\langle v^t - v^*(x^t), \diff v^*(x^t)(x^{t+1}-x^t)\rangle] &\leq  \frac{\rho\mu_G}8\delta^t_v + \frac{8L_*^2\gamma^2}{\mu_G\rho}\bbE[\|D_x(z^t, v^t, x^t)\|^2]\\
\end{align}
Then smoothness of $v^*$ for any $\eta>0$ gives us
\begin{align}
    -\bbE[\langle v^t - v^*(x^t),  v^*(x^{t+1}) - v^*(x^t) - \diff v^*(x^t)(x^{t+1}-x^t)\rangle] &\leq  \frac{L_{vx}B^2_x\nu}{2}\gamma^2\delta^t_v+ \frac{L_{vx}}{2\nu}\gamma^2V_x^t\enspace.
\end{align}
With $\nu = \frac{L_{vx}}{L_*^2}$ we get
\begin{align}
    -\bbE[\langle v^t - v^*(x^t),  v^*(x^{t+1}) - v^*(x^t) - \diff v^*(x^t)(x^{t+1}-x^t)\rangle] &\leq  \frac{L_{vx}^2B^2_x}{2L_*^2}\gamma^2\delta^t_v+ \frac{L_*^2}2\gamma^2V_x^t\enspace.
\end{align}
}

\new{
With the assumption $\gamma^2\leqslant \frac{\rho\mu_GL_*^2}{8L_{vx}^2B_x^2}$, we get
\begin{align}
    \delta^{t+1}_v &\leq \left(1-\frac{\rho\mu_G}2 + \frac{\rho\mu_G}4 + \frac{L_{vx}^2B_x}{L_*^2}\right)\delta^t_v + \rho\beta_{vz}\delta^t_z + 2\rho^2V_z^t+ 3L_*^2\gamma^2V^t_x +\frac{16L_*^2\gamma^2}{\mu_G\rho} \bbE[\|D_x(z^t, v^t, x^t)\|^2]\\
    &\leq \left(1-\frac{\rho\mu_G}8\right)\delta^t_v + \rho\beta_{vz}\delta^t_z + 2\rho^2V_z^t + 3L_*^2\gamma^2V^t_x +  \frac{16L_*^2\gamma^2}{\mu_G\rho}\bbE[\|D_x(z^t, v^t, x^t)\|^2]\enspace.
\end{align}
}

\new{
And finally we have
\begin{equation}
\label{app:descent_v}
    \boxed{
    \delta^{t+1}_v \leq \left(1-\frac{\rho\mu_G}8\right)\delta^t_v + \rho\beta_{vz}\delta^t_z + 2\rho^2V_z^t + \beta_{vx}\gamma^2V^t_x +  \overline{\beta}_{vx}\frac{\gamma^2}{\rho}\bbE[\|D_x(z^t, v^t, x^t)\|^2]
    }
\end{equation}
with $\beta_{vz} = \frac{\omega^2}{2\mu_G}$, $\beta_{vx} = 3L_*^2$ and $\overline{\beta}_{vx} =\frac{16L_*^2\gamma^2}{\mu_G}$.
}
\end{proof}
\subsection{Proof of \cref{lemma:descent_lemma}}

\begin{proof}
We use smoothness of $h$ to get
\begin{align}
\bbE_t[h(x^{t+1})]&\leq h(x^t) -\gamma\langle D_x(z^t, v^t, x^t), \nabla h(x^t)\rangle + \frac{L^h}2\gamma^2\bbE_t[\|D_x^t\|^2]\\
&\new{\leq h(x^t) - \frac\gamma2(\|\nabla h(x^t)\|^2 + \|D_x(z^t,v^t,x^t)\|^2 - \|\nabla h(x^t) - D_x(z^t,v^t,x^t)\|^2) + \frac{L^h}2\gamma^2\bbE_t[\|D^t_x\|^2]
}
\end{align}
\new{where the last inequality comes from the identity $\langle a,b\rangle = \frac12(\|a\|^2+\|b\|^2 - \|a-b\|)^2$.}
We take the total expectation and use the previous \cref{lemma:smoothness} to get
\new{
\begin{equation}
\label{eq:descent_h}
\boxed{
        h^{t+1} \leq h^t-\frac\gamma2g^t-\frac{\gamma}2\bbE[\|D_x(z^t,v^t,x^t)\|^2] + \frac{\gamma L_x^2}2(\delta^t_z + \delta^t_v) + \frac{L^h}2\gamma^2V_x^t
    }
\end{equation}
}
\end{proof}

\subsection{Proof of \cref{thm:soba_fixed}}
This section is devoted to the proof of \cref{thm:soba_fixed} that we recall here.
\sobafixed*
The values of the different constants are
\begin{align*}
    \phi_z' = \frac1{8\overline{\beta}_{zx}}&,\quad\phi_v' = \min\left(\frac1{8\overline{\beta}_{vx}}, \frac{\mu_G\phi_z'}{32\beta_{vz}}\right)\enspace, \quad \overline{\rho} = \min\left(\frac{16}{\mu_G},\frac{\mu_G}{16L_z^2B_z^2},\frac{\mu_G}{32L_v^2B_v^2}, \frac{\beta_{vz}}{L_v^2B_v^2}\right)\enspace,\\
    &
    \quad\text{and}\quad\xi^2 =\frac{\mu_G}4\min\left[\min\left(\frac1{L_{zx}^2},\frac1{L_{vx}^2}\right)\frac{L_*^2}{B_x^2\overline{\rho}}, \min\left(\phi_v', \phi_z'\right)\frac1{2L_x^2}\right]\enspace.
\end{align*}
\new{Before, one has to adapt our descent lemmas to the case of SOBA.
\begin{lemma}\label{app:lemma:coupled_inequalities_soba}
    Assume that the step sizes $\rho$ and $\gamma$ verify $\rho\leq\min\left(\frac{\mu_G}{16L_z^2B_z^2} ,\frac{\mu_G}{32L_v^2B_v^2},\frac{\beta_{vz}}{L_v^2B_v^2}\right)$ and $\gamma^2\leq\min\left(\frac{\rho\mu_GL_*^2}{4B_x^2L_{zx}^2},\frac{\rho\mu_GL_*^2}{8B_x^2L_{vx}^2} \right)$. Then it holds
    \begin{align}
        \label{app:eq:descent_z_soba}\delta^{t+1}_z &\leq \left(1-\frac{\rho\mu_G}8\right)\delta^t_z + 2\rho^2B_z^2 + \beta_{zx}\gamma^2B_x^2 + \overline{\beta}_{zx}\frac{\gamma^2}\rho\bbE[\|D_x(z^t, v^t, x^t)\|^2]\\
        \label{app:eq:descent_v_soba} \delta^{t+1}_v &\leq \left(1-\frac{\rho\mu_G}{16}\right)\delta^t_v + 2\beta_{vz}\rho \delta^t_z + 2\rho^2B_v^2 + \beta_{vx}\gamma^2B_x^2 +\overline{\beta}_{vx}\bbE[\|D_x(z^t,v^t,x^t)\|^2]\enspace.
    \end{align}
\end{lemma}
\begin{proof}
    From \cref{assumption:bounded_gradients} and \cref{lemma:smoothness}, we have
    $$V_z^t \leq B_z^2(1+D_z(z^t,v^t,x^t)) \leq B_z^2(1+L_z^2\delta^t_z)\enspace.$$
    Plugging this into \cref{app:descent_z} and using $V_x^t\leq B_x^2$ yields
    \begin{align}
        \delta^{t+1}_z &\leq \left(1-\frac{\rho\mu_G}4 + 2L_z^2B_z^2\rho^2\right)\delta^t_z + 2\rho^2B_z^2 + \beta_{zx}\gamma^2B_x^2 + \overline{\beta}_{zx}\frac{\gamma^2}\rho\bbE[\|D_x(z^t, v^t, x^t)\|^2]\enspace.
    \end{align}
    Since by assumption $\rho\leq\frac{\mu_G}{16L_z^2B_z^2}$, we have
    \begin{align}
        \delta^{t+1}_z &\leq \left(1-\frac{\rho\mu_G}8\right)\delta^t_z + 2\rho^2B_z^2 + \beta_{zx}\gamma^2B_x^2 + \overline{\beta}_{zx}\frac{\gamma^2}\rho\bbE[\|D_x(z^t, v^t, x^t)\|^2]\enspace.
    \end{align}
    For $\delta^t_v$, \cref{ass:3} and \cref{lemma:smoothness} provide us
    $$V_v^t \leq B_v^2(1+L_v^2(\delta^t_z + \delta^t_v))\enspace.$$
    Since the assumptions of \cref{lemma:coupled_inequalities} are verified, we can plug the previous inequality into \cref{app:descent_v} to get
    \begin{align}
        \delta^{t+1}_v \leq \left(1-\frac{\rho\mu_G}8 + 2L_v^2B_v^2\rho^2\right)\delta^t_v + (\beta_{vz}\rho + 2L_v^2\rho^2B_v^2)\delta^t_z + 2\rho^2B_v^2 + \beta_{vx}\gamma^2B_x^2 +\overline{\beta}_{vx}\bbE[\|D_x(z^t,v^t,x^t)\|^2]
    \end{align}
    which can be simplified using $\rho\leq\min\left(\frac{\mu_G}{32L_v^2B_v^2},\frac{\beta_{vz}}{L_v^2B_v^2}\right)$ to get finally
    \begin{align}
         \delta^{t+1}_v \leq \left(1-\frac{\rho\mu_G}{16}\right)\delta^t_v + 2\beta_{vz}\rho \delta^t_z + 2\rho^2B_v^2 + \beta_{vx}\gamma^2B_x^2 +\overline{\beta}_{vx}\bbE[\|D_x(z^t,v^t,x^t)\|^2]\enspace .
    \end{align}
\end{proof}
}

We can now prove \cref{thm:soba_fixed}.
\begin{proof}
\new{
    Consider the Lyapunov function $\calL^t = h^t + \phi_z\delta^t_z + \phi_v\delta^t_v$. Using the Equations \eqref{eq:descent_h}, \eqref{app:descent_z} and \eqref{app:descent_v}, we can bound $\calL^{t+1}-\calL^t$:
    \begin{align}
        \calL^{t+1}-\calL^t &\leq -\frac\gamma2 g^t - \left(\frac\gamma2 - \phi_z\overline{\beta}_{zx}\frac{\gamma^2}\rho - \phi_v\overline{\beta}_{vx}\frac{\gamma^2}\rho \right)\bbE[\|D_x(z^t,v^t,x^t)\|^2]\\\nonumber
        &\qquad - \left(\phi_z\frac{\mu_G}8\rho - \frac{L_x^2}2\gamma - 2\phi_v\beta_{vz}\rho\right)\delta^t_z \\\nonumber
        &\qquad - \left(\phi_v\frac{\mu_G}{16}\rho - \frac{L_x^2}2\gamma\right)\delta^t_v\\\nonumber
        &\qquad + \left(\frac{L^h}2 + \phi_z\beta_{zx} + \phi_v\beta_{vx}\right)B_x^2\gamma^2 \\\nonumber
        &\qquad + 2(\phi_zB_z^2+\phi_vB_v^2)\rho^2\enspace.
    \end{align}
    Let $\phi_z' = \phi_z \frac\gamma\rho$ and $\phi_v' = \phi_v \frac\gamma\rho$, so that:
        \begin{align}
        \calL^{t+1}-\calL^t &\leq -\frac\gamma2 g^t - \left(\frac\gamma2 - \phi_z'\overline{\beta}_{zx}\gamma - \phi_v'\overline{\beta}_{vx}\gamma \right)\bbE[\|D_x(z^t,v^t,x^t)\|^2]\\\nonumber
        &\qquad - \left(\phi_z'\frac{\mu_G}8\frac{\rho^2}\gamma - \frac{L_x^2}2\gamma - 2\phi_v'\beta_{vz}\frac{\rho^2}\gamma\right)\delta^t_z \\\nonumber
        &\qquad - \left(\phi_v'\frac{\mu_G}{16}\frac{\rho^2}\gamma - \frac{L_x^2}2\gamma\right)\delta^t_v\\\nonumber
        &\qquad + \left(\frac{L^h}2 + \phi_z'\beta_{zx}\frac\rho\gamma + \phi_v\beta_{vx}\frac\rho\gamma\right)B_x^2\gamma^2 \\\nonumber
        &\qquad + 2\left(\phi_z'B_z^2\frac\rho\gamma+\phi_v'B_v^2\frac\rho\gamma\right)\rho^2\enspace.
    \end{align}
    In order to get a decrease, $\phi_z'$, $\phi_v'$, $\rho$ and $\gamma$ must verify
    \begin{align}\label{app:eq:cond_lyap_soba}
        \left\{
        \begin{array}{l}
            \phi_z'\overline{\beta}_{zx} + \phi_v'\overline{\beta}_{vx} \leq \frac12\\
             \frac{L_x^2}2\gamma + 2\phi_v'\beta_{vz}\frac{\rho^2}\gamma \leq \phi_z'\frac{\mu_G}8\frac{\rho^2}\gamma\\
             \frac{L_x^2}2\gamma\leq\phi_v'\frac{\mu_G}{16}\frac{\rho^2}\gamma
        \end{array}
        \right.
    \end{align}
    Let us take $\phi_z' = \frac1{8\overline{\beta}_{zx}}$ and $\phi_v' = \min\left(\frac1{8\overline{\beta}_{vx}}, \frac{\mu_G\phi_z'}{32\beta_{vz}}\right)$. We have
    $$
    \phi_z'\overline{\beta}_{zx} + \phi_v'\overline{\beta}_{vx} \leq \frac14<\frac12
    $$
    and
    $$
    \frac{L_x^2}2\gamma + 2\phi_v'\beta_{vz}\frac{\rho^2}\gamma\leq \frac{L_x^2}2\gamma + \phi_z'\frac{\mu_G}{16}\frac{\rho^2}\gamma\enspace.
    $$
    If we impose $\frac{L_x^2}2\gamma + \phi_z'\frac{\mu_G}{16}\frac{\rho^2}\gamma\leq \phi_z'\frac{\mu_G}8\frac{\rho^2}\gamma$, this combined with the third condition in \cref{app:eq:cond_lyap_soba} gives the condition $\frac{L_x^2}2\gamma^2\leq \min\left(\phi_v', \phi_z'\right)\frac{\mu_G}{16}\rho^2$. We also have the conditions coming from the assumptions of \ref{app:lemma:coupled_inequalities_soba}, that is
    \begin{equation}\label{app:eq:overline_rho}
        \rho \leq \overline{\rho} = \min\left(\frac{16}{\mu_G},\frac{\mu_G}{16L_z^2B_z^2},\frac{\mu_G}{32L_v^2B_v^2}, \frac{\beta_{vz}}{L_v^2}\right)
    \end{equation}
    and $\gamma^2\leq\min\left(\frac1{L_{zx}^2},\frac1{L_{vx}^2}\right)\frac{\mu_GL_*^2}{4B_x^2}\rho$. Let us take $\rho = \frac{\overline{\rho}}{\sqrt{T}}$ with
    $\gamma = \xi\rho$ where $\xi$ is defined as
    \begin{equation}\label{app:eq:xi_def}
        \xi^2 \triangleq \frac{\mu_G}4\min\left[\min\left(\frac1{L_{zx}^2},\frac1{L_{vx}^2}\right)\frac{L_*^2}{B_x^2\overline{\rho}}, \min\left(\phi_v', \phi_z'\right)\frac1{2L_x^2}\right]\enspace.
    \end{equation}
    }

    \new{
    From now, we have
    \begin{align}\label{app:eq:diff_lyap_soba}
        \calL^{t+1}-\calL^t&\leq -\frac\gamma2g^t + \frac{L^h}2B_x^2\gamma^2 + \left(\phi_z'\beta_{zx} + \phi_v'\beta_{vx}\right)B_x^2\rho\gamma + 2\left(\phi_z'B_z^2+\phi_v'B_v^2\right)\frac{\rho^3}\gamma\enspace .
    \end{align}
    Summing and telescoping yields
    \begin{align}
        \frac1T\sum_{t=1}^Tg^t &\leq \frac{2\calL^1}{T\gamma} + L^hB_x^2\gamma + 2\left(\phi_z'\beta_{zx} + 2\phi_v'\beta_{vx}\right)B_x^2\rho + 4\left(\phi_z'B_z^2+\phi_v'B_v^2\right)\frac{\rho^3}{\gamma^2}\\
        &\leq  \frac{2\calL^1}{\sqrt{T}\xi\overline{\rho}} + L^hB_x^2\frac{\xi\alpha}{\sqrt{T}} + \left(\phi_z'\beta_{zx} + 2\phi_v'\beta_{vx}\right)B_x^2\frac\alpha{\sqrt{T}} + 4\left(\phi_z'B_z^2+\phi_v'B_v^2\right)\frac{\alpha}{\xi^2\sqrt{T}}\\
    \end{align}
    and so
    $$
    \boxed{
    \frac1T\sum_{t=1}^Tg^t = \calO\left(\frac1{\sqrt{T}}\right)\enspace.
    }
    $$
    }
\end{proof}

\subsection{Proof of \cref{thm:decreasing_step_size}}

\begin{proof}
\new{In the decreasing step size case, we take $\rho^t = \overline{\rho}\sqrt{t}$ and $\gamma^t = \xi\rho^t$ where $\overline{\rho}$ is defined in \cref{app:eq:overline_rho} and $\xi$ is defined in \cref{app:eq:xi_def}}. We recall the integral majorization:
$$\sum_{t=1}^Tt^{-1} \leq 1 + \int_{1}^Tt^{-1} \diff t = 1 + \log (T)\enspace.$$

\new{
With such definition of $\rho^t$ and $\gamma^t$, \cref{app:eq:diff_lyap_soba} is still valid for any $t\geqslant 1$. The only difference is that the step sizes decrease with $t$. Hence, by summing and rearranging in \cref{app:eq:diff_lyap_soba}, we get
\begin{align}\label{app:eq:diff_lyap_soba_decrease}
    \sum_{t=1}^T \gamma^tg^t\leq 2\calL^1 + \left(L^h + 2\left(\phi_z'\beta_{zx} + 2\phi_v'\beta_{vx}\right)B_x^2\frac1{\xi} + 4\left(\phi_z'B_z^2+\phi_v'B_v^2\right)\frac1{\xi^3}\right)\sum_{t=1}^T(\gamma^t)^2
\end{align}
The left-hand-side in \cref{app:eq:diff_lyap_soba_decrease} can be lower bounded by
\begin{equation}\label{app:eq:lhs_soba_decrease}
    \sum_{t=1}^T \gamma^tg^t \geq \left(\inf_{t\in\setcomb{T}}g^t\right) \xi\overline{\rho}\sum_{t=1}^T t^{-\frac12}\geq \left(\inf_{t\in\setcomb{T}}g^t\right)\xi\overline{\rho}T^{\frac12}\enspace.
\end{equation}
Also we have
\begin{equation}\label{app:eq:majoration_sum_1/t}
    \sum_{t=1}^T (\gamma^t)^2 = \xi^2\overline{\rho}^2 \sum_{t=1}^T t^{-1} \leq \xi^2\overline{\rho}^2(1 + \log(T))\enspace.
\end{equation}
Plugging Equations \eqref{app:eq:lhs_soba_decrease} and \eqref{app:eq:majoration_sum_1/t} into \cref{app:eq:diff_lyap_soba_decrease} and rearranging give
\begin{equation}
    \inf_{t\in\setcomb{T}}g^t\leq \frac{2\calL^1}{\xi\overline{\rho}\sqrt{T}} + \xi\overline{\rho}\left(L^h + 2\left(\phi_z'\beta_{zx} + 2\phi_v'\beta_{vx}\right)B_x^2\frac1{\xi} + 4\left(\phi_z'B_z^2+\phi_v'B_v^2\right)\frac1{\xi^3}\right)\frac{1+\log(T)}{\sqrt{T}}
\end{equation}
that is to say
\begin{equation}
\boxed{
    \inf_{t\in\setcomb{T}}g^t = \calO\left(\frac1{\sqrt{T}} + \frac{\log(T)}{\sqrt{T}}\right)\enspace.
}
\end{equation}
}
\end{proof}

\subsection{Proof of \cref{th:cvg_saba_smooth}}

In this section, we prove \cref{th:cvg_saba_smooth} that we recall here
\sabasmooth*

\new{
The constants $\rho'$ and $\xi$ are given by
\begin{align}
    \rho' = \min\left(\sqrt{\frac{K_1}{K_5}},\left(\frac{K_2}{K_5}\right)^{\frac25},\left(\frac{K_3}{K_5}\right)^{\frac57},\left(\frac{K_4}{K_5}\right)^{\frac13},\frac{\mu_G}{64L_z^2},\frac{\overline{\beta}_{zx}}{2\beta_{zx}},\frac{\mu_G}{128(L_v^2 + L_v'')},\frac{\beta_{vz}}{8(L_v^2+L_v'')},\frac{\overline{\beta}_{vx}}{2\beta_{vx}}\right)
\end{align}
and
\begin{align}
    \xi = \min(K_1, K_2(\rho')^{-\frac12}, K_3(\rho')^{-\frac32},K_4(\rho')^{-1})
\end{align}
where
$$
    \phi_z'' = \frac1{32\overline{\beta}_{zx}},\quad
    \phi_v'' = \min\left(\frac1{32\overline{\beta}_{vx}}, \phi_z''\frac{\mu_G}{128\beta_{vz}}\right)\enspace,
$$
\begin{align*}
    &K_1 = \min\left(\sqrt{\frac{\phi_z''\mu_G}{32L_x^2}}, \sqrt{\frac{\phi_v''\mu_G}{48L_x^2}}, \sqrt{\frac{L_z'}{2L_x'\beta_{zx}}},\sqrt{\frac{L_v'}{2L_x'\beta_{vx}}}\right)\enspace,\\
    &K_2 = \min\left(\sqrt{\frac{\mu_G}{64\beta_{zx}L_x''}},\sqrt{\frac{\mu_G}{128\beta_{vx}L_x''}},\sqrt{\frac{\beta_{vz}}{4L_x''\beta_{vx}}}\right)\enspace,\\
    K_3 = &\sqrt{\frac{\phi_v''\mu_G}{384\phi_z''L_x''}}\enspace,\quad
    K_4 =\min\left(\frac1{4L^h},\frac{L_x^2}{2L^hL_x''},\sqrt{\frac{\Gamma'}{6L^hL_x'}},\frac1{8P'},\frac{\phi_z''\mu_G}{32\beta'_{sz}},\frac{\phi_v''\mu_G}{48\beta'_{sv}} \right)
\end{align*}
and
$$
    K_5 = \frac{15(\phi_z''L_z'+\phi_v''L_v')}{\Gamma'}\enspace.
$$
}

\subsubsection{Control of distance from memory to iterates}

We can view our method has having two ``parallel'' memories for each variable
$
(z^t_i, v^t_i, x^t_i)
$ for $i\in1\setcomb{n}$ corresponding to calls in $G$ and $
(z'^t_j, v'^t_j, x'^t_j)
$
for $j\in\setcomb{m}$ corresponding to calls to $F$.
At each iteration, we sample $i$ at random uniformly and do $
(z^{t+1}_i, v^{t+1}_i, x^{t+1}_i) = (z^t, v^t, x^t)
$ and $(z^{t+1}_{i'}, v^{t+1}_{i'}, x^{t+1}_{i'}) = (z^t_{i'}, v^t_{i'}, x^t_{i'})$ for $i'\neq i$, and similarly for the other memory.

In what follows, we focus on controlling the error between the iterates and the memories.
We define to make things simpler
$$
E_z^t = \frac1n \sum_{i=1}^n\bbE[\|z^t - z_i^t\|^2]\enspace,
\quad
E_v^t = \frac1n \sum_{i=1}^n\bbE[\|v^t - v_i^t\|^2]\enspace,
\quad
E_x^t = \frac1n \sum_{i=1}^n\bbE[\|x^t - x_i^t\|^2]\enspace,
$$
and similarly $E_x'^t, E_v'^t$ and $E_x'^t$.

\begin{lemma}\label{lemma:bound_memory}
We have the following inequalities:
$$
E_z^{t+1} \leq \left(1-\frac1{2n}\right)E_z^t +  \rho^2\bbE\|D_z^t\|^2 + 2n\rho^2 \bbE[\|D_z(z^t, v^t, x^t)\|^2]\enspace ,
$$
$$
E_v^{t+1} \leq \left(1-\frac1{2n}\right)E_v^t +  \rho^2\bbE\|D_v^t\|^2 + 2n\rho^2 \bbE[\|D_v(z^t, v^t, x^t)\|^2]\enspace,
$$
$$
E_x^{t+1} \leq \left(1-\frac1{2n}\right)E_x^t +  \gamma^2\bbE\|D_x^t\|^2 + 2n\gamma^2\bbE[\|D_x(z^t, v^t, x^t)\|^2]\enspace,
$$
$$
E_z'^{t+1} \leq \left(1-\frac1{2m}\right)E_z^t +  \rho^2\bbE\|D_z^t\|^2 + 2m\rho^2 \bbE[\|D_z(z^t, v^t, x^t)\|^2]\enspace ,
$$
$$
E_v'^{t+1} \leq \left(1-\frac1{2m}\right)E_v^t +  \rho^2\bbE\|D_v^t\|^2 + 2m\rho^2 \bbE[\|D_v(z^t, v^t, x^t)\|^2]\enspace,
$$
and
$$
E_x'^{t+1} \leq \left(1-\frac1{2m}\right)E_x^t +  \gamma^2\bbE\|D_x^t\|^2 + 2m\gamma^2\bbE[\|D_x(z^t, v^t, x^t)\|^2]\enspace.
$$

\end{lemma}

\begin{proof}
We provide the detailed proof for $E^t_z$. The approach for the five others is similar.

Let $i\in\setcomb{n}$. Taking the expectation of $\|z^{t+1} - z_i^{t+1}\|^2$ conditionaly to $z^t, v^t, x^t$ yields

$$
\bbE_t[\|z^{t+1} - z_i^{t+1}\|^2] = \frac1n\bbE_t[\|z^{t+1} - z^t\|^2] +\frac{n -1}n\bbE_t[\|z^{t+1} - z_i^{t}\|^2]\enspace .
$$

Then, using the fact that $\bbE_t[D_z^t(z^t,v^t,x^t)] = D_z(z^t,v^t,x^t)$, we have
\begin{align}
\bbE_t[\|z^{t+1} - z_i^{t}\|^2] = \bbE_t[\|z^{t+1} - z^t\|^2] + \|z^t - z_i^{t}\|^2 - 2 \rho\langle D_z(z^t, v^t, x^t), z^t - z_i^{t}\rangle\enspace .
\end{align}
We then upper-bound crudely the scalar product by Cauchy-Schwarz and Young inequalities with parameter $\beta$:
$$
\bbE_t[\|z^{t+1} - z_i^{t}\|^2] \leq \bbE_t[\|z^{t+1} - z^t\|^2] +\rho\beta^{-1} \|D_z(z^t, v^t, x^t)\|^2 + (1 + \rho\beta)\|z^t - z_i^{t}\|^2
$$

As a consequence, by taking the total expectation and summing for all $i\in\setcomb{n}$, we find
$$
E_z^{t+1} \leq \rho^2\bbE[\|D_z^t\|^2] + \rho \beta^{-1}\left(1 - \frac1n \right) \bbE[\|D_z(z^t, v^t, x^t)\|^2] + (1+\rho\beta)\left(1 - \frac1n\right) E_z^t\enspace .
$$

Finally, we take $\beta = \frac{1}{2n\rho}$ to obtain
\begin{equation}
    \label{eq:bound_sum_memory_z}
\boxed{
E_z^{t+1} \leq \left(1-\frac1{2n}\right)E_z^t +  \rho^2\bbE\|D_z^t(z^t,v^t,x^t)\|^2 + 2n\rho^2 \bbE[\|D_z(z^t, v^t, x^t)\|^2]\enspace .}
\end{equation}
\end{proof}

\subsubsection{Bounds on the variances}

The following lemma gives us upper-bounds for $\bbE[\|D_z^t(z^t,v^t,x^t)\|^2]$, $\bbE[\|D_v^t(z^t,v^t,x^t)\|^2]$, and $\bbE[\|D_x^t(z^t,v^t,x^t)\|^2]$.

\begin{lemma}\label{app:lemma:bound_variance}
For SABA, there are constants $L_z', L'_v, L'_x>0$ such that
$$
\bbE[\|D_z^t(z^t,v^t,x^t)\|]^2 \leq 2\bbE[\|D_z(z^t, v^t, x^t)\|^2] + 2L_z'(E_z^t + E_x^t)\enspace,
$$
$$
\bbE[\|D_v^t(z^t,v^t,x^t)\|^2] \leq 2\bbE[\|D_v(z^t, v^t, x^t)\|^2] + 2L_v'(E_z^t + E_x^t+ E_v^t + E_z'^t + E_x'^t) \new{+ 2L''_v(\delta^t_z + \delta^t_v)}
$$
and
$$
\bbE[\|D_x^t(z^t,v^t,x^t)\|^2] \leq 2\bbE[\|D_x(z^t,v^t,x^t)\|^2] + 2L_x'(E_z^t + E_x^t+ E_v^t + E_z'^t + E_x'^t)\new{+ 2L''_x(\delta^t_z + \delta^t_v)}\enspace.
$$
\end{lemma}
\begin{proof}
For SABA, if we consider $i$ sampled from $\setcomb{n}$ at iteration $t$, we have
$$
D_z^t = \nabla_1 G_i(z^t, x^t) -\nabla_1 G_i(z^t_i, x^t_i) +\frac1n\sum_{i'=1}^n\nabla_1 G_{i'}(z^t_{i'}, x^t_{i'})\enspace.
$$

Hence we get
\begin{align}\nonumber
 \bbE_t[\|D_z^t(z^t,v^t,x^t)\|^2] &= \bbE_t[\|\nabla_1 G_i(z^t, x^t) -\nabla_1 G_i(z^t_i, x^t_i)
+\frac1n\sum_{i'=1}^N\nabla_1 G_{i'}(z^t_{i'}, x^t_{i'})\\\nonumber
&\qquad-\nabla_1 G(z^t, x^t) + \nabla_1 G(z^t, x^t)\|^2] \\
 &\label{eq:var_dz}\leq 2\|\nabla_1 G(z^t, x^t)\|^2 + 2 \bbE_t[\|\nabla_1 G_i(z^t, x^t) -\nabla_1 G_i(z^t_i, x^t_i) \\\nonumber
 &\qquad+\frac1n\sum_{i'=1}^N\nabla_1 G_{i'}(z^t_{i'}, x^t_{i'})-\nabla_1 G(z^t, x^t)\|^2]\enspace .
\end{align}
The second term is the variance of $\nabla_1 G_i(z^t, x^t) -\nabla_1 G_i(z^t_i, x^t_i)$, which is therefore upper-bounded by

\begin{align}\nonumber
\bbE_t[\|[\nabla_1 G_i(z^t, x^t) -\nabla_1 G_i(z^t_i, x^t_i)\|^2]
&= \frac1n \sum_{i=1}^n\|[\nabla_1 G_i(z^t, x^t) -\nabla_1 G_i(z^t_i, x^t_i)\|^2\\
&\label{eq:var_grad_G_i}\leq \frac{L_z'}n\sum_{i=1}^n(\|z^t-z_i^t\|^2 + \|x^t-x_i^t\|^2)
\end{align}
where the inequality comes from the Lipschitz continuity of each $\nabla_1 G_i$ with $L'_z = \max_{i\in\setcomb{n}} L^{G_i}_1$.

Then, by plugging \eqref{eq:var_grad_G_i} into \eqref{eq:var_dz} and taking the total expectation, we get
\begin{equation}
    \label{eq:bound_variance_saga_z}
    \boxed{
    \bbE[\|D_z^t(z^t,v^t,x^t)\|]^2 \leq 2\bbE[\|D_z(z^t, v^t, x^t)\|^2] + 2L_z'(E_z^t + E_x^t)\enspace.}
\end{equation}

Things are quite similar for the other variables, albeit a bit more difficult.

In $v$, it holds
\begin{align}
 \bbE_t[\|D_v^t(z^t,v^t,x^t)\|^2] = &\bbE_t[\|\nabla_1 F_j(z^t, x^t) -\nabla_1 F_j(z'^t_j, x'^t_j) +\frac1m\sum_{j'=1}^m\nabla_1 F_{j'}(z'^t_{j'}, x'^t_{j'}) \\\nonumber
 &+\nabla_{11}^2 G_i(z^t, x^t)v^t -\nabla_{11}^2 G_i(z^t_i, x^t_i)v^t_i +\frac1n\sum_{i'=1}^n\nabla_1^2 G_{i'}(z^t_{i'}, x^t_{i'})v_{i'^t}\\\nonumber
 &-D_v(z^t, v^t, x^t) +D_v(z^t, v^t, x^t)\|^2] \\
 \leq \label{eq:var_dv}&2[\|D_v(z^t,v^t,x^t)\|^2\\\nonumber
 &+ 2\bbE_t[\|\nabla_1 F_j(z^t, x^t) -\nabla_1 F_j(z'^t_j, x'^t_j) +\frac1m\sum_{j'=1}^m\nabla_1 F_{j'}(z'^t_{j'}, x'^t_{j'}) \\\nonumber
 &+\nabla_{11}^2 G_i(z^t, x^t)v^t -\nabla_{11}^2 G_i(z^t_i, x^t_i)v^t_i +\frac1n\sum_{i'=1}^n\nabla_1 G_{i'}^2(z^t_{i'}, x^t_{i'})v_{i'}^t\\\nonumber
 &-D_v(z^t, v^t, x^t) \|^2]\\
\end{align}

Here, we see that we need to control the variance of $\nabla_1 F_j(z^t, x^t) -\nabla_1 F_j(z'^t_j, x'^t_j) + \nabla_{11}^2 G_i(z^t, x^t)v^t -\nabla_{11}^2 G_i(z^t_i, x^t_i)v^t_i$. Since $i$ and $j$ are independent, this is a sum of two independent random variables, hence its variance is the sum of the variances, which is upper-bounded by
$$
\bbE_t[\|\nabla_1 F_j(z^t, x^t) -\nabla_1 F_j(z'^t_j, x'^t_j)\|^2] + \bbE_t[\| \nabla_{11}^2 G_i(z^t, x^t)v^t -\nabla_{11}^2 G_i(z^t_i, x^t_i)v^t_i\|^2]\enspace.
$$
\new{
For $\bbE_t[\|\nabla_1 F_j(z^t, x^t) -\nabla_1 F_j(z'^t_j, x'^t_j)\|^2] $ we use the lipschitz continuity of the $\nabla_1 F_j$:
\begin{align}
    \bbE_t[\|\nabla_1 F_j(z^t, x^t) -\nabla_1 F_j(z'^t_j, x'^t_j)\|^2] &\leq \left[\max_{j\in\setcomb{m}}L^{F_j}_1\right]\bbE_t[\|z^t-z_j^t\|^2 + \|x^t-x_j^t\|^2]\\
    &\leq \left[\max_{j\in\setcomb{m}}L^{F_j}_1\right]\frac1m\sum_{j=1}^m(\|z^t-z_j^t\|^2 + \|x^t-x_j^t\|^2)\enspace .
\end{align}
}

\new{
The control of $\bbE_t[\| \nabla_{11}^2 G_i(z^t, x^t)v^t -\nabla_{11}^2 G_i(z^t_i, x^t_i)v^t_i\|^2]$ is a bit harder without assuming the boundness of $v$ beforehand. But, we can bypass the difficulty by introducing $\nabla^2_{11} G_i(z^*(x^t), x^t)v^*(x^t)$:
\begin{align}
    \bbE_t[\| \nabla_{11}^2 G_i(z^t, x^t)v^t -\nabla_{11}^2 G_i(z^t_i, x^t_i)v^t_i\|^2] &\leq 4\{\bbE_t[\|\nabla_{11}^2 G_i(z^t, x^t)(v^t-v^*(x^t))\|^2]\\\nonumber
    &\qquad + \bbE_t[\|(\nabla_{11}^2 G_i(z^t, x^t) - \nabla_{11}^2 G_i(z^*(x^t), x^t))v^*(x^t)\|^2]\\\nonumber
    &\qquad + \bbE_t[\|(\nabla_{11}^2 G_i(z^*(x^t), x^t) - \nabla_{11}^2 G_i(z^t_i, x^t_i))v^*(x^t)\|^2]\\\nonumber
    &\qquad + \bbE_t[\|\nabla_{11}^2 G_i(z^t_i, x^t_i)(v^*(x^t)-v^t_i)\|^2]\}\\
    &\leq 4((\max_{i\in\setcomb{n}}L^{G_i}_1)\|v^t-v^*(x^t)\|^2 + (\max_{i\in\setcomb{n}}L^{G_i}_2)\frac{C_F}{\mu_G}\|z^t-z^*(x^t)\|^2\\\nonumber
    &\qquad + (\max_{i\in\setcomb{n}}L^{G_i}_2)\frac{C_F}{\mu_G}(\|x^t-x^t_i\|^2 + 2(\|z^t-z^*(x^t)\|^2 + \|z^t-z_i^t\|^2))\\\nonumber
    &\qquad + (\max_{i\in\setcomb{n}}L^{G_i}_1)(\|x^t-x^t_i\|^2 + 2(\|v^t-v^*(x^t)\|^2 + \|v^t-v^t_i\|^2))
\end{align}
Let $L_v' = 4\max\left(2\max_{i\in\setcomb{n}}L^{G_i}_1,  2\max_{i\in\setcomb{n}}L^{G_i}_2\frac{C_F}{\mu_G}, \max_{j\in\setcomb{m}}L^{F_j}_1\right)$ and $L_v'' = 4\max\left(3\max_{i\in\setcomb{n}}L^{G_i}_1), 3\max_{i\in\setcomb{n}}L^{G_i}_2)\frac{C_F}{\mu_G}\right)$. Taking the total expectation and putting all together yields}
\begin{equation}\label{eq:bound_variance_saga_v}
\boxed{
\bbE[\|D_v^t(z^t,v^t,x^t)\|^2] \leq 2\bbE[\|D_v(z^t, v^t, x^t)\|^2] + 2L_v'(E_z^t + E_x^t+ E_v^t + E_z'^t + E_x'^t) \new{+ 2L_v''(\delta^t_z + \delta^t_v)}\enspace.}
\end{equation}

In $x$ we have similarly
\begin{equation}\label{eq:bound_variance_saga_x}
\boxed{\bbE[\|D_x^t(z^t,v^t,x^t)\|^2] \leq 2\bbE[\|D_x(z^t,v^t,x^t)\|^2] + 2L_x'(E_z^t + E_x^t+ E_v^t + E_z'^t + E_x'^t) \new{+ 2L_x''(\delta^t_z + \delta^t_v)}\enspace.}
\end{equation}
\end{proof}

We now form $S^t = E_z^t + E_x^t+ E_v^t + E_z'^t + E'^t_v + E_x'^t$, and letting $\Gamma = \min (\frac1m, \frac1n)$.
Note that by definition, each quantity $E_z^t$ is smaller than $S^t$.

We will therefore use the cruder bounds on $\bbE[\|D_z^t\|^2]$, $\bbE[\|D_v^t\|^2]$ and $\bbE[\|D_x^t\|^2]$ as follows thanks to \cref{lemma:smoothness} and \cref{app:lemma:bound_variance}
\begin{equation}\label{app:eq:bound_variance_z_saba}
    \bbE[\|D_z^t(z^t,v^t,x^t)\|^2] \leq 2L^2_z\delta_z^t + 2L_z'S^t\enspace,
\end{equation}

\begin{equation}\label{app:eq:bound_variance_v_saba}
    \bbE[\|D_v^t(z^t,v^t,x^t)\|^2] \leq 2(L^2_v + L''_v)(\delta_z^t + \delta_v^t) + 2L_v'S^t
\end{equation}

and
\begin{equation}\label{app:eq:bound_variance_x_saba}
    \bbE[\|D_x^t(z^t,v^t,x^t)\|^2] \leq 2\bbE[\|D_x\|^2] + 2L_x'S^t + 2L_x''(\delta^t_z + \delta^t_v)\enspace .
\end{equation}

We have the following lemma
\begin{lemma}
If $4\rho^2(L_z'+L_v') + 4\gamma^2L_x'\leq \frac\Gamma 2$ and \new{$4L_x''\gamma^2\leq \rho^2(L_v^2+4L_v'')$}, it holds
$$S^{t+1} \leq \left(1-\frac\Gamma 2\right)S^t +\beta_{sz} \rho^2\delta_z^t+\beta_{sv}\rho^2\delta_v^t + P\gamma^2\bbE[\|D_x\|^2]$$
for some $L_s,\beta_{sz},P >0$.
\end{lemma}
\begin{proof}
It holds following eq.~\eqref{eq:bound_sum_memory_z} (and omitting the dependencies in $(z^t,v^t,x^t)$ in the direction for simplicity)
\begin{align*}
S^{t+1}&\leq \left(1-\Gamma\right)S^t + \bbE\left[2\rho^2(\|D_z^t\|^2 + \|D_v^t\|^2) + 2\gamma^2\|D_x^t\|^2\right.\\
&\qquad\left.+ 2(m + n)[\rho^2(\|D_z\|^2 + \|D_v\|^2) + \gamma^2\|D_x\|^2]\right]\enspace.
\end{align*}

Using the previous bounds \eqref{eq:bound_variance_saga_z}, \eqref{eq:bound_variance_saga_v} and \eqref{eq:bound_variance_saga_x}, we get
\begin{align*}
    S^{t+1}&\leq \left(1 -\Gamma + 4\rho^2(L_z' + L_v') + 4\gamma^2L_x'\right)S^t + (2(m + n) + 4)\bbE[\rho^2(\|D_z\|^2 + \|D_v\|^2)\\\nonumber
    &\qquad +\gamma^2\|D_x\|^2] \new{+ 4L_v''\rho^2(\delta^t_z+\delta^t_v) + 4L_x''\gamma^2(\delta^t_z+\delta^t_v)}\enspace .
\end{align*}

Next,
using $4\rho^2(L_z' + L_v')  + 4\gamma^2L_x' \leq \frac\Gamma 2$ and letting $P = (2(m + n) + 4)$ we get
$$
S^{t+1} \leq \left(1-\frac\Gamma 2\right)S^t + P\bbE[\rho^2(\|D_z\|^2 + \|D_v\|^2)+\gamma^2\|D_x\|^2] + \new{+ 4L_v''\rho^2(\delta^t_z+\delta^t_v) + 4L_x''\gamma^2(\delta^t_z+\delta^t_v)}\enspace.
$$
To finish, we use \cref{lemma:smoothness} to get
$$
S^{t+1} \leq \left(1-\frac\Gamma 2\right)S^t + P[\rho^2((L_z^2  + L_v^2)\delta_z^t+L_v^2\delta_v^t) \new{+(4L''_v\rho^2+4L''_x\gamma^2)(\delta^t_z+\delta^t_v)}+\new{\gamma^2\bbE[\|D_x\|^2]}]\enspace .
$$

Then, using that \new{$4L_x''\gamma^2\leq \rho^2(L_v^2+4L_v'')$}, we get the bound, letting $L_{sz} = L_z^2 + L_v^2+4L_v''$ and $L_{sv} = L_v^2+4L_v''$:
$$
\boxed{
    S^{t+1} \leq \left(1-\frac\Gamma 2\right)S^t +\beta_{sz} \rho^2\delta_z^t+\beta_{sv}\rho^2\delta_v^t + P\gamma^2\bbE[\|D_x\|^2]
}
$$
with $\beta_{sz} = 2PL_{sz}$, $\beta_{sv} = 2PL_{sv}$
\end{proof}

\subsubsection{Putting it all together}

Recall that we denote $g^t = \bbE[\|\nabla h(x^t)\|^2]$ and $h^t = \bbE[h(x^t)]$.
In the following lemma, we adapt \cref{lemma:coupled_inequalities} and \cref{lemma:descent_lemma} to the SABA algorithm.

\begin{lemma}\label{app:lemma:coupled_inequality_saba}
\new{
If
$$
\rho\leq\min\left(\frac{\mu_G}{64L_z^2},\frac{\overline{\beta}_{zx}}{2\beta_{zx}},\frac{\mu_G}{128(L_v^2 + L_v'')},\frac{\beta_{vz}}{8(L_v^2+L_v'')},\frac{\overline{\beta}_{vx}}{2\beta_{vx}}\right)
$$
and
$$
\gamma\leq\min\left(\sqrt{\frac{\rho\mu_G}{64\beta_{zx}L_x''}}, \sqrt{\frac{L_z'}{2L_x'\beta_{zx}}}\rho,\sqrt{\frac{\rho\mu_G}{128\beta_{vx}L_x''}},\sqrt{\frac{\rho\beta_{vz}}{4L_x''\beta_{vx}}},\sqrt{\frac{L_v'}{2L_x'\beta_{vx}}}\rho,\frac1{4L^h},\frac{L_x^2}{2L^hL_x''} \right)
$$
then it holds
\begin{equation}\label{app:eq:decrease_z_saba}
    \delta^{t+1}_z \leq \left(1-\frac{\rho\mu_G}8\right)\delta^t_z + 2L_x''\beta_{zx}\gamma^2\delta^t_v+5L_z'\rho^2S^t + 2\overline{\beta}_{zx}\frac{\gamma^2}\rho\bbE[\|D_x(z^t,v^t,x^t)\|^2]\enspace,
\end{equation}
}

\new{
\begin{equation}\label{app:eq:decrease_v_saba}
\delta^{t+1}_v \leq \left(1-\frac{\rho\mu_G}{16}\right)\delta_v^t +   3\beta_{vz}\rho\delta_z^t + 5L_v'\rho^2S^t +2\overline{\beta}_{vx}\frac{\gamma^2}{\rho} \bbE[\|D_x(z^t,v^t,x^t)\|^2]
\end{equation}
}
and
\new{
\begin{equation}\label{app:eq:decrease_h_saba}
h^{t+1}\leq h^t - \frac\gamma2 g^t - \frac\gamma4\bbE[\|D_x(z^t,v^t,x^t)\|^2] + L_x^2\gamma(\delta^t_z + \delta^t_v) + L^hL_x'\gamma^2 S^t\enspace.
\end{equation}
}
\end{lemma}

\begin{proof}
\new{
We start from \cref{lemma:coupled_inequalities} and plug the bounds of Equations \eqref{app:eq:bound_variance_z_saba} and \eqref{app:eq:bound_variance_v_saba}.
\begin{align}\label{app:eq:z_decrease_saba_1}
    \delta^{t+1}_z &\leq \left(1-\frac{\rho\mu_G}4 + 4L_z^2\rho^2 + 4\beta_{zx}L_x''\gamma^2\right)\delta^t_z + 2L_x''\beta_{zx}\gamma^2\delta^t_v\\\nonumber
    &\qquad+(4L_z'\rho^2+2L_x'\beta_{zx}\gamma^2)S^t + \left(2\beta_{zx}\gamma^2 + \overline{\beta}_{zx}\frac{\gamma^2}\rho\right)\bbE[\|D_x(z^t,v^t,x^t)\|^2]
\end{align}
Since $\rho\leq \frac{\mu_G}{64L_z^2}$ and $\gamma^2\leq\frac{\rho\mu_G}{64\beta_{zx}L_x''}$, we have
\begin{equation}\label{app:eq:simplify_dz_1}
    -\frac{\rho\mu_G}4 + 4L_z^2\rho^2 + 4\beta_{zx}L_x''\gamma^2\leq -\frac{\rho\mu_G}8\enspace.
\end{equation}
The condition $\gamma^2\leq\frac{L_z'}{2L_x'\beta_{zx}}\rho^2$ gives us
\begin{equation}\label{app:eq:simplify_dz_2}
    4L_z'\rho^2+2L_x'\beta_{zx}\gamma^2 \leq 5L_z'\rho^2\enspace.
\end{equation}
With $\rho\leq \frac{\overline{\beta}_{zx}}{2\beta_{zx}}$, we get
\begin{equation}\label{app:eq:simplify_dz_3}
    2\beta_{zx}\gamma^2 + \overline{\beta}_{zx}\frac{\gamma^2}\rho \leq 2\overline{\beta}_{zx}\frac{\gamma^2}\rho\enspace.
\end{equation}
We can plug Equations \eqref{app:eq:simplify_dz_1}, \eqref{app:eq:simplify_dz_2} and \eqref{app:eq:simplify_dz_3} into \cref{app:eq:z_decrease_saba_1} and we end up with
$$
     \delta^{t+1}_z \leq \left(1-\frac{\rho\mu_G}8\right)\delta^t_z + 2L_x''\beta_{zx}\gamma^2\delta^t_v+5L_z'\rho^2S^t + 2\overline{\beta}_{zx}\frac{\gamma^2}\rho\bbE[\|D_x(z^t,v^t,x^t)\|^2]\enspace.
$$
The proof for $\delta^t_v$ is quite similar. From \cref{lemma:coupled_inequalities}, Equations \eqref{app:eq:bound_variance_v_saba} and \eqref{app:eq:bound_variance_x_saba}.
\begin{align}\label{app:eq:v_decrease_saba_1}
\delta^{t+1}_v &\leq \left(1-\frac{\rho\mu_G}8\right)\delta_v^t +   \beta_{vz}\rho\delta_z^t +2\rho^2V_v^t +  \beta_{vx}\gamma^2V_x^t +\overline{\beta}_{vx}\frac{\gamma^2}{\rho} \bbE[\|D_x(z^t,v^t,x^t)\|^2]\\
    &\leq \left(1-\frac{\rho\mu_G}8 + 4(L_v^2+L_v'')\rho^2 + 4L_x''\beta_{vx}\gamma^2\right)\delta_v^t +   (4(L_v^2+L_v'')\rho^2 + 2L_x''\beta_{vx}\gamma^2 + \beta_{vz}\rho)\delta_z^t + \\\nonumber &\qquad+\left(4L_v'\rho^2+2L_x'\beta_{vx}\gamma^2\right)S^t+\left(2\beta_{vx}\gamma^2+\overline{\beta}_{vx}\frac{\gamma^2}{\rho}\right) \bbE[\|D_x(z^t,v^t,x^t)\|^2]\enspace.
\end{align}
Using $\rho\leq\frac{\mu_G}{128(L_v^2 + L_v'')}$ and $\gamma^2\leq\frac{\rho\mu_G}{128L_x''\beta_{vx}}$, we get
\begin{align}\label{app:eq:simplify_dv_1}
    -\frac{\rho\mu_G}8 + 4(L_v^2+L_v'')\rho^2 + 4L_x''\beta_{vx}\gamma^2\leq -\frac{\rho\mu_G}{16}\enspace.
\end{align}
With $\gamma^2\leq\frac{\rho\beta_{vz}}{4L_x''\beta_{vx}}$ and $\rho\leq\frac{\beta_{vz}}{8(L_v^2+L_v'')}$, we have
\begin{align}\label{app:eq:simplify_dv_2}
    4(L_v^2+L_v'')\rho^2 + 2L_x''\beta_{vx}\gamma^2 + \beta_{vz}\rho\leq 3\beta_{vz}\rho\enspace.
\end{align}
The condition $\gamma^2\leq\frac{L'_v}{2L_x'\beta_{vx}}\rho^2$ yields
\begin{align}\label{app:eq:simplify_dv_3}
    4L_v'\rho^2 + 2L_x'\beta_{vx}\gamma^2  \leq 5L_v'\rho^2\enspace.
\end{align}
With $\rho\leq\frac{\overline{\beta}_{vx}}{2\beta_{vx}}$ we get
\begin{align}\label{app:eq:simplify_dv_4}
   2\beta_{vx}\gamma^2+\overline{\beta}_{zx}\frac{\gamma^2}{\rho} \leq 2\overline{\beta}_{vx}\frac{\gamma^2}\rho\enspace.
\end{align}
As a consequence of Equations \eqref{app:eq:v_decrease_saba_1}, \eqref{app:eq:simplify_dv_1},  \eqref{app:eq:simplify_dv_2}, \eqref{app:eq:simplify_dv_3} and \eqref{app:eq:simplify_dv_4}, we have
$$
    \delta^{t+1}_v \leq \left(1-\frac{\rho\mu_G}{16}\right)\delta_v^t +   3\beta_{vz}\rho\delta_z^t + 5L_v'\rho^2S^t +2\overline{\beta}_{vx}\frac{\gamma^2}{\rho} \bbE[\|D_x(z^t,v^t,x^t)\|^2]\enspace.
$$
For the inequality on $h^t$, we start from Equations \eqref{eq:descent_h} and \eqref{app:eq:bound_variance_x_saba}
\begin{align}
    h^{t+1}&\leq h^t - \frac\gamma2 g^t - \left(\frac\gamma2 - L^h\gamma^2\right)\bbE[\|D_x(z^t,v^t,x^t)\|^2] \\\nonumber
    &\qquad + \left(\frac{L_x^2}2\gamma + L^hL_x''\gamma^2\right)(\delta^t_z + \delta^t_v) + L^hL_x'\gamma^2 S^t\enspace.
\end{align}
Assuming $\gamma\leq\min\left(\frac1{4L^h}, \frac{L_x^2}{2L^hL_x''}\right)$ leads
\begin{equation}
        h^{t+1}\leq h^t - \frac\gamma2 g^t - \frac\gamma4\bbE[\|D_x(z^t,v^t,x^t)\|^2] + L_x^2\gamma(\delta^t_z + \delta^t_v) + L^hL_x'\gamma^2 S^t\enspace.
\end{equation}
}
\end{proof}

We are now ready to prove \cref{th:cvg_saba_smooth}.

\begin{proof}
We consider the Lyapunov function
\begin{equation}\label{eq:lyap_def}
\mathcal{L}^t = h^t + \phi_s S^t + \phi_z \delta_z^t + \phi_v\delta_v^t
\end{equation}
for some constants $\phi_s$, $\phi_z$ and $\phi_v$.

We have
\new{
\begin{align*}
    \mathcal{L}^{t+1} -\mathcal{L}^t &\leq -\frac\gamma2 g^t - \left(\frac\gamma4-2\phi_z\overline{\beta}_{zx}\frac{\gamma^2}\rho -2\phi_v\overline{\beta}_{vx}\frac{\gamma^2}\rho - \phi_sP\gamma^2 \right)\bbE[\|D_x(z^t,v^t,x^t)\|^2]\\
    &\qquad - \left(\phi_z\frac{\mu_G}8\rho - L_x^2\gamma - 8\phi_v\beta_{vz}\rho - \phi_s\beta_{sz}\rho^2\right)\delta^t_z\\
    &\qquad - \left(\phi_v\frac{\mu_G}{16}\rho - L_x^2\gamma - 2\phi_zL_x''\gamma^2 - \phi_s\beta_{sv}\rho^2\right)\delta^t_v\\
    &\qquad -\left(\phi_s\frac\Gamma2 - 5\phi_zL_z'\rho^2-5\phi_vL_v'\rho^2-L^hL_x'\gamma^2\right)S^t\enspace.
\end{align*}
}

\new{
To get a decrease, $\phi_z$, $\phi_v$ and $\phi_s$, $\rho$ and $\gamma$ must be such that:
\begin{align*}
   2\phi_z\overline{\beta}_{zx}\frac{\gamma^2}\rho  + 2\phi_v\overline{\beta}_{vx}\frac{\gamma^2}\rho + \phi_sP\gamma^2 &\leq\frac\gamma4\\
   L_x^2\gamma + 8\phi_v\beta_{vz}\rho + \phi_s\beta_{sz}\rho^2&\leq\phi_z\frac{\mu_G}8\rho\\
   L_x^2\gamma + 8\phi_zL_x''\gamma^2 + \phi_s\beta_{sv}\rho^2&\leq\phi_v\frac{\mu_G}{16}\rho\\
    5\phi_zL_z'\rho^2+5\phi_vL_v'\rho^2+L^hL_x'\gamma^2 &\leq \phi_s\frac\Gamma2\enspace.
\end{align*}
}

\new{
In order to take into account the scaling of the quantities with respect to $N = n + m$, we take $\rho = \rho'N^{n_\rho}$, $\gamma = \gamma'N^{n_\gamma}$, $\phi_z = \phi_z'N^{n_z}$, $\phi_v = \phi_v'N^{n_v}$ and $\phi_s = \phi_s'N^{n_s}$. Since $\Gamma = \calO(N^{-1})$, $P=\calO(N)$, $\beta_{sz} = \calO(N)$ and $\beta_{sv} = \calO(N)$, we also define $\Gamma' = \Gamma N$, $P' = PN^{-1}$, $\beta_{sz}' = \beta_{sz}N^{-1}$ and $\beta_{sv}'N^{-1}$. Now, the previous Equations read (after slight simplifications):
\begin{align*}
   (2\phi_z'\overline{\beta}_{zx}  + 2\phi_v'\overline{\beta}_{vx})\frac{\gamma'}{\rho'}N^{n_z+n_\gamma-n_\rho} + \phi_s'P'\gamma'N^{n_s+n_\gamma+1} &\leq\frac14\\
   L_x^2\gamma'N^{n_\gamma} + 8\phi_v'\beta_{vz}\rho'N^{n_v+n_\rho} + \phi_s'\beta_{sz}'(\rho')^2 N^{2n_\rho+n_s+1}&\leq\phi_z'\frac{\mu_G}8\rho'N^{n_z+n_\rho}\\
   L_x^2\gamma'N^{n_\gamma} + 8\phi_z'L_x''(\gamma')^2N^{2n_\gamma+n_z} + \phi_s'\beta_{sv}'(\rho')^2N^{n_s+2n_\rho+1}&\leq\phi_v'\frac{\mu_G}{16}\rho'N^{n_v+n_\rho}\\
    5\phi_z'L_z'(\rho')^2N^{n_z+2n_\rho}+5\phi_v'L_v'(\rho')^2N^{2n_\rho+n_v}+L^hL_x'(\gamma')^2N^{n_\gamma} &\leq \phi_s\frac{\Gamma'}2N^{n_s-1}\enspace .
\end{align*}
In order to ensure that the exponents on $N$ are lower in the left-hand-side than those on the right-hand-side, we take $n_z = n_v = 0$, $n_\rho = n_\gamma = -\frac{2}{3}$ and $n_s =- \frac{1}{3}$. The Equations become
\begin{align*}
   (2\phi_z'\overline{\beta}_{zx}  + 2\phi_v'\overline{\beta}_{vx})\frac{\gamma'}{\rho'} + \phi_s'P'\gamma' &\leq\frac14\\
   L_x^2\gamma'N^{-2/3} + 8\phi_v'\beta_{vz}\rho'N^{-2/3} + \phi_s'\beta_{sz}'(\rho')^2 N^{-2/3}&\leq\phi_z'\frac{\mu_G}8\rho'N^{-2/3}\\
   L_x^2\gamma'N^{-2/3} + 8\phi_z'L_x''(\gamma')^2N^{-4/3} + \phi_s'\beta_{sv}'(\rho')^2N^{-2/3}&\leq\phi_v'\frac{\mu_G}{16}\rho'N^{-2/3}\\
    5\phi_z'L_z'(\rho')^2N^{-4/3}+5\phi_v'L_v'(\rho')^2N^{-4/3}+L^hL_x'(\gamma')^2N^{-4/3} &\leq \phi_s'\frac{\Gamma'}2N^{-4/3}\enspace.
\end{align*}
We can replace the penultimate equation by the stronger
$$
 L_x^2\gamma'N^{-2/3} + 8\phi_z'L_x''(\gamma')^2N^{-2/3} + \phi_s'\beta_{sv}'(\rho')^2N^{-2/3}\leq\phi_v'\frac{\mu_G}{16}\rho'N^{-2/3}
$$
so that we can simplify all the equations by dropping the dependencies in $N$:
\begin{align*}
   (2\phi_z'\overline{\beta}_{zx}  + 2\phi_v'\overline{\beta}_{vx})\frac{\gamma'}{\rho'} + \phi_s'P'\gamma' &\leq\frac14\\
   L_x^2\gamma' + 8\phi_v'\beta_{vz}\rho' + \phi_s'\beta_{sz}'(\rho')^2 &\leq\phi_z'\frac{\mu_G}8\rho'\\
   L_x^2\gamma' + 8\phi_z'L_x''(\gamma')^2 + \phi_s'\beta_{sv}'(\rho')^2&\leq\phi_v'\frac{\mu_G}{16}\rho'\\
    5\phi_z'L_z'(\rho')^2+5\phi_v'L_v'(\rho')^2+L^hL_x'(\gamma')^2 &\leq \phi_s'\frac{\Gamma'}2\enspace.
\end{align*}
}

\new{
Let us take $\phi'_s = 1$, $\phi_z' = \phi_z''\frac{\rho'}{\gamma'}$ and $\phi_v' = \phi_v''\frac{\rho'}{\gamma'}$ with $\phi_z'' = \frac1{32\overline{\beta}_{zx}}$ and $\phi_v'' = \min\left(\frac1{32\overline{\beta}_{vx}}, \phi_z''\frac{\mu_G}{128\beta_{vz}}\right)$. The equations become
\begin{align*}
   P'\gamma' &\leq\frac18\\
   L_x^2\gamma' + \beta_{sz}'(\rho')^2 &\leq\phi_z''\frac{\mu_G}{16}\frac{(\rho')^2}{\gamma'}\\
   L_x^2\gamma' + 8\phi_z''L_x''\gamma'\rho'+ \beta_{sv}'(\rho')^2&\leq\phi_v''\frac{\mu_G}{16}\frac{(\rho')^2}{\gamma'}\\
    5\phi_z''L_z'\frac{(\rho')^3}{\gamma'}+5\phi_v''L_v'\frac{(\rho')^3}{\gamma'}+L^hL_x'(\gamma')^2 &\leq \frac{\Gamma'}2\enspace.
\end{align*}
The condition $\gamma'\leq \frac1{8P'}$ ensures that the first equation is verified. With $\gamma'\leq\min\left(\sqrt{\frac{\phi_z''\mu_G}{32L_x^2}}\rho', \frac{\phi_z''\mu_G}{32\beta'_{sz}}\right)$, the second equations is verified. With $\gamma'\leq\min\left(\sqrt{\frac{\phi_v''\mu_G}{48L_x^2}}\rho', \frac{\phi_v''\mu_G}{48\beta'_{sv}}, \sqrt{\frac{\phi_v''\mu_G}{384\phi_z''L_x''\rho'}}\right)$, the third is verified. With $\gamma'\leq\sqrt{\frac{\Gamma'}{6L^hL_x'}}$, the last can be simplified:
$$
(5\phi_z''L_z'+5\phi_v''L_v')(\rho')^3\leq \frac{\Gamma'}3\gamma'\enspace.
$$
Let us write $\gamma' = \xi\rho'$. If we want that equation does no contradict the previous upper bound on $\gamma'$ involving $\rho'$ and the conditions of \cref{app:lemma:coupled_inequality_saba}, that is
\begin{align*}
    \gamma' &\leq \underbrace{\min\left(\sqrt{\frac{\phi_z''\mu_G}{32L_x^2}}, \sqrt{\frac{\phi_v''\mu_G}{48L_x^2}}, \sqrt{\frac{L_z'}{2L_x'\beta_{zx}}},\sqrt{\frac{L_v'}{2L_x'\beta_{vx}}}\right)}_{K_1}\rho'\\
        \gamma' &\leq \underbrace{\min\left(\sqrt{\frac{\mu_G}{64\beta_{zx}L_x''}},\sqrt{\frac{\mu_G}{128\beta_{vx}L_x''}},\sqrt{\frac{\beta_{vz}}{4L_x''\beta_{vx}}}\right)}_{K_2}\sqrt{\rho'}\\
    \gamma' &\leq \underbrace{\sqrt{\frac{\phi_v''\mu_G}{384\phi_z''L_x''}}}_{K_3}\frac1{\sqrt{\rho'}}\\
    \gamma'&\leq\underbrace{\min\left(\frac1{4L^h},\frac{L_x^2}{2L^hL_x''}, \sqrt{\frac{\Gamma'}{6L^hL_x'}},\frac1{8P'},\frac{\phi_z''\mu_G}{32\beta'_{sz}},\frac{\phi_v''\mu_G}{48\beta'_{sv}}\right)}_{K_4}\\
    \gamma' &\geq \underbrace{\frac{15(\phi_z''L_z'+\phi_v''L_v')}{\Gamma'}}_{K_5}\rho^3
\end{align*}
$\xi$ must verify
\begin{align*}
    \xi &\leq K_1\\
    \xi &\leq K_2(\rho')^{-\frac12}\\
    \xi &\leq K_3(\rho')^{-\frac32}\\
    \xi &\leq K_4(\rho')^{-1}\\
    \xi &\geq K_5(\rho')^2
\end{align*}
which is possible if $\rho'$ satisfies
\begin{align*}
    \rho'\leq\min\left(\sqrt{\frac{K_1}{K_5}},\left(\frac{K_2}{K_5}\right)^{-\frac32},\left(\frac{K_3}{K_5}\right)^{-\frac52},\left(\frac{K_4}{K_5}\right)^{-2}\right)\enspace.
\end{align*}
Let us take
\begin{align}
    \rho' = \min\left(\sqrt{\frac{K_1}{K_5}},\left(\frac{K_2}{K_5}\right)^{-\frac32},\left(\frac{K_3}{K_5}\right)^{-\frac52},\left(\frac{K_4}{K_5}\right)^{-2},\frac{\mu_G}{64L_z^2},\frac{\overline{\beta}_{zx}}{2\beta_{zx}},\frac{\mu_G}{128(L_v^2 + L_v'')},\frac{\beta_{vz}}{8(L_v^2+L_v'')},\frac{\overline{\beta}_{vx}}{2\beta_{vx}}\right)
\end{align}
and
\begin{align}
    \xi = \min(K_1, K_2(\rho')^{-\frac12}, K_3(\rho')^{-\frac32},K_4(\rho')^{-1})\enspace.
\end{align}
Finally, we have
$$
    \calL^{t+1}-\calL^t \leq -\frac\gamma2 g^t
$$
and therefore, summing and telescoping yields
\begin{equation*}
    \frac1T\sum_{t=1}^{T}g^t \leq \frac{\calL^1}{\gamma T} = \frac{\calL^0N^{\frac23}}T\enspace.
\end{equation*}
Since with respect to $N$ we have
$$
    \calL^0 = h^0 + \phi_z\delta^0_z + \phi_v\delta^0_v + \phi_s S^0 = \calO(N^{-1} + 1 + 1 + N^{-\frac13}) = \calO(1)\enspace,
$$
we end up with
\begin{equation*}
\boxed{
     \frac1T\sum_{t=1}^{T}\bbE[\|\nabla h(x^t)\|^2] = \calO\left(\frac{N^{\frac23}}T\right)\enspace.
}
\end{equation*}
}

\end{proof}
\subsection{Proof of \cref{th:cvg_saba_pl}}

We are now going to prove \cref{th:cvg_saba_pl} that we recall here:
\sabapl*

Here, we have
\new{
\begin{align*}
    \rho' = \min\left(\sqrt{\frac{K_1'}{K_5}},\left(\frac{K_2}{K_5}\right)^{\frac25},\left(\frac{K_3}{K_5}\right)^{\frac27},\left(\frac{K_4'}{K_5}\right)^{\frac13},\frac{\mu_G}{64L_z^2},\frac{\overline{\beta}_{zx}}{2\beta_{zx}},\frac{\mu_G}{128(L_v^2 + L_v'')},\frac{\beta_{vz}}{8(L_v^2+L_v'')},\frac{\overline{\beta}_{vx}}{2\beta_{vx}}\right)\enspace,
\end{align*}
}
and
\new{
$$
\xi = \min(K_1', K_2(\rho')^{-\frac12}, K_3(\rho')^{-\frac32},K_4'(\rho')^{-1})\enspace.
$$}
where \new{$P' = PN^{-1}$}, \new{$\Gamma' = \Gamma N$},
\new{
$$
    \phi_z'' = \frac1{32\overline{\beta}_{zx}}\enspace,
    \phi_v'' = \min\left(\frac1{32\overline{\beta}_{vx}}, \phi_z''\frac{\mu_G}{128\beta_{vz}}\right)\enspace,
$$
\begin{align*}
    K_1' &= \min\left(\frac{\mu_G}{64c'},\sqrt{\frac{\phi_z''\mu_G}{48L_x^2}}, \sqrt{\frac{\phi_v''\mu_G}{64L_x^2}}, \sqrt{\frac{L_z'}{2L_x'\beta_{zx}}},\sqrt{\frac{L_v'}{2L_x'\beta_{vx}}}\right)\enspace,\\
    &K_2 = \min\left(\sqrt{\frac{\mu_G}{64\beta_{zx}L_x''}},\sqrt{\frac{\mu_G}{128\beta_{vx}L_x''}},\sqrt{\frac{\beta_{vz}}{4L_x''\beta_{vx}}}\right)\enspace,\\
    K_3 = \sqrt{\frac{\phi_v''\mu_G}{512\phi_z''L_x''}}\enspace,&\quad
    K_4' =\min\left(\frac{\Gamma'}{6c'},\frac1{4L^h},\frac{L_x^2}{2L^hL_x''},\sqrt{\frac{\Gamma'}{6L^hL_x'}},\frac1{18P'},\frac{\phi_z''\mu_G}{48\beta'_{sz}},\frac{\phi_v''\mu_G}{64\beta'_{sv}} \right)
\end{align*}
and
$$
    K_5 = \frac{20(\phi_z''L_z'+\phi_v''L_v')}{\Gamma'}\enspace.
$$
}
\begin{proof}

For simplicity, we assume that $h^* = 0$ and so for any $x\in\bbR^d$ the PL inequality reads:
\begin{equation}\label{eq:pl_inequality}
   \frac12\|\nabla h(x)\|^2\geq \mu_hh(x) \enspace .
\end{equation}

Then, eq.~\eqref{app:eq:decrease_h_saba} gives
$$
h^{t+1}\leq \left(1 -\frac{\gamma\mu_h}2\right)h^t\new{ - \frac\gamma4\bbE[\|D_x(z^t,v^t,x^t)\|^2]} + \gamma L_x^2(\delta_z^t + \delta_v^t) + L^hL_x' \gamma^2 S^t\enspace.
$$

We take $\calL^t$ the Lyapunov function given in \cref{eq:lyap_def}. We find
\new{
\begin{align*}
    \mathcal{L}^{t+1} -\mathcal{L}^t &\leq -\gamma\mu_h h^t - \left(\frac\gamma4-2\phi_z\overline{\beta}_{zx}\frac{\gamma^2}\rho -2\phi_v\overline{\beta}_{vx}\frac{\gamma^2}\rho - \phi_sP\gamma^2 \right)\bbE[\|D_x(z^t,v^t,x^t)\|^2]\\
    &\qquad - \left(\phi_z\frac{\mu_G}8\rho - L_x^2\gamma - 8\phi_v\beta_{vz}\rho - \phi_s\beta_{sz}\rho^2\right)\delta^t_z\\
    &\qquad - \left(\phi_v\frac{\mu_G}{16}\rho - L_x^2\gamma - 2\phi_zL_x''\gamma^2 - \phi_s\beta_{sv}\rho^2\right)\delta^t_v\\
    &\qquad -\left(\phi_s\frac\Gamma2 - 5\phi_zL_z'\rho^2-5\phi_vL_v'\rho^2-L^hL_x'\gamma^2\right)S^t\enspace.
\end{align*}
}

We now try to find linear convergence, hence we add to this $c\mathcal{L}^t$ to get

\new{
\begin{align*}
    \mathcal{L}^{t+1} - (1-c)\mathcal{L}^t &\leq -(\gamma\mu_h - c) h^t - \left(\frac\gamma4-2\phi_z\overline{\beta}_{zx}\frac{\gamma^2}\rho -2\phi_v\overline{\beta}_{vx}\frac{\gamma^2}\rho - \phi_sP\gamma^2 - c\right)\bbE[\|D_x(z^t,v^t,x^t)\|^2]\\
    &\qquad - \left(\phi_z\frac{\mu_G}8\rho - L_x^2\gamma - 8\phi_v\beta_{vz}\rho - \phi_s\beta_{sz}\rho^2 - c\phi_z\right)\delta^t_z\\
    &\qquad - \left(\phi_v\frac{\mu_G}{16}\rho - L_x^2\gamma - 2\phi_zL_x''\gamma^2 - \phi_s\beta_{sv}\rho^2 - c\phi_v\right)\delta^t_v\\
    &\qquad -\left(\phi_s\frac\Gamma2 - 5\phi_zL_z'\rho^2-5\phi_vL_v'\rho^2-L^hL_x'\gamma^2 - c\phi_S\right)S^t\enspace.
\end{align*}
}

Hence, the set of inequations for  decrease becomes
\new{
\begin{align*}
    c &\leq \gamma\mu_h\\
   2\phi_z\overline{\beta}_{zx}\frac{\gamma^2}\rho  + 2\phi_v\overline{\beta}_{vx}\frac{\gamma^2}\rho + \phi_sP\gamma^2 + c &\leq\frac\gamma4\\
   L_x^2\gamma + 8\phi_v\beta_{vz}\rho + \phi_s\beta_{sz}\rho^2 + \phi_zc&\leq\phi_z\frac{\mu_G}8\rho\\
   L_x^2\gamma + 8\phi_zL_x''\gamma^2 + \phi_s\beta_{sv}\rho^2 + \phi_v c&\leq\phi_v\frac{\mu_G}{16}\rho\\
    5\phi_zL_z'\rho^2+5\phi_vL_v'\rho^2+L^hL_x'\gamma^2 +\phi_s c &\leq \phi_s\frac\Gamma2\enspace.
\end{align*}
}

We see that it is more convenient to write $c = \gamma c'$. \new{As previously, we write $\gamma = \gamma'N^{n_\gamma}$, $\rho = \rho' N^{n_\rho}$, $\phi_z = \phi_z'N^{n_z}$, $\phi_v = \phi_v'N^{n_v}$, $\phi_s = \phi_s'N^{n_s}$, $P = P'N$, $\Gamma = \Gamma'N^{-1}$, $\beta_{sx} = \beta_{sx}'N$ and $\beta_{sv} = \beta_{sv}'N$. The equations read:
\new{
\begin{align*}
    c' &\leq \mu_h\\
   2\phi_z'\overline{\beta}_{zx}\frac{\gamma'}{\rho'}N^{n_z+n_\gamma-n_\rho}  + 2\phi_v'\overline{\beta}_{vx}\frac{\gamma'}{\rho'}N^{n_v+n_\gamma-n_\rho} + \phi_s'P'\gamma'N^{n_s+1+n_\gamma} + c' &\leq\frac14\\
   L_x^2\gamma'N^{n_\gamma} + 8\phi_v'\beta_{vz}\rho'N^{n_v+n_\rho} + \phi_s'\beta_{sz}'(\rho')^2N^{n_s+2n_\rho+1} + \phi_z' c'\gamma' N^{n_z+n_\gamma}&\leq\phi_z'\frac{\mu_G}8\rho'N^{n_\rho+n_z}\\
   L_x^2\gamma'N^{n_\gamma} + 8\phi_z'L_x''(\gamma')^2N^{n_z+2n_\gamma} + \phi_s'\beta_{sv}'(\rho')^2N^{n_s+1+2n_\rho} + \phi_v' c'\gamma'N^{n_v+n_\gamma}&\leq\phi_v'\frac{\mu_G}{16}\rho'N^{n_v+n_\rho}\\
    5\phi_z'L_z'(\rho')^2N^{n_z+2n_\rho}+5\phi_v'L_v'(\rho')^2N^{n_v + 2n_\rho}+L^hL_x'(\gamma')^2N^{2n_\gamma} +\phi_s' c'\gamma' N^{n_s+n_\gamma} &\leq \phi_s'\frac{\Gamma'}2N^{n_s - 1}\enspace.
\end{align*}
}
In order to ensure that the exponents on $N$ are lower in the left-hand-side than those on the right-hand-side, we take $n_z = n_v = 0$, $n_\rho = -\frac{2}{3}$, $n_\gamma = -1$ and $n_s =- \frac{1}{3}$. The Equations become
\begin{align*}
    c' &\leq \mu_h\\
   2\phi_z'\overline{\beta}_{zx}\frac{\gamma'}{\rho'}N^{-\frac13}  + 2\phi_v'\overline{\beta}_{vx}\frac{\gamma'}{\rho'}N^{-\frac13} + \phi_s'P'\gamma'N^{-\frac13} + c' &\leq\frac14\\
   L_x^2\gamma'N^{-1} + 8\phi_v'\beta_{vz}\rho'N^{-\frac23} + \phi_s'\beta_{sz}'(\rho')^2N^{-\frac23} + \phi_z' c'\gamma' N^{-1}&\leq\phi_z'\frac{\mu_G}8\rho'N^{-\frac23}\\
   L_x^2\gamma'N^{-1} + 8\phi_z'L_x''(\gamma')^2N^{-2} + \phi_s'\beta_{sv}'(\rho')^2N^{-\frac23} + \phi_v' c'\gamma'N^{-1}&\leq\phi_v'\frac{\mu_G}{16}\rho'N^{-\frac23}\\
    5\phi_z'L_z'(\rho')^2N^{-\frac43}+5\phi_v'L_v'(\rho')^2N^{-2}+L^hL_x'(\gamma')^2N^{-2} +\phi_s' c'\gamma' N^{-\frac43} &\leq \phi_s'\frac{\Gamma'}2N^{-\frac43}\enspace.
\end{align*}
}

\new{
Now we have to find $\rho'$, $\gamma'$, $\phi_z'$, $\phi_v'$ and $\phi_s'$ that verifies the following conditions (which are a bit stronger than thoose in the previous Equations):
\begin{align*}
    c' &\leq \mu_h\\
   2\phi_z'\overline{\beta}_{zx}\frac{\gamma'}{\rho'} + 2\phi_v'\overline{\beta}_{vx}\frac{\gamma'}{\rho'} + \phi_s'P'\gamma' + c' &\leq\frac14\\
   L_x^2\gamma'+ 8\phi_v'\beta_{vz}\rho' + \phi_s'\beta_{sz}'(\rho')^2 + \phi_z' c'\gamma' &\leq\phi_z'\frac{\mu_G}8\rho'\\
   L_x^2\gamma' + 8\phi_z'L_x''(\gamma')^2 + \phi_s'\beta_{sv}'(\rho')^2 + \phi_v' c'\gamma'&\leq\phi_v'\frac{\mu_G}{16}\rho'\\
    5\phi_z'L_z'(\rho')^2+5\phi_v'L_v'(\rho')^2+L^hL_x'(\gamma')^2 + \phi_s' c'\gamma'  &\leq \phi_s'\frac{\Gamma'}2\enspace.
\end{align*}
As previously, we take $\phi'_s = 1$ and we denote $\phi'_z = \phi_z''\frac{\rho'}{\gamma'}$ with $\phi_z'' = \frac1{32\overline{\beta}_{zx}}$ and $\phi'_z = \phi_z''\frac{\rho'}{\gamma'}$ with $\phi_v'' = \min\left(\frac1{32\overline{\beta}_{vx}}, \phi_z''\frac{\mu_G}{128\beta_{vz}}\right)$, the equations become
\begin{align*}
    c' &\leq \mu_h\\
    P'\gamma' + c' &\leq\frac18\\
   L_x^2(\gamma')^2+ \beta_{sz}'(\rho')^2\gamma' + \phi_z'' c'\rho'\gamma' &\leq\phi_z''\frac{\mu_G}{16}(\rho')^2\\
   L_x^2(\gamma')^2 + 8\phi_z''L_x''\rho'(\gamma')^2 + \beta_{sv}'(\rho')^2\gamma' + \phi_v'' c'\rho'\gamma'&\leq\phi_v''\frac{\mu_G}{16}(\rho')^2\\
    5\phi_z''L_z'(\rho')^3+5\phi_v''L_v'(\rho')^3+L^hL_x'(\gamma')^3 + c'(\gamma')^2  &\leq \frac{\Gamma'}2\gamma'\enspace.
\end{align*}
}

\new{
Since $c'\leq \frac1{16}$ and $\gamma'\leq\frac1{16P'}$, the second equation is verified. With $\gamma'\leq \min\left(\sqrt{\frac{\phi_z''\mu_G}{48L_x^2}}\rho', \frac{\phi_z''\mu_G}{48\beta_{sv}}\right)$ and $c'\leq \frac{\mu_G\rho'}{48\gamma'}$ the third is verified. The conditions $\gamma'\leq\min\left(\sqrt{\frac{\phi_v''\mu_G}{64L_x^2}}\rho',\sqrt{\frac{\phi_v''\mu_G}{512\phi_z''L_x''\rho'}},\frac{\phi_v''\mu_G}{64\beta_{sv}'}\right)$ and $c'\leq\frac{\mu_G\rho'}{64\gamma'}$ ensure that the forth is verified. With $\gamma'\leq\sqrt{\frac{\Gamma'}{8L^hL_x'}}$ and $c'\leq\frac{\Gamma'}{8\gamma'}$, the fifth is simplified in
$$5\phi_z''L_z'(\rho')^3+5\phi_v''L_v'(\rho')^3 \leq \frac{\Gamma'}4\gamma'\enspace .$$
}

\new{
As in the proof of \cref{th:cvg_saba_smooth}, let us denote $\gamma' = \xi\rho'$. To verify this equation and the previous bounds on $\gamma'$ and $c'$, we need
\begin{align*}
    \gamma'&\leq\underbrace{\min\left(\sqrt{\frac{\phi_z''\mu_G}{48L_x^2}},\sqrt{\frac{\phi_v''\mu_G}{64L_x^2}},\sqrt{\frac{L_z'}{2L_x'\beta_{zx}}},\sqrt{\frac{L_v'}{2L_x'\beta_{zx}}}\right)}_{K_1}\rho'\enspace,\\
    \gamma'&\leq \underbrace{\min\left(\sqrt{\frac{\mu_G}{64\beta_{zx}L_x''}},\sqrt{\frac{\mu_G}{128\beta_{vx}L_x''}},\sqrt{\frac{\beta_{vz}}{4L_x''\beta_{vx}}}\right)}_{K_2}\sqrt{\rho'}\enspace,\\
     \gamma'&\leq \underbrace{\sqrt{\frac{\phi_v''\mu_G}{512\phi_z''L_x''}}}_{K_3}\frac1{\sqrt{\rho'}}\enspace,\\
     \gamma'&\leq \underbrace{\min\left(\frac1{4L^h},\frac{L_x^2}{2L^hL_x''},\frac{\phi_z''\mu_G}{48\beta_{sv}},\frac{\phi_v''\mu_G}{64\beta_{sv}'},\frac1{16P'},\sqrt{\frac{\Gamma'}{8L^hL_x'}}\right)}_{K_4}\\
     \gamma'&\geq \underbrace{\frac{20(\phi_z''L_z'+\phi_v''L_v')}{20}}_{K_5}(\rho')^3 \enspace,\\
     c'&\leq \underbrace{\min\left(\mu_h, \frac1{16},\frac1{16P'}\right)}_{K_6}\enspace,\\
     c'&\leq \underbrace{\frac{\mu_G}{64}}_{K_7}\frac1\xi\enspace,\\
     c'&\leq\underbrace{\frac{\Gamma'}{8}}_{K_8}\frac1{\gamma'}\enspace.
\end{align*}
So, $\xi$, $\rho'$ and $c'$ must verify
\begin{align*}
    \xi&\leq \underbrace{\min\left(K_1,\frac{K_7}{c'}\right)}_{K_1'}\enspace,\\
    \xi&\leq K_2(\rho')^{-\frac12}\enspace,\\
    \xi&\leq K_3(\rho')^{-\frac32}\enspace,\\
    \xi&\leq \underbrace{\min\left(K_4,\frac{K_8}{c'}\right)}_{K_4'}(\rho')^{-1}\\
    \xi&\geq K_5(\rho')^2 \enspace,\\
     c'&\leq \underbrace{\min\left(\mu_h, \frac1{16},\frac1{16P'}\right)}_{K_6}\enspace,\\
\end{align*}
which is possible if
\begin{align*}
    \rho'\leq\min\left(\sqrt{\frac{K_1'}{K_5}},\left(\frac{K_2}{K_5}\right)^{\frac25},\left(\frac{K_3}{K_5}\right)^{\frac27},\left(\frac{K_4'}{K_5}\right)^{\frac13}\right)\enspace.
\end{align*}
So let us take $c' = \min\left(\mu_h, \frac1{16},\frac1{16P'}\right) = \min\left(\mu_h, \frac1{16P'}\right)$,
$$
\rho' = \min\left(\sqrt{\frac{K_1'}{K_5}},\left(\frac{K_2}{K_5}\right)^{\frac25},\left(\frac{K_3}{K_5}\right)^{\frac27},\left(\frac{K_4'}{K_5}\right)^{\frac13},\frac{\mu_G}{64L_z^2},\frac{\overline{\beta}_{zx}}{2\beta_{zx}},\frac{\mu_G}{128(L_v^2 + L_v'')},\frac{\beta_{vz}}{8(L_v^2+L_v'')},\frac{\overline{\beta}_{vx}}{2\beta_{vx}}\right)
$$
and
$$
\xi = \min(K_1, K_2(\rho')^{-\frac12}, K_3(\rho')^{-\frac32},K_4(\rho')^{-1})\enspace.
$$
We have
$$
\calL^{t+1}\leq(1-c)\calL^t
$$
therefore, unrolling yields
$$
    \boxed{h^t - h^*\leq \calL^t\leq(1-c'\gamma)^t\calL^0.}
$$
}
\end{proof}

\section{Convergence rates with weaker regularity assumptions}
To get our rates, we need stronger assumptions than in the stochastic bilevel optimization literature~\cite{Ghadimi2018, Hong2021, Ji2021a, Arbel2022}. In this section, we shortly present the convergence rates we can expect if we replace Assumptions \ref{ass:1} and \ref{ass:2} by Assumptions \ref{ass:1b} and \ref{ass:2b}.

\begin{assumption}\label{ass:1b}
The function $F$ is differentiable. The gradient $\nabla F$ is Lipschitz continuous in $(z,x)$ with Lipschitz constants $L^F_1$.
\end{assumption}

\begin{assumption} \label{ass:2b}
The function $G$ is twice continuously differentiable on $\bbR^p\times\bbR^d$. For any $x\in\bbR^d$, $G(\,\cdot\,,x)$ is $\mu_G$-strongly convex. The derivatives $\nabla G$ are $\nabla^2 G$ are Lipschitz continuous in $(z,x)$ with respective Lipschitz constants $L^G_1$ and $L^G_2$.
\end{assumption}

With these assumptions, we are not ensured that $v^*$ is smooth, and so the descent lemmas take the form of \cref{lemma:old_coupled_inequalities}.

\begin{lemma}\label{lemma:old_coupled_inequalities}
Assume that $\rho\leq\frac{2}{\mu_G}$. We have:
\begin{align*}
    \delta_z^{t+1}&\leq \left(1-\frac{\rho\mu_G}2\right)\delta_z^t + 2\rho^2V_z^t + 4\frac{L_*^2}{\mu_G} \frac{\gamma^2}{\rho}V_x^t \\
    \delta_v^{t+1}&\leq \left(1-\frac{\rho\mu_G}4\right)\delta_v^t +\rho \beta_{vz}\delta_z^t + 2\rho^2V_v^t + 8\frac{L_*^2}{\mu_G} \frac{\gamma^2}{\rho}V_x^t
\end{align*}
where $L_*$ is the maximum between the Lipschitz constants of $z^*$ and $v^*$ (see \cref{lemma:smoothness_star}) and $\beta_{vz} = \frac1{\mu_G^3}(L^F\mu_G+L^G_{2})^2$.
\end{lemma}

\begin{proof}
\textbf{Inequality for $\delta_z$.\quad}

Instead of expanding the square as done in the proof of \cref{lemma:coupled_inequalities} in \cref{app:eq:expansion_delta_z}, we use Young's inequality for some $a>0$
\begin{align}
    \delta^{t+1}_z &\leq (1+a) \bbE[\|z^{t+1} - z^*(x^t)\|^2] + (1+a^{-1})\bbE[\|z^*(x^{t+1}) - z^*(x^t)\|^2]\enspace.
\end{align}

Treating $\bbE[\|z^{t+1} - z^*(x^t)\|^2]$ and $\bbE[\|z^*(x^{t+1}) - z^*(x^t)\|^2]$ as done in the proof of \cref{lemma:coupled_inequalities} leads to
\begin{equation}
    \delta^{t+1}_z \leq (1+a)\left[(1 - \rho\mu_G)\delta_z^t + \rho^2V_z^t\right] + (1+a^{-1})L^2_*\gamma^2V_x^t
\end{equation}

In order to keep a decrease in $\delta_z$, we might want to use $a = \frac12\rho\mu_G$, which gives the bound
\begin{equation}
    \label{eq:bound_delta_z}
\boxed{
    \delta^{t+1}_z \leq \left(1 - \frac{\rho\mu_G}{2}\right)\delta_z^t + 2\rho^2V_z^t +\beta_{zx} \frac{\gamma^2}{\rho}V_x^t
    }
\end{equation}
with $\beta_{zx} = 4\frac{L_*^2}{\mu_G}$.
Indeed, this gives $(1 + \frac12\rho\mu_G)(1-\rho\mu_G) \leq 1 - \frac12\rho\mu_G$.
We have $a \leq 1$ since $\rho\leq\frac{2}{\mu_G}$, so $(1+a)\rho^2\leq 2\rho^2$. Finally, we also have $1+ a^{-1}\leq 2a^{-1} = \frac4{\rho\mu_G}$.

\textbf{Inequality for $\delta_v$.\quad} As for $\delta_z$, the difference with the proof of \cref{lemma:coupled_inequalities} is that we use we use Young's inequality for some $b>0$ to get
\begin{align}
    \delta^{t+1}_v &\leq (1+b) \bbE[\|v^{t+1} - v^*(x^t)\|^2] + (1+b^{-1})\bbE[\|v^*(x^{t+1}) - v^*(x^t)\|^2]\enspace.
\end{align}
The remaining part of the proof is similar to the proof of \cref{lemma:coupled_inequalities}.
\end{proof}

The main difference with \cref{lemma:coupled_inequalities} is that we have $O(\frac{\gamma^2}{\rho})$ in factor of $V_x^t$ instead of $O(\gamma^2)$. As a consequence, we need that the ratio $\frac\gamma\rho$ goes to zero to get convergence, as in \cite{Hong2021}. This prevent us in getting rates that match rates of single level algorithms.

Hence, for SOBA, we have to choose $\gamma = O(T^{-\frac35})$ and $\rho = O(T^{-\frac25})$ and we end up with a convergence rate in $O(T^{-\frac25})$. For SABA, we get a $O((n+m)\epsilon^{-1})$ sample complexity, which is actually the sample complexity of SOBA used with full batch estimated directions.

\end{document}